%% file: main.tex
\documentclass{article}

\PassOptionsToPackage{numbers, compress}{natbib}
\usepackage[preprint]{neurips_2026}

\usepackage[utf8]{inputenc} 
\usepackage[T1]{fontenc}    
\usepackage{hyperref}       
\usepackage{url}            
\usepackage{booktabs}       
\usepackage{amsfonts}       
\usepackage{nicefrac}       
\usepackage{microtype}      
\usepackage{xcolor}         

\usepackage{microtype}
\usepackage{graphicx}
\usepackage{subcaption}
\usepackage{booktabs} 
\usepackage{threeparttable} 
\usepackage{multirow}
\usepackage{amsmath}
\usepackage{amssymb}
\usepackage{mathtools}
\usepackage{amsthm}
\usepackage{siunitx}
\usepackage{placeins}
\usepackage{wrapfig}
\usepackage{xcolor}
\usepackage{tabularx}
\usepackage{xcolor}    
\usepackage[dvipsnames]{xcolor}
\usepackage{tcolorbox} 
\usepackage{longtable}
\usepackage{cleveref}
\usepackage{algorithm}
\usepackage{algorithmic}
\usepackage{makecell}
\usepackage{enumitem}
\input{macro}

\usepackage[nottoc]{tocbibind}
\usepackage{minitoc}
\usepackage{bbm}

\title{Language-Critique Imitation Learning from Suboptimal Demonstrations}

\author{
  Chih-Han Yang$^{1}$ \quad
  Dai-Jie Wu$^{2}$\thanks{Equal contribution} \quad
  Yun-Ping Huang$^{1}$\footnotemark[1]
  \quad
  Ping-Chun Hsieh$^{3}$ \\
  \textbf{Kenneth Marino$^{2}$ \quad
  Shao-Hua Sun$^{1, 4}$} \\
  $^{1}$Graduate Institute of Communication Engineering, National Taiwan University (NTU)\\
  $^{2}$University of Utah\\
  $^{3}$National Yang Ming Chiao Tung University\\
  $^{4}$NTU Artificial Intelligence Center of Research Excellence
}

\begin{document}
\doparttoc 
\faketableofcontents 

\maketitle

\input{tex/0_abstract}
\input{tex/1_intro}
\input{tex/2_related}
\input{tex/3_preliminary}

\input{tex/4_approach}
\input{tex/5_experiments}
\input{tex/6_discussion}

\input{tex/ack}

\begingroup
\hypersetup{colorlinks=false, linkcolor=black}
\hypersetup{pdfborder={0 0 0}}
\renewcommand{\bibname}{References}
\bibliographystyle{unsrtnat}
\bibliography{reference}
\endgroup


\appendix
\input{tex/7_appendix}



\end{document}

%% file: macro.tex
\newcommand{\ie}{\textit{i}.\textit{e}.,\ }
\newcommand{\eg}{\textit{e}.\textit{g}.,\ }

\newcommand{\vspacesection}[1]{\vspace{-0.00cm}
\section{#1}
\vspace{-0.00cm}}

\newcommand{\vspacesubsection}[1]{\vspace{-0.00cm}
\subsection{#1}
\vspace{-0.00cm}}

\newcommand{\method}{LC-BC}
\newcommand{\methoddp}{LC-DP}
\newcommand{\methodcls}{LC-BC-classifier}
\newcommand{\methoddpcls}{LC-DP-classifier}
\newcommand{\rbc}{LC-BC-reward}
\newcommand{\rdp}{LC-DP-reward}
\newcommand{\lang}{language label}
\newcommand{\Lang}{Language label}
\newcommand{\llm}{LLM-Captioner}

\newcommand{\lcloss}{LC-loss}

\newcommand{\sweep}{\textsc{Sweep}}
\newcommand{\boxclose}{\textsc{Box-close}}
\newcommand{\metaworld}{\textsc{Metaworld}}
\newcommand{\blockpushing}{\textsc{BlockPush}}
\newcommand{\peginsertionside}{\textsc{PegInsert}}
\newcommand{\parking}{\textsc{Parking}}
\newcommand{\maze}{\textsc{Maze}}
\newcommand{\hammer}{\textsc{Hammer}}
\newcommand{\relocate}{\textsc{Relocate}}
\newcommand{\adroit}{\textsc{Adroit}}

\newcommand{\mytable}[1]{Table~\ref{#1}}

\newcommand{\mysecref}[1]{Section~\ref{#1}}
\newcommand{\myalgo}[1]{Algorithm~\ref{#1}}

\theoremstyle{plain}
\newtheorem{theorem}{Theorem}[section]

\newtheorem{lemma}[theorem]{Lemma}

\theoremstyle{definition}

\newtheorem{assumption}[theorem]{Assumption}
\theoremstyle{remark}

\definecolor{UserColor}{rgb}{0.6, 0.6, 0.6}
\newtcolorbox{myblock}[1]{colback=UserColor!10, colframe=UserColor, coltitle=white, title=#1, fonttitle=\bfseries}
\newtcolorbox{contblock}{colback=UserColor!10, colframe=UserColor, coltitle=white}

%% file: tex/0_abstract.tex
\begin{abstract}

Prior work on imitation learning from suboptimal demonstrations typically relies on compressed supervision signals such as confidence estimates, discriminator scores, or importance weights. These scalar signals are inherently limited, as they cannot explicitly express intermediate reasoning about task progress, failure modes, or corrective actions.
We propose a language-critique framework for imitation learning from suboptimal demonstrations that instead leverages natural language as a structured supervision signal, avoiding the collapse of expressive feedback into scalars.
Our method first constructs language labels from demonstrations that explicitly describe current progress, identify suboptimal behaviors, and provide fine-grained corrective guidance. 
We then introduce a language-critique loss that directly trains policies using these structured signals without reducing them to scalars, and instantiate it for both behavior cloning and diffusion policies, yielding LC-BC and LC-DP. 
We further provide a theoretical result showing that the proposed objective upper-bounds the expert performance gap under standard assumptions. 
Empirically, we evaluate on diverse continuous control tasks spanning navigation, manipulation, and gameplay, where our methods consistently outperform strong imitation learning and offline reinforcement learning baselines. These results demonstrate that language can serve as a powerful and structured form of supervision for learning robust policies from suboptimal data.

\end{abstract}

%% file: tex/1_intro.tex
\vspacesection{Introduction}

Imitation learning (IL; \citep{morgan1988bc, schaal1996learning, hussein2017imitation, mandlekar2020gti, chen2024diffusion}) has achieved strong results in navigation, driving, robotic control, game playing, and embodied decision making, but policies trained on limited demonstrations often suffer from compounding error~\citep{ross2010efficient} and distribution shift~\citep{ross2011reduction}. A natural remedy is to collect more expert data, yet expert demonstrations are costly and often difficult to scale~\citep{osa2018algorithmic}. Online imitation methods, such as adversarial and distribution-matching approaches~\citep{ho2016generative, fu2017learning, lee2021generalizable, kostrikov2019imitation, lai2024diffusion, huang2024diffusion}, can reduce this gap by querying the environment or an expert during training, but such interaction is often unsafe, expensive, or impractical in real-world settings. This has motivated imitation learning from suboptimal demonstrations~\citep{sasaki2021behavioral, yu2023offline, kim2022demodice, wang2021learning, zhang2021confidence, xu2022discriminator, wang2023imitation, yue2024leverage, yang2022trail, wu2019imitation, 11168119}, where a small amount of expert data is supplemented with abundant imperfect trajectories that are cheaper and easier to collect. When properly exploited, these suboptimal demonstrations can improve state-action coverage and policy robustness without requiring additional environment interaction.

The key challenge is how to extract useful supervision from mixed-quality offline data. Existing methods typically reduce supervision to scalar signals such as confidence estimates~\citep{zhang2021confidence, wu2019imitation}, discriminator scores~\citep{xu2022discriminator, wang2023imitation, yue2024leverage, yang2022trail}, importance weights~\citep{sasaki2021behavioral, yu2023offline, kim2022demodice, wang2021learning}, or rewards (\ie offline RL~\citep{levine2020offline, kumar2020conservative, chen2021decision, fujimoto2021minimalist, kang2023efficient}). While these scalar signals can rank trajectories, they remain highly compressed: they say little about what went wrong, what subgoal should come next, or how behavior should be corrected. As a result, they often fail to capture the structured information that is present in demonstrations, especially in long-horizon or multimodal tasks where progress depends on subtle decisions and stage-specific corrections.

Natural language offers a richer alternative. Instead of collapsing behavior into a single score, language can describe task progress, subgoals, action quality, failure modes, and corrective guidance~\citep{fidler2017teaching, zhang2023bootstrap, cheng2024llfbench, chen2024learning, feng2025natural, lee2025feedback, liu2025synthesizing}.
For example, in a multi-stage manipulation task (\blockpushing{} in \Cref{fig:5-4-blockpushing}) where a robot arm must push two blocks into target regions, language labels can indicate which object to approach, which target to prioritize, and how to adjust the motion, rather than merely assigning a higher or lower reward. This makes language particularly attractive as a supervision signal for learning from suboptimal demonstrations. However, two challenges remain: (1) language labels are not naturally available in offline continuous control datasets, and (2) existing methods often convert language back into scalar rewards or preferences, discarding much of its descriptive power~\citep{cao2024enhancing, wang2024rl, jian2025lapp, urcelay2025from}.

To address these challenges, we introduce language-critique imitation learning from suboptimal demonstrations, an offline IL framework that adds language guidance to learning from expert and suboptimal demonstrations and explicitly uses that guidance to train policies. Our method first constructs structured language labels that describe task progress, action optimality, and movement guidance from offline trajectories. We then design a language-critique loss that uses these labels directly to supervise policy learning, without reducing them to scalar rewards or weights. This loss is applied to both behavior cloning~\citep{morgan1988bc} and diffusion policies~\citep{chi2023diffusionpolicy}, yielding language-critique behavior cloning (LC-BC) and language-critique diffusion policy (LC-DP) as end-to-end offline learning algorithms. We also provide a theoretical justification showing that the proposed language-critique loss upper-bounds the expert-policy performance gap under standard assumptions~\citep{Kakade2002ApproximatelyOA, jin2023provably}.

We evaluate our framework on eight diverse continuous-control tasks spanning navigation, gameplay, and manipulation, including multimodal and precision-manipulation settings. Across these tasks, the proposed language-guided methods, LC-BC and LC-DP, consistently outperform prior imitation learning approaches (\eg DemoDICE~\citep{kim2022demodice}, DWBC~\citep{xu2022discriminator}) and offline RL methods (\eg CQL~\citep{kumar2020conservative}, DT~\citep{chen2021decision}, TD3+BC~\citep{fujimoto2021minimalist}), with especially large gains in multimodal, long-horizon, and precision-control settings. Ablation studies further show that both the structured language design and the language-critique loss are essential, and that the method remains effective even when the expert partition is imperfect. Together, these results show that language is not merely a descriptive aid: it is a powerful supervision signal for learning from suboptimal demonstrations.

%% file: tex/2_related.tex
\section{Related works}

\textbf{Imitation learning (IL)} enables agents to mimic behaviors from expert demonstrations~\citep{schaal1996learning, shafiullah2022behavior}. Online IL methods~\citep{ross2011reduction, ho2016generative, fu2017learning} improve robustness through online environment interaction and expert correction, which incurs high costs and may involve unsafe interactions during training. Offline IL methods based on behavioral cloning~\citep{morgan1988bc, sasaki2021behavioral} instead learn from static datasets, avoiding additional interaction but often suffering from compounding errors~\citep{ross2010efficient} and distribution shift~\citep{ross2011reduction}. To address these limitations, recent work augments limited expert data with abundant suboptimal demonstrations~\citep{yu2023offline, yang2022trail, wu2019imitation, brown2019better}, leveraging techniques such as discriminator- or confidence-based reweighting and behavior purification~\citep{kim2022demodice, wang2021learning, zhang2021confidence,  xu2022discriminator, wang2023imitation, yue2024leverage, florence2021implicit}. Complementary approaches formulate offline IL as inverse reinforcement learning, recovering rewards or world models from mixed-quality data~\citep{watson2023coherent, zeng2023demonstrations, brown2019extrapolating, chen2020learning}. More recently, diffusion policies have further improved robustness and multimodal behavior modeling~\citep{chi2023diffusionpolicy, pearce2023imitating,reuss2023goal, reuss2024multimodal}. Our method follows this offline IL setting with expert and suboptimal demonstrations, while further incorporating \lang{} supervision directly from offline data.

\textbf{Offline reinforcement learning (offline RL)}~\citep{levine2020offline, lange2012batch} learns policies from static datasets without additional environment interaction, but suffers from distribution shift, out-of-distribution actions, and value overestimation. Existing methods address these issues through conservative or implicit value learning~\citep{kumar2020conservative, kostrikov2022offline}, uncertainty estimation~\citep{an2021uncertainty}, value-guided behavior cloning~\citep{jiang2025value}, and conservative model-based planning~\citep{sun2023model, park2024model}. More recently, generative approaches formulate offline RL as sequence modeling~\citep{chen2021decision} or parameterize policies with expressive models such as VAEs and diffusion policies~\citep{kang2023efficient, qing2024a2po, wang2023diffusion, gao2025behavior}, including Q-guided diffusion and diffusion-based planning~\citep{hansen2023idql, lu2023contrastive, chen2024simple, ki2025prior}. However, these methods ultimately rely on scalar reward supervision, whereas our approach instead leverages more expressive \lang{} supervision for offline policy learning.

\textbf{Learning from language.} Natural language has been widely used in decision-making and control, primarily as task instructions or conditioning signals for improving generalization~\citep{shi2024yell, Lynch2021LanguageCI, stepputtis2020language, sun2018neural, sun2020program, liu2023hierarchical, lin2024hierarchical, liu2025synthesizing}. However, most existing methods focus on discrete or text-based domains, with limited study in continuous control~\citep{hermann2017grounded, dai2025racer, xu2025provably, xi2024teaching}. Other approaches convert language into scalar rewards or preferences~\citep{cao2024enhancing, wang2024rl, jian2025lapp, urcelay2025from, yu2023language}, losing much of its structural information. In contrast, our work uses \lang{}s as direct supervision for offline policy learning, preserving the expressiveness of language to enable robust learning from suboptimal demonstrations.

%% file: tex/3_preliminary.tex
\vspacesection{Preliminaries}
\label{prelim}

\textbf{Markov decision process.} A finite-horizon Markov decision process (MDP) is defined by the tuple $\mathcal{M}:= (\mathcal{S}, \mathcal{A}, P, r, \rho_0, \gamma, T)$, where $\mathcal{S}$ and $\mathcal{A}$ are the state and action spaces, $P(s' \mid s, a)$ is the transition function, $r: \mathcal{S} \times \mathcal{A} \to \mathbb{R}$ is the reward function, $\rho_0$ is the initial state distribution, $\gamma \in (0, 1)$ is the discount factor, and $T$ is the horizon. 
A policy $\pi:\mathcal{S}\rightarrow \Delta(\mathcal{A})$ induces a trajectory distribution 
with $s_0\sim\rho_0$, and we denote the expected discounted return by $J_r(\pi):= \mathbb{E}_{\tau \sim \pi}\left[\sum_{t=0}^{T-1} \gamma^t r(s_t, a_t)\right]$.

\textbf{Offline imitation learning from suboptimal demonstrations.}
We consider offline imitation learning with an expert dataset
$\mathcal{D}_E=\{\tau^E_i\}_{i=1}^{M}$ of $M$ trajectories from expert policy $\pi_E$, where
$\tau^E_i=\{(s_{i,t},a^E_{i,t})\}_{t=0}^{|\tau^E_i|-1}$.
The learner also has access to a general dataset
$\mathcal{D}_G=\{\tau^\beta_i\}_{i=1}^{N}$ of $N$ state-action trajectories
$\tau^\beta_i=\{(s_{i,t},a_{i,t})\}_{t=0}^{|\tau^\beta_i|-1}$ collected by behavior policies $\pi_\beta$~\citep{levine2020offline,fu2020d4rl}, spanning expert-like, suboptimal, and random behaviors, with $\mathcal{D}_E \subset \mathcal{D}_G$.
The goal is to learn a policy $\pi_\theta$ that imitates expert behavior while exploiting useful behaviors in $\mathcal{D}_G$.
This motivates supervision signals that identify useful behaviors in heterogeneous offline data.
Additional details are in \Cref{appx:env_details}.

%% file: tex/4_approach.tex
\vspacesection{Approach}
\label{approach}

We present language-critique imitation learning from suboptimal demonstrations, an offline framework that supervises policy learning with structured language labels instead of scalar signals.
\Cref{approach:lc_imitation} formulates the language-critique objective and performance guarantee, \Cref{approach:lang_design} instantiates a structured \lang{} generator $\mu_g$, and \Cref{approach:lc-bc} distills $\mu_g$ into a differentiable \llm{} $\mu_\phi$ with the language-critique loss to train BC and diffusion policies.

\vspacesubsection{From action imitation to language-critique imitation}
\label{approach:lc_imitation}

We use \lang{s} to describe state-action features, distinguish expert from suboptimal behaviors, and exploit them through our language-critique imitation objective. Let $\mu$ denote a \lang{} function conditioned on a state-action pair. For expert state $s_t\sim d_E^t$, action $a_t^E\sim\pi_E(\cdot\mid s_t)$, and label $l_t^E \sim \mu(\cdot \mid s_t, a_t^E)$, we score a policy action $\hat a_t\sim\pi_\theta(\cdot\mid s_t)$ by the likelihood it assigns to $l_t^E$:
\begin{align}
    J_\text{LC}(\pi_\theta) := 
    \sum_{t=0}^{T-1}
        \gamma^t
        \mathbb{E}_{
            s_t\sim d^t_E, 
            a^E_t\sim \pi_E(\cdot\mid s_t),
            l_t^E \sim \mu(\cdot \mid s_t, a_t^E), 
            \hat{a}_t\sim \pi_\theta(\cdot\mid s_t)
        }
        \left[
            \log \mu (l_t^E \mid s_t, \hat{a}_t)
        \right].
    \label{eq:lang_critique_objective}
\end{align}

Mirroring the performance difference formulation in Lemma~\ref{lemma:pdl}, we define the language-critique gap as $\Delta_\text{LC}(\pi_\theta) \coloneqq J_\text{LC}(\pi_E)-J_\text{LC}(\pi_\theta)$.
Since $J_\text{LC}(\pi_E)$ is independent of $\pi_\theta$, minimizing $\Delta_\text{LC}(\pi_\theta)$ is equivalent to maximizing $J_\text{LC}(\pi_\theta)$. Expanding the gap gives
\begin{align}
    \Delta_\text{LC}(\pi_\theta)
    =\sum_{t=0}^{T-1}
        \gamma^t
        \mathbb{E}_{
            s_t\sim d^t_E,
            a^E_t\sim \pi_E(\cdot| s_t),
            l_t^E\sim \mu,
            \hat{a}_t\sim \pi_\theta(\cdot| s_t)
        }
            \left[
                \log \mu(l_t^E | s_t, a_t^E) -
                \log \mu(l_t^E | s_t, \hat{a}_t)
            \right].
    \label{eq:lc_gap_expanded}
\end{align}

Intuitively, the gap penalizes policy actions that make the expert critique less likely than the expert action. In \Cref{appx:theory}, we show that under standard linear-feature realizability and a \textbf{language-sufficiency condition} on $\mu$, minimizing $\Delta_\text{LC}(\pi_\theta)$ controls the expert-state performance gap and recovers expert behavior. The condition rules out uninformative labels and requires $\mu$ to preserve distinctions in state-action features $\psi(s, a):\mathcal{S}\times\mathcal{A}\rightarrow \mathbb{R}^d$: if two actions $a, a'$ at the same state differ, $\mu(\cdot\mid s, a)$ and $\mu(\cdot\mid s, a')$ should be distinguishable. We realize this principle with a structured generator $\mu_g$ (\Cref{approach:lang_design}) and differentiable distillation $\mu_\phi$ for policy optimization (\Cref{approach:lc-bc}).

\vspacesubsection{\Lang{} generation}
\label{approach:lang_design}

To satisfy the language-critique sufficiency in \Cref{approach:lc_imitation}, we define a structured generator $\mu_g$ over state-action pairs. Compared with other $\mu$ candidates (\eg human or free-form LLM/VLM annotations), this design is cheaper, more consistent, and less prone to prompt sensitivity and control-irrelevant variation.

Motivated by reward design principles~\citep{zhang2025rewind, ng1999policy}, each label has three components: task progress \texttt{<T>}, action optimality \texttt{<A>}, and movement guidance \texttt{<M>}. They capture the current stage, whether an action is desirable, and how it should be corrected, while suppressing irrelevant linguistic variation. As shown in \Cref{fig:lc_generate_pipeline} and detailed in \Cref{appx:lang_details}, $\mu_g$ extracts task-relevant heuristics (\eg object distance and gripper status) from a state-action pair $(s_t,a_t)$ without privileged environmental information, and routes them to three component-specific selectors, each sampling a snippet from a pool. Concatenating the snippets yields the final \lang{} $l_t$. Applying $\mu_g$ to each state-action pair in $\mathcal{D}_E$ and $\mathcal{D}_G$ gives language-labeled datasets, where $l_t$ denotes the generated \lang{} for $(s_t,a_t)$:
\begin{align}
    \mathcal{D}^{\text{lang}}_E 
    &= \{(s_t, a^E_t, l^E_t) \mid (s_t, a^E_t) \in \mathcal{D}_E,\; l^E_t \sim \mu_g(\cdot \mid s_t, a^E_t)\}, \\
    \mathcal{D}^{\text{lang}}_G 
    &= \{(s_t, a_t, l_t) \mid (s_t, a_t) \in \mathcal{D}_G,\; l_t \sim \mu_g(\cdot \mid s_t, a_t)\}.
\end{align}

\input{figures/lc_generate_pipeline}

\vspacesubsection{Language-critique imitation learning}
\label{approach:lc-bc}

To enable end-to-end policy optimization, we instantiate the theoretical gap $\Delta_{\text{LC}}(\pi_\theta)$ with $\mathcal{L}_\text{LC}$, a token-level cross-entropy loss computed through the differentiable \lang{} function $\mu_\phi$, namely the \llm{}. The full method is illustrated in \Cref{fig:lc_main_figure}.

\input{figures/lc_main_figure}

\textbf{\llm{}.} We distill the \lang{} generator $\mu_g$ into a differentiable \llm{} $\mu_\phi$ that maps state-action pairs to structured \lang{} outputs. Because $\mu_\phi$ is differentiable, gradients from the language objective in \Cref{eq:lc_gap_expanded} can back-propagate to $\pi_\theta$, enabling joint optimization instead of a staged pipeline that collapses language to scalar targets. Architecturally, $\mu_\phi$ bridges continuous control and \lang{}s by combining a pretrained LLM with an MLP projector that maps state-action pairs into the transformer's hidden space; details appear in \Cref{appx:llm_captioner_details} and \ref{appx:llm_ablation}. We fine-tune $\mu_\phi$ on $\mathcal{D}^\text{lang}_G$ with token-level cross entropy:
\begin{align}
    \ell^\phi_\text{CE}(l_t, s_t, a_t)
    =
    -\frac{1}{N_t}\sum_{n=1}^{N_t}
        \log \mu_\phi(l_{t, n}\mid l_{t, <n}, s_t,a_t)
    \label{eq:ce_loss}
\end{align}
where $N_t$ denotes the tokenized sequence length of $l_t$. Both the LLM-backbone and the projector are updated according to this objective:
\begin{align}
    \mathcal{L}_\text{CE}
    =
    \mathbb{E}_{
        (s_t, a_t, l_t) \sim \mathcal{D}^\text{lang}_G
    }
    \left[
        \ell^\phi_\text{CE}(l_t, s_t,a_t)
    \right]
    \label{eq:llm_captioner_objective}
\end{align}
After fine-tuning, we freeze $\mu_\phi$ and use it for downstream policy optimization. Since the general dataset contains both expert and suboptimal demonstrations, $\mu_\phi$ learns a discriminative mapping from state-action pairs to language. Each component of the predicted \lang{} exposes behavior-quality discrepancies, providing structured supervision that can be exploited during policy learning.

\textbf{Language-critique loss.} We then instantiate the practical language-critique loss through $\mu_\phi$. For a policy action $\hat{a}_t\sim\pi_\theta(\cdot\mid s_t)$ and expert action $a^E_t$, we define
\begin{align}
    \mathcal{L}_{\text{CE}}(\pi_\theta)
    &=
    \mathbb{E}_{
        (s_t,l_t)\sim \mathcal{D}^\text{lang}_E,
        \hat{a}_t\sim \pi_\theta
    }
    \left[
        \ell^\phi_\text{CE}(l_t, s_t, \hat{a}_t)
    \right].
    \label{eq:policy_ce_loss}
    \\
    \mathcal{L}_{\text{CE}}(\pi_E)
    &=
    \mathbb{E}_{
        (s_t,a^E_t,l_t)\sim \mathcal{D}^\text{lang}_E
    }
    \left[
        \ell^\phi_\text{CE}(l_t, s_t, a^E_t)
    \right].
    \label{eq:expert_ce_loss}
\end{align}
The \lcloss{} clips the policy loss at the expert loss:
\begin{align}
    \mathcal{L}_\text{LC}(\pi_\theta, \pi_E)=
    \left[
        \mathcal{L}_\text{CE}(\pi_\theta)-
        \mathrm{sg}(\mathcal{L}_\text{CE}(\pi_E))
    \right]_+,
    \label{eq:lc_loss}
\end{align}
where $[x]_+=\max(x, 0)$ and $\mathrm{sg}(\cdot)$ denotes stop-gradient.
This loss is a differentiable surrogate for $\Delta_{\text{LC}}(\pi_\theta)$: it penalizes policy actions that make the expert \lang{} less likely than the corresponding expert actions, and vanishes once the policy matches the expert under $\mu_\phi$.
Each update transfers the behavior-quality distinctions learned by $\mu_\phi$ from both expert and suboptimal data into the policy.
In \Cref{appx:lc_gap_to_lc_loss}, we show that $\mathcal{L}_\text{LC}$ upper-bounds the performance gap between $\pi_\theta$ and $\pi_E$ up to the captioner distillation residual, extending \Cref{theo:lc_gap_performance_bound} to the practical objective.

\textbf{Language-critique imitation learning.} We apply the \lcloss{} to behavior cloning (BC) and diffusion policy (DP)~\citep{chi2023diffusionpolicy}. Training uses the expert-labeled dataset $\mathcal{D}^\text{lang}_E$, while the \lcloss{} transfers the behavior distinctions learned by $\mu_\phi$ from the broader labeled dataset $\mathcal{D}^\text{lang}_G$, utilizing suboptimal demonstrations into auxiliary supervision.
The \llm{} is used only during training and incurs no test-time cost.
First, for BC, a feedforward MLP policy $\pi_\theta$ is updated with
\begin{align}
    \mathcal{L}_\text{BC}
    = \mathbb{E}_{(s_t,a^E_t)\sim\mathcal{D}_E, \hat{a}_t\sim\pi_\theta}
    \left[
        \left\| a^E_t - \hat{a}_t \right\|_2^2
    \right].
    \label{eq:bc_loss}
\end{align}
Language-Critique Behavior Cloning (\method{}) optimizes \Cref{eq:bc_loss} with the \lcloss{}: $\mathcal{L}_\text{BC}+\lambda\mathcal{L}_\text{LC}$, where $\lambda$ controls \lcloss{} strength.
For DP, the policy $\epsilon_\theta$ is a UNet~\cite{ronneberger2015u} noise-prediction model, yielding Language-Critique Diffusion Policy (\methoddp{}). Its standard objective is
\begin{align}
    \mathcal{L}_\text{DP}
    = \mathbb{E}_{(s_t, a^E_t)\sim\mathcal{D}_E, \epsilon\sim \mathcal{N}(0,\mathbf{I}),k}
    \left[
        \left\|
            \epsilon - \hat{\epsilon}
        \right\|^2_2
    \right]
    \label{eq:dp_loss}
\end{align}
with $\hat{\epsilon}\sim\epsilon_\theta(\cdot\mid s_t,a_t^k,k)$ and $a^k_t=\sqrt{\bar{\alpha}^k}a^E_t+\sqrt{1-\bar{\alpha}^k}\epsilon$. Since DP predicts noise rather than actions, we apply the \lcloss{} to one-step reconstructed action $\hat{a}^0_t$ from the sample $a_t^k$ at diffusion step $k$:
\begin{align}
    \hat{a}_t^0
    =
    \frac{1}{\sqrt{\bar{\alpha}^k}}a_t^k
    -
    \frac{\sqrt{1-\bar{\alpha}^k}}{\sqrt{\bar{\alpha}^k}}\hat{\epsilon},
    \qquad
    \hat{\epsilon}\sim\epsilon_\theta(\cdot\mid s_t,a_t^k,k).
    \label{eq:dp_1step_reconstruct}
\end{align}

Since reconstructions from larger $k$ are noisier, we weight each diffusion-step samples by $\omega^k=\frac{\bar{\alpha}^k}{1-\bar{\alpha}^k}$ and define the reweighted \lcloss{} as
\begin{align}
    &\tilde{\mathcal{L}}_\text{CE}(\pi_\theta)=
    \mathbb{E}_{
        (s_t,l_t)\sim\mathcal{D}^\text{lang}_E, 
        \hat{a}_t^0\sim\epsilon_\theta(s_t),
        k
    }
    \left[
        \omega^k
        \ell^\phi_\text{CE}(l_t, s_t, \hat{a}_t^0)
    \right]
    \label{eq:reweight_agent_ce_loss}\\
    &\tilde{\mathcal{L}}_\text{CE}(\pi_E)=
    \mathbb{E}_{
        (s_t, a^E_t, l_t)\sim\mathcal{D}^\text{lang}_E,
        k
    }
    \left[
        \omega^k
        \ell^\phi_\text{CE}(l_t, s_t, a^E_t)
    \right]
    \label{eq:reweight_expert_ce_loss}\\
    &\tilde{\mathcal{L}}_\text{LC}(\pi_\theta, \pi_E)=
    \left[
        \tilde{\mathcal{L}}_\text{CE}(\pi_\theta)-
        \mathrm{sg}(\tilde{\mathcal{L}}_\text{CE}(\pi_E))
    \right]_+
    \label{eq:reweight_lc_loss}
\end{align}
The \methoddp{} objective, $\mathcal{L}_\text{DP}+\lambda\tilde{\mathcal{L}}_\text{LC}$, encourages reconstructed actions to stay close to expert actions while inducing expert-consistent \lang{}s under $\mu_\phi$'s learned expert-suboptimal distinctions. We provide \methoddp{} details in \Cref{appx:lcdp_details} and full algorithms for both methods in \Cref{appx:algo}.

%% file: figures/lc_generate_pipeline.tex
\begin{figure}[t]
    \includegraphics[width=\linewidth]{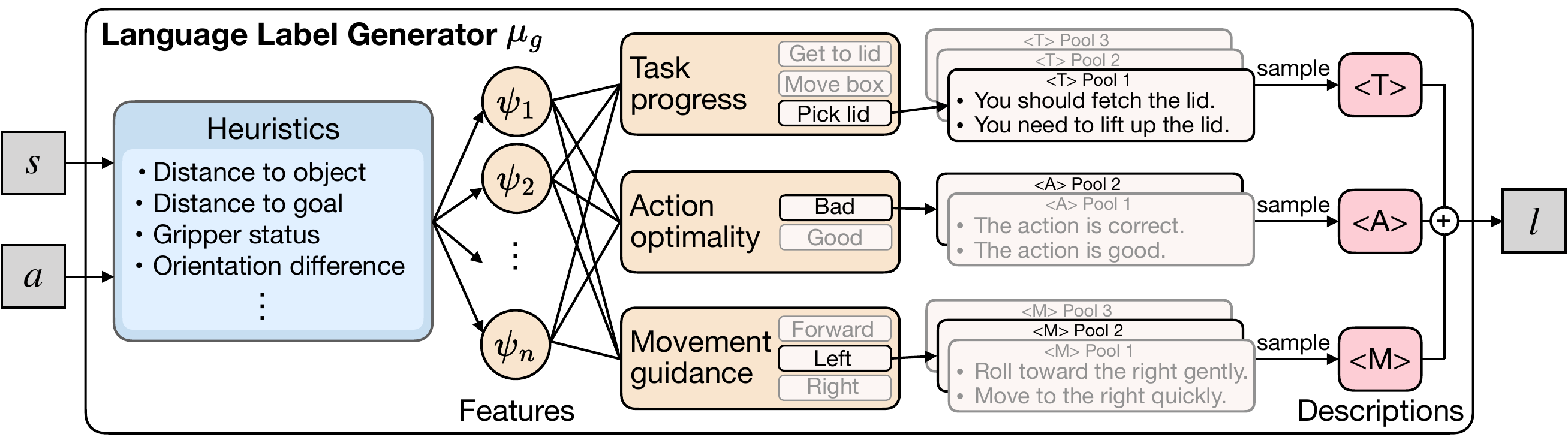}
    \caption{\textbf{\Lang{} generator $\mu_g$.} Given a state-action pair $(s, a)$, $\mu_g$ extracts task-specific features (\eg object distance, gripper status) and routes them to three selectors: task progress \texttt{<T>}, action optimality \texttt{<A>}, and movement guidance \texttt{<M>}. Each selector samples a description snippet, concatenated into the final \lang{} $l$ for scalable and structured supervision.}
    \label{fig:lc_generate_pipeline}
    \vspace{-0.15in}
\end{figure}

%% file: figures/lc_main_figure.tex
\begin{figure}[t]
    \centering
    \def\imgheight{3.6cm} 
    
    \begin{subfigure}[b]{0.28\textwidth}
        \centering
        \includegraphics[height=\imgheight]{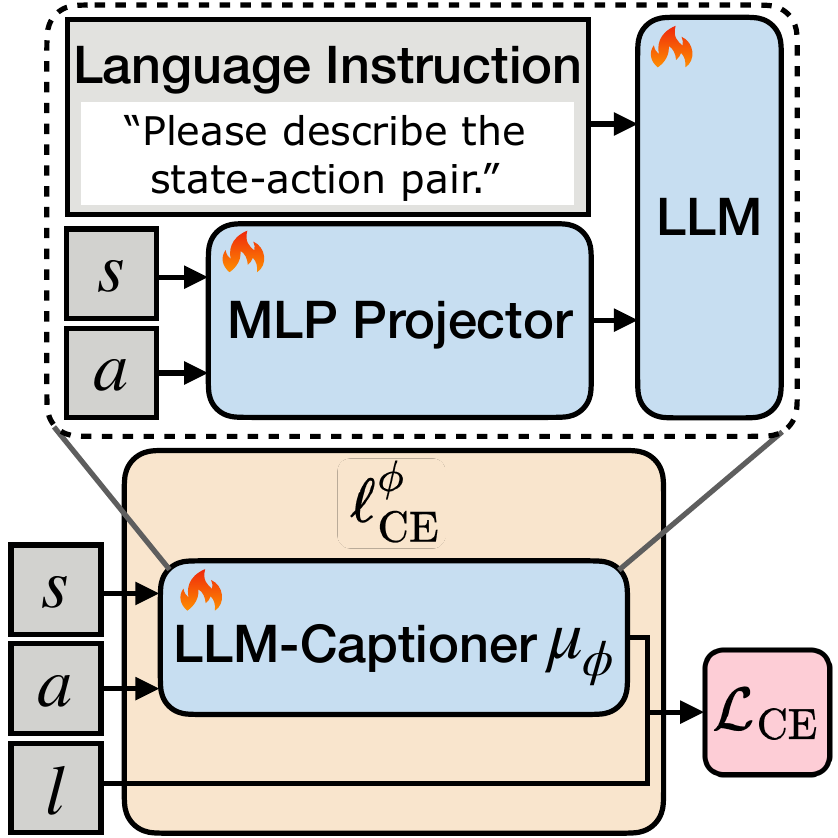}
        \caption{LLM-Captioner}
        \label{fig:llm-cationer}
    \end{subfigure}
    \hfill
    \begin{subfigure}[b]{0.32\textwidth}
        \centering
        \includegraphics[height=\imgheight]{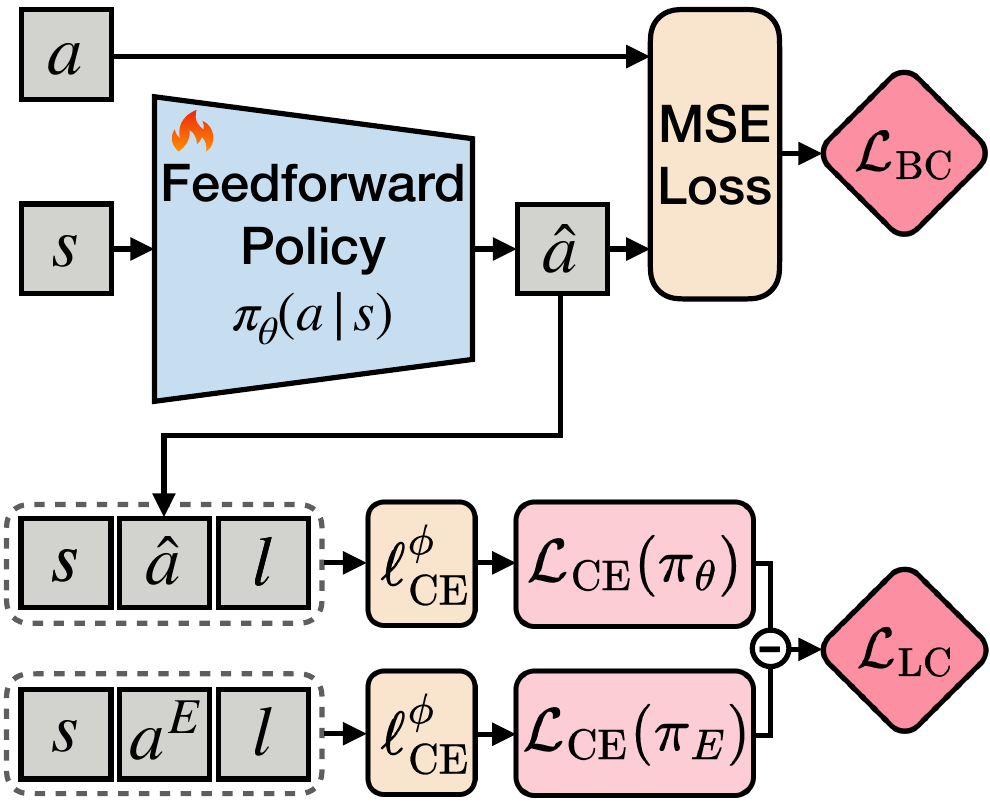}
        \caption{LC-BC}
        \label{fig:lcbc} 
    \end{subfigure}
    \hfill
    \begin{subfigure}[b]{0.36\textwidth}
        \centering
        \includegraphics[height=\imgheight]{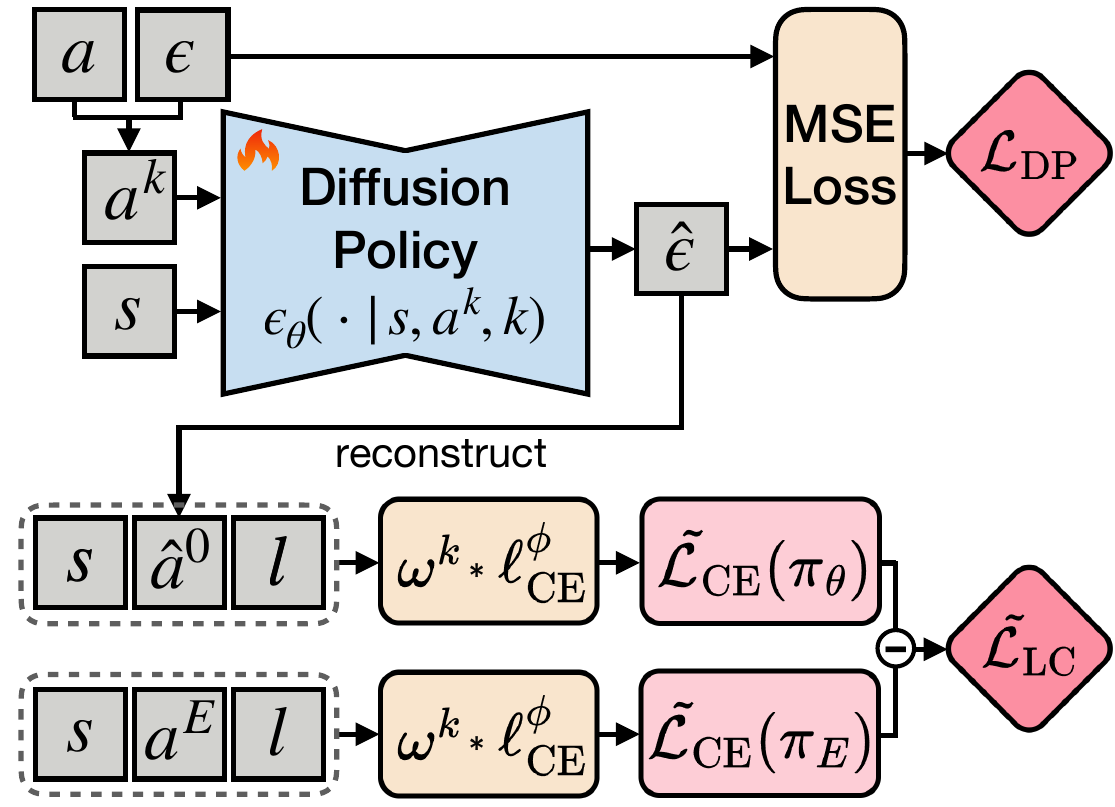}
        \caption{LC-DP}
        \label{fig:lcdp} 
    \end{subfigure}
    
    \caption{\textbf{(a) \llm{} architecture and fine-tuning.} \llm{} uses an LLM backbone and MLP projector to map state-action pairs into the transformer's hidden space. We fine-tune $\mu_\phi$ on $\mathcal{D}^\text{lang}_G$ to distill $\mu_g$.
    \textbf{(b) Language-critique behavior cloning (\method{}).} A feedforward policy $\pi_\theta$ is trained with $\mathcal{L}_\text{BC}$ and \lcloss{} $\mathcal{L}_\text{LC}$. The \lcloss{} applies cross-entropy $\ell^\phi_\text{CE}$ to policy and expert actions via a frozen \llm{}.
    \textbf{(c) Language-critique diffusion policy (\methoddp{}).} A diffusion policy $\epsilon_\theta$ is trained with $\mathcal{L}_\text{DP}$ and reweighted $\tilde{\mathcal{L}}_\text{LC}$. A step-dependent reweighting factor $\omega^k$ handles varying diffusion steps, and \lcloss{} applies to clean actions $\hat{a}^0$ reconstructed from predicted noise $\hat{\epsilon}$.}
    \label{fig:lc_main_figure}
    \vspace{-0.1in}
\end{figure}

%% file: tex/5_experiments.tex
\vspacesection{Experiments}
\label{exp}

\vspacesubsection{Environments}
\label{exp:env_setup}

For evaluation, we use diverse continuous-control environments spanning navigation, driving, and manipulation, with varying horizon length, action precision, multimodality, and state-action dimensionality.
\Cref{fig:environment} shows the experimental tasks.
For each task, we collect an expert dataset $\mathcal{D}_E$ and a general dataset $\mathcal{D}_G$; details are provided in \Cref{appx:env_details} and \ref{appx:training_details}.

\begin{itemize}[leftmargin=2em,itemsep=0pt,topsep=2pt]
    \item \textbf{\maze}: A navigation task ~\cite{fu2020d4rl} where an agent applies $x$/$y$ forces to traverse a maze toward a sampled goal, observing its position, velocity, and the goal location.
    \item \textbf{\parking{}}: A driving task~\citep{highway-env} where an agent parks in a target spot via throttle and steering.
    \item \textbf{\sweep{}} and \textbf{\boxclose{}}: MetaWorld~\citep{yu2019meta} manipulation tasks where a Sawyer robot sweeps a block to a target and places a lid onto a box, respectively.
    \item \textbf{\blockpushing{}}: A multistage task~\citep{florence2021implicit} requiring the robot to push two blocks into two target zones. Multiple valid strategies make this task inherently multimodal.
    \item \textbf{\peginsertionside{}}: A high-precision task~\citep{tao2025maniskill3} requiring joint-level control to insert a peg into a tightly matched side hole, demanding fine-grained corrective feedback.
    \item \textbf{\hammer{}} and \textbf{\relocate{}}: High-dimensional Adroit dexterous hand tasks~\citep{rajeswaran2018learning} involving hammering a nail and relocating a ball.
\end{itemize}

\input{figures/envs}

\vspacesubsection{Task performance}
\label{exp:task_performance}

We first evaluate whether additional \lang{} supervision can better leverage general demonstrations in the suboptimal IL setting. Given a small expert dataset $\mathcal{D}_E$ and a broader general dataset $\mathcal{D}_G$, the key challenge is to identify which suboptimal behaviors are useful for imitation. This experiment tests our central claim that \lang{} provides more expressive behavioral feedback than scalar signals such as rewards, expert-likeness weights, or discriminator scores. We compare \method{} and \methoddp{} against representative IL from suboptimal demonstrations:

\begin{itemize}[leftmargin=2em,itemsep=0pt,topsep=2pt]
    \item \textbf{BC}~\citep{morgan1988bc}: Behavior cloning that maximizes expert action likelihood using expert data only.
    \item \textbf{DWBC}~\citep{xu2022discriminator}: A discriminator-weighted variant of BC that learns a discriminator to distinguish expert from non-expert data and reweights the imitation loss accordingly.
    \item \textbf{DemoDICE}~\citep{kim2022demodice}: An offline imitation method that matches stationary distributions with a KL regularizer toward general data, favoring expert-like state-action visitations.
    \item \textbf{ILID}~\citep{yue2024leverage}: An offline imitation learning method selecting demonstrations whose resulting states lie within the expert state manifold, and then performs BC on both the expert and selected data.
    \item \textbf{DP}~\citep{chi2023diffusionpolicy}: A diffusion-policy baseline that models the action distribution with a conditional denoising process, trained on the expert dataset only.
    \item \textbf{LPB-Offline}: An offline adaptation of LPB~\citep{sun2025latent}, which leverages a dynamic model trained on general data to guide the DP near demonstrated behaviors.
    \item \textbf{\rbc{}} and \textbf{\rdp{}}: Reward-critique variants that reduce \lang{} to scalar rewards under a \lcloss{}-style objective, collapsing all linguistic structure into a single numerical signal. See \Cref{appx:rbc_rdp_details} for details.
    \item \textbf{\methodcls{} and \methoddpcls{}}: Class-critique variants that reduce \lang{} to discrete categorical labels mirroring the \texttt{<T>}, \texttt{<A>}, \texttt{<M>} structure, preserving multi-dimensionality but discarding natural-language expressiveness. See \Cref{appx:mhcbc_mhcdp_details} for details.
\end{itemize}

As shown in \Cref{tab:main_results}, our \method{} and its variants outperform or match feedforward baselines across most environments, with the largest gains on multimodal tasks and those with larger test-time distribution shift, \eg out-of-distribution object or target positions. On \blockpushing{}, \method{} improves over BC from $34.4\%$ to $47.2\%$, indicating that \lang{} distinguishes multimodal demonstrations that scalar signals cannot separate. On manipulation tasks, \method{} outperforms BC and matches or exceeds DWBC, DemoDICE, and ILID, suggesting that language-critique supervision performs better under distribution shift than the scalar signals used by prior methods.

The same pattern carries over to diffusion--our \methoddp{} and its variants match or improve over DP and LPB-Offline on all tasks, with notable gains on manipulation tasks.
On simpler tasks such as \parking{} and \maze{}, most methods already achieve near-saturated performance, limiting the gains from language-critique supervision. DemoDICE performs strongly on several tasks but degrades on \blockpushing{}, suggesting that expert-likeness weighting can be brittle in multimodal or multistage settings. On the precision-control task \peginsertionside{}, \method{} shows only modest improvement over BC, whereas \methoddp{} maintains a clear advantage over DP. This suggests that language labels are more effective for high-level behavioral guidance and are better integrated by diffusion policies through iterative denoising than by feedforward policies for fine-grained control.

Language-critique methods consistently outperform \rbc{} and \rdp{}, confirming that \lang{} carries a richer signal than scalar rewards. The classifier variants reveal a subtler trade-off: \methodcls{} and \methoddpcls{} encode the same information as \lang{} via categorical labels over \texttt{<T>}, \texttt{<A>}, \texttt{<M>}, but discard natural-language expressiveness. While \lang{} captures finer-grained distinctions in behavior quality, it also introduces lexical noise irrelevant to the action signal. Diffusion policies, whose denoising process is robust to such noise, fully exploit the added expressiveness—hence \methoddp{} consistently outperforms \methoddpcls{}. Feedforward policies are more sensitive to this noise, so \methodcls{} occasionally outperforms \method{}.

Overall, the results show that our proposed language-critique imitation learning provides an effective supervision signal for learning from suboptimal demonstrations across various domains.
The qualitative results in \Cref{appx:lc_loss_rollouts} (\Cref{fig:pegInsertion_rollout} \& \Cref{fig:box-close_rollout}) show policy rollouts on \peginsertionside{} and \boxclose{}, comparing BC and \method{} and the corresponding \lang{} and \lcloss{}.

\input{tables/main_table}

\vspacesubsection{Learning from mixed-quality demonstrations without expert partition}

\input{tables/mix_quality_data}

We further relax the assumption of access to a clean expert partition, \ie separate expert and general datasets $\mathcal{D}_E$ and $\mathcal{D}_G$. Specifically, we consider the mixed-quality setting~\citep{levine2020offline, liu2021curriculum}, where only a single offline dataset $\mathcal{D}_G$ containing both expert and suboptimal trajectories is available, analogous to the offline RL setting without reward signals. We evaluate this setting on the \sweep{} and \boxclose{}.

We construct pseudo-expert datasets by selecting sequences with $N$ consecutive positive \lang{}s (where \texttt{<A>} labels each action as positive); larger $N$ ensures stricter alignment with ground-truth expert behavior. Although these datasets contain mixed-quality demonstrations, they are treated as the expert distribution $\mathcal{D}_E$ during imitation learning. As shown in \Cref{tab:mix_quality_data}, we evaluate $N \in \{2, 3, 5, 10\}$, where expert accuracy—the proportion of ground-truth expert samples—serves as a quality proxy ranging from $22.7\%$ to $86.8\%$. For consistency, all methods are trained on identical filtered partitions, though only our approach utilizes \lang{} supervision during policy learning.

\Cref{tab:mix_quality_data} demonstrates that \method{} excels with low-to-medium quality data, outperforming baselines when expert accuracy is lowest. This indicates that \lang{} provides critical fine-grained supervision when demonstrations are suboptimal. 
While DWBC and DemoDICE perform worse with noisy expert partitions, ILID remains competitive in several settings, and our method proves the most robust by learning directly from the \lang{}s.
For diffusion-based policies, \methoddp{} achieves state-of-the-art or competitive results, particularly on \boxclose{}. Ultimately, these results confirm that the proposed language-critique loss enables effective learning from mixed-quality data without requiring a clean expert split.

\vspacesubsection{\Lang{} ablation studies}
\label{exp:lang_ablation_comp}

We conduct a series of ablation studies to identify the contribution of each proposed component and compare our \lang{} generator $\mu_g$ against a general-purpose VLM-based labeling pipeline. 

\textbf{Compositional ablation of \lang{}s.} We study how each proposed component of \lang{}s contributes to policy learning on the task \boxclose{}. As shown in \Cref{tab:lang_ablation} (left), using all three components achieves the best average performance for both \method{} and \methoddp{}, confirming that task progress \texttt{<T>}, action optimality \texttt{<A>}, and movement guidance \texttt{<M>} provide an effective signal. For \method{}, \texttt{<A>} and \texttt{<M>} are individually more useful than \texttt{<T>} alone, as they are more closely related to action selection, while \methoddp{} is well supervised by stage-level signals (\texttt{<T>}) alone; nevertheless, the full label remains the best to distinguish behaviors among the general demonstrations.

\textbf{VLM-generated \lang{}s.}
We investigate whether general-purpose vision-language models (VLMs) can replace our structured language generator $\mu_g$ on \boxclose{}. \Cref{tab:lang_ablation} (right) shows that replacing $\mu_g$ with \texttt{o4-mini} prompting yields limited gains when paired with a small \texttt{SmolLM2-135M-Instruct} captioner: \method{} reaches only $55.2\%$ and \methoddp{} only $63.6\%$, comparable to semantically shuffled $\mu_g$ labels ($55.6\%$). 

We attribute this gap to two distinct factors. First, open-ended VLM labels are verbose, diverse, and redundant, which makes their distribution harder for the captioner $\mu_\phi$ to model. Second, treating $\mu_g$ as a feature-aligned reference reveals substantial semantic inaccuracy: \texttt{o4-mini} agrees with $\mu_g$ on only $51.8\%$ of \texttt{<T>}, $46.1\%$ of \texttt{<A>}, and just $18.8\%$ of \texttt{<M>} descriptions. The sharp drop on fine-grained movement guidance \texttt{<M>} indicates that current VLMs struggle with the precise spatial reasoning over state-action effects that accurate movement labels demand. 

We identify two strategies targeting each factor. First, increasing captioner capacity to \texttt{SmolLM2-360M-Instruct} better models the diverse VLM distribution, raising \method{} to $61.6\%$ and \methoddp{} to $71.2\%$---the latter even surpassing $\mu_g$. Second, post-processing the \texttt{o4-mini} labels into a concise form improves \method{} to $63.4\%$ but does not help \methoddp{}, whereas rewriting them into a $\mu_g$-style structured format performs worse ($56.4\%$ for \method{}, $62.4\%$ for \methoddp{}), confirming that surface-level structure or style alone is insufficient. 

To further disentangle linguistic form from semantic content, we construct two controlled variants: $\mu_g$-shuffled preserves the concise structured form but destroys semantic alignment, while $\mu_g$-verbose preserves semantic correctness but rewrites labels into diverse, human-like language. For \method{}, $\mu_g$-verbose substantially outperforms $\mu_g$-shuffled ($68.0\%$ vs. $55.6\%$), showing that semantic accuracy matters more than surface conciseness for feedforward policies; for \methoddp{}, $\mu_g$-verbose gives no gain over shuffled labels, indicating that diffusion-policy training is instead more sensitive to label-distribution complexity and benefits from concise, well-modeled supervision. 

This asymmetry is consistent with the language-vs-classifier analysis in \Cref{exp:task_performance}: feedforward policies are more sensitive to lexical noise and benefit most from semantically accurate labels, while diffusion policies, whose denoising process is inherently robust to such noise, can tolerate noisier or more verbose supervision and instead benefit more from increased captioner capacity. Thus high-quality VLM labels can be useful when paired with a stronger captioner for DP, but $\mu_g$ remains the most reliable supervision source due to its inherent semantic consistency and feature-aligned structure. Given that the primary failure mode is inaccurate fine-grained movement guidance, we leave to future work the use of stronger VLMs with better spatial and temporal reasoning capabilities, which may close the gap to $\mu_g$ and enable a fully VLM-driven labeling pipeline. Please refer to \Cref{appx:vlm_labels} for a further detailed analysis.

\input{tables/lang_label_ablation}

\vspacesubsection{Comparing to offline RL algorithms}

\input{tables/compare_offlineRL_table}

We examine whether our language labels can capture behavioral information that standard value-learning methods fail to recover from scalar reward signals
by comparing against offline RL algorithms spanning four families: value-based (CQL~\citep{kumar2020conservative}), policy-regularized (TD3+BC~\citep{fujimoto2021minimalist}), sequence modeling (Decision Transformer, DT~\citep{chen2021decision}), and diffusion-based (EDP~\cite{kang2023efficient}). \method{} and \methoddp{} learn from \lang{}s directly, while these baselines use scalar rewards converted from \lang{}s following \citet{cao2024enhancing}.

\Cref{tab:comp_offlineRL_table} shows \method{} and \methoddp{} consistently match or outperform offline RL baselines. \method{} outperforms TD3+BC on nearly all tasks, with particularly large gains on complex manipulation tasks, \eg \boxclose{}, \blockpushing{}, \relocate{}. While TD3+BC performs competitively on simpler tasks, its reliance on value estimation limits its ability to leverage suboptimal data. In the diffusion setting, \methoddp{} achieves the best or near-best performance across most tasks, with notable improvements on long-horizon tasks such as \blockpushing{} and high-dimensional task \relocate{}. 

%% file: figures/envs.tex
\begin{figure}[t]
    \centering
    \begin{minipage}{1.0\textwidth}
        \begin{subfigure}[b]{0.16\textwidth}
            \centering
            \includegraphics[width=\textwidth, height=\textwidth]{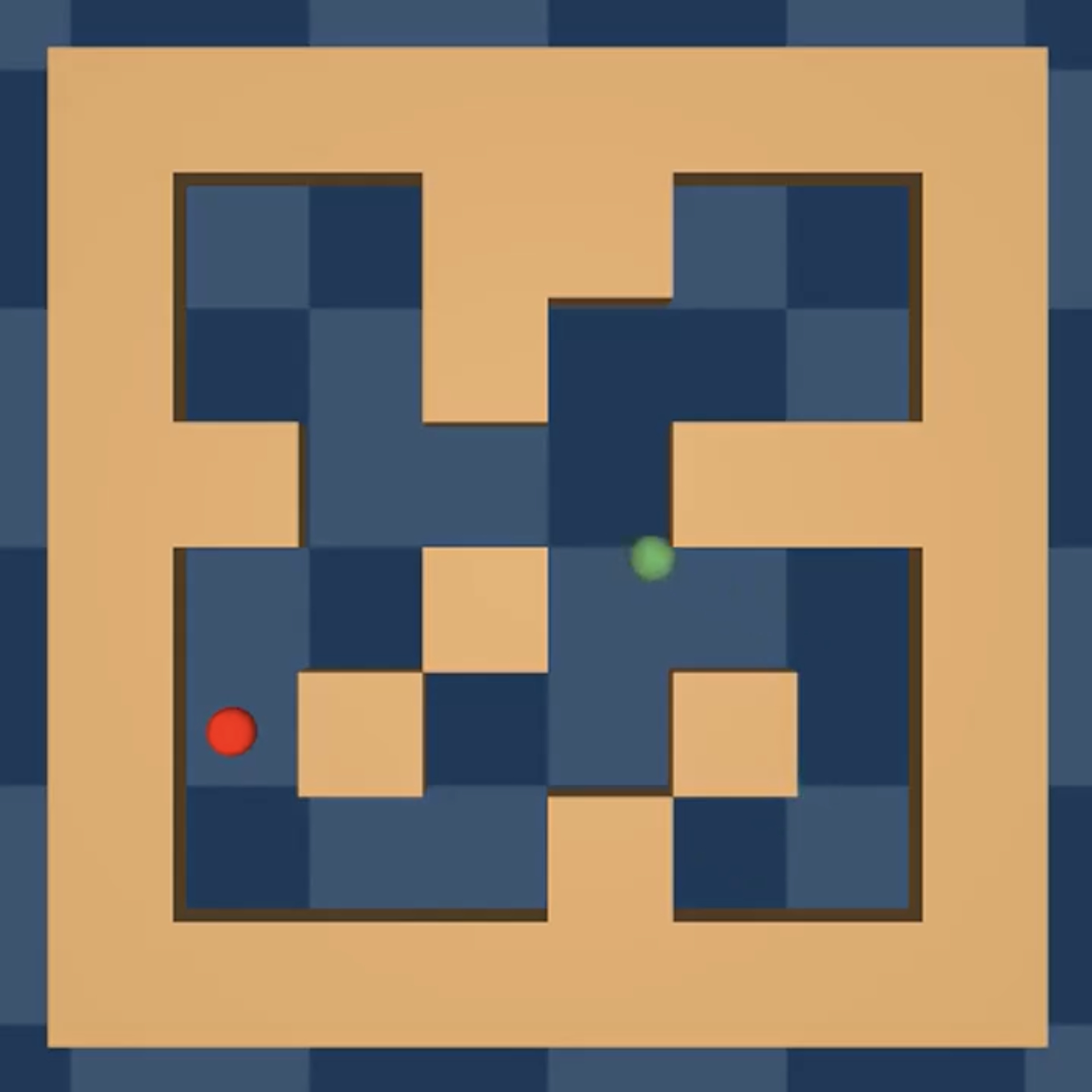}
            \caption{\maze{}}
            \label{fig:5-1-maze2d}
        \end{subfigure}
        \hfill
        \begin{subfigure}[b]{0.16\textwidth}
            \centering
            \includegraphics[width=\textwidth, height=\textwidth]{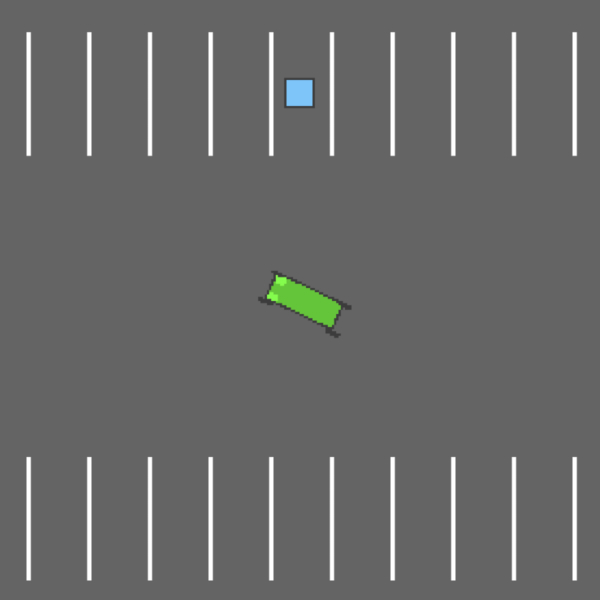}
            \caption{\parking{}}
            \label{fig:5-2-parking-v0} 
        \end{subfigure}
        \hfill
        \begin{subfigure}[b]{0.16\textwidth}
            \centering
            \includegraphics[width=\textwidth, height=\textwidth]{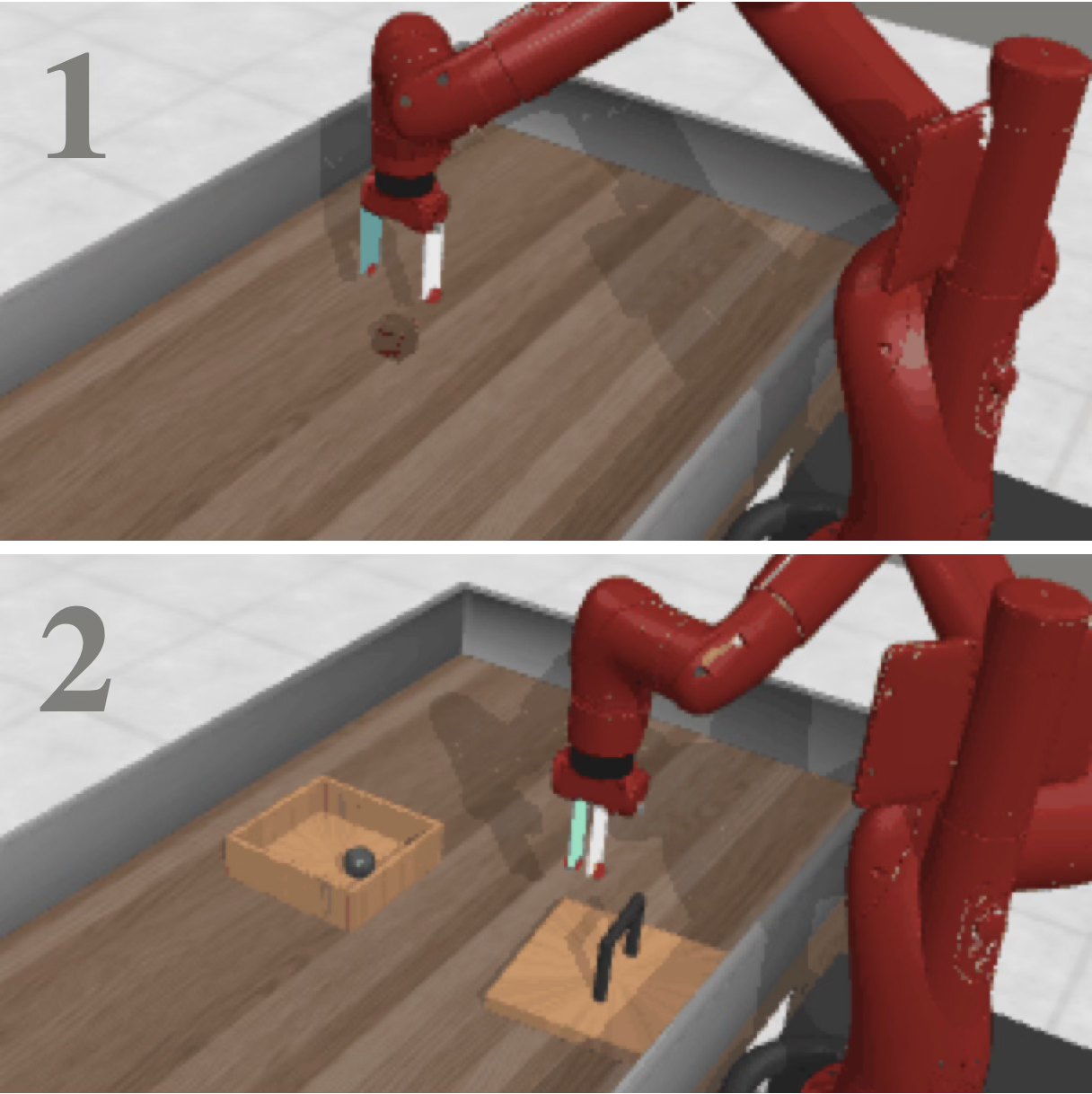}
            \caption{\metaworld{}}
            \label{fig:5-3-metaworld} 
        \end{subfigure}
        \hfill
        \begin{subfigure}[b]{0.16\textwidth}
            \centering
            \includegraphics[width=\textwidth, height=\textwidth]{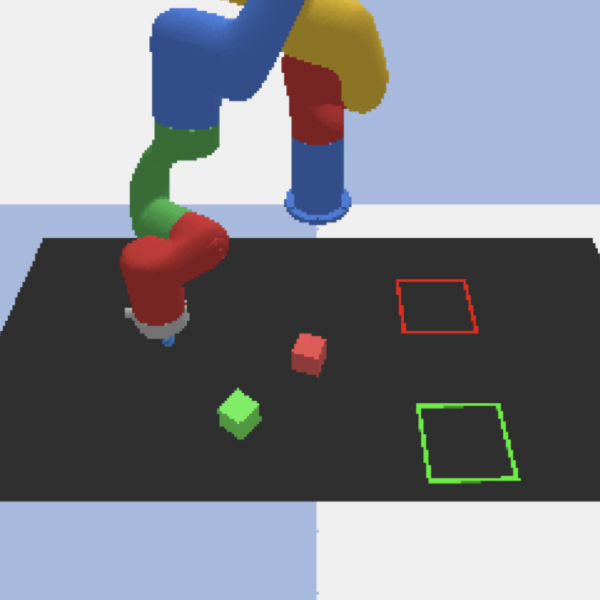}
            \caption{\blockpushing{}}
            \label{fig:5-4-blockpushing}
        \end{subfigure}    
        \hfill
        \begin{subfigure}[b]{0.16\textwidth}
            \centering
            \includegraphics[width=\textwidth, height=\textwidth]{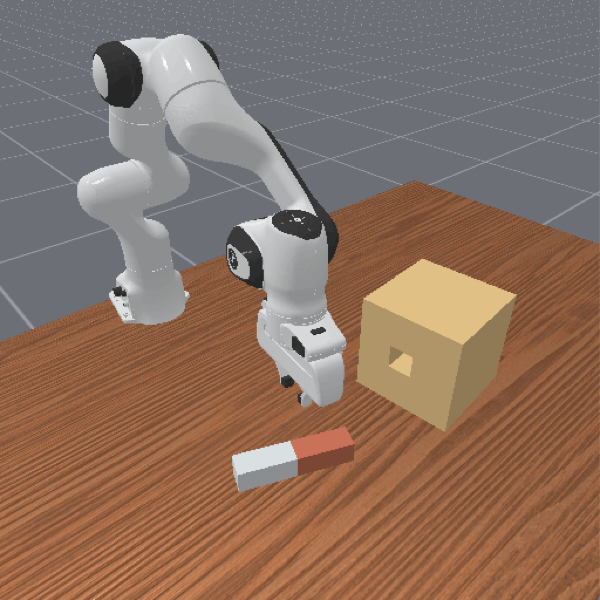}
            \caption{\peginsertionside{}}
            \label{fig:5-5-peginsertionside-v1}
        \end{subfigure}
        \hfill
        \begin{subfigure}[b]{0.16\textwidth}
            \centering
            \includegraphics[width=\textwidth, height=\textwidth]{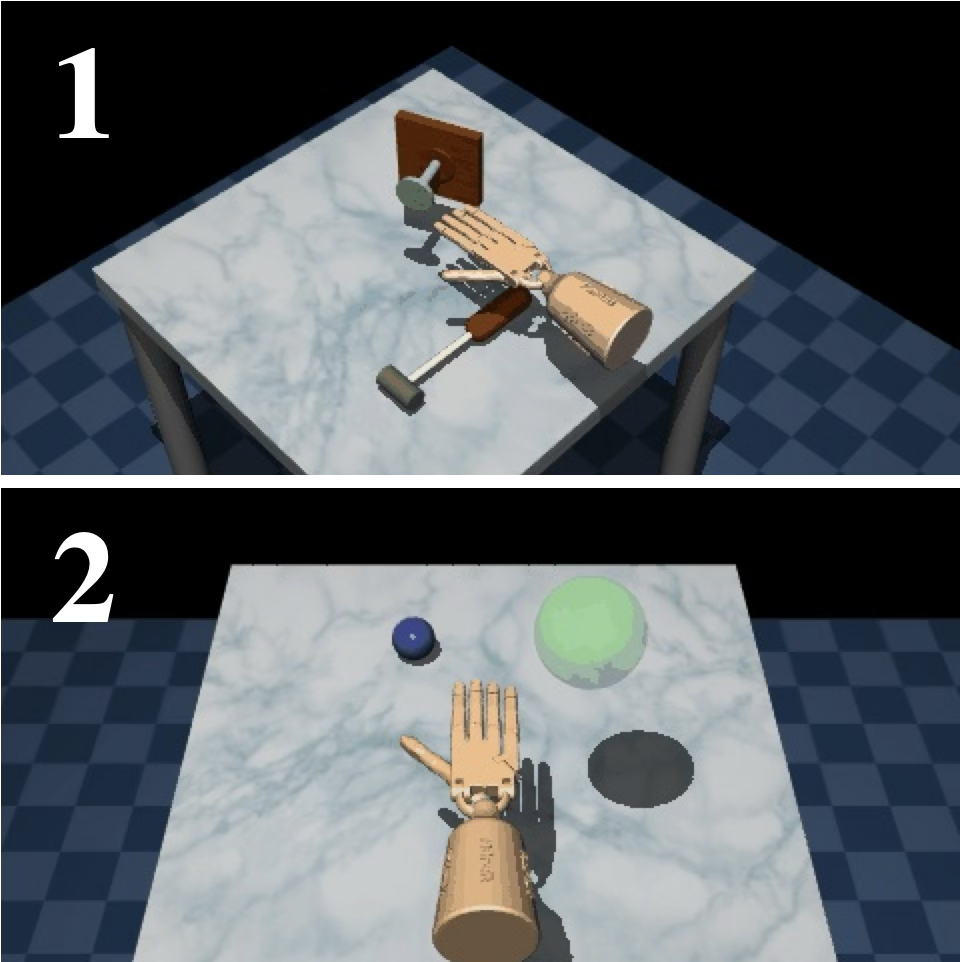}
            \caption{\adroit{}}
            \label{fig:5-6-adroit}
        \end{subfigure}
    \end{minipage}
    \caption[]{\textbf{Environments \& tasks.} 
    \textbf{(a) \maze{}}: Navigate a green agent to a red goal.
    \textbf{(b) \parking{}}: Park a green car in the blue spot.
    \textbf{(c-1) \sweep{}}: Sweep a block sideways.
    \textbf{(c-2) \boxclose{}}: Place a lid onto a box.
    \textbf{(d) \blockpushing{}}: Multi-stage multimodal task pushing two blocks to their targets.
    \textbf{(e) \peginsertionside{}}: Precisely insert a peg into a side hole.
    \textbf{(f-1) \hammer{}}: High-dimensional control to hammer a nail.
    \textbf{(f-2) \relocate{}}: High-dimensional control to relocate a ball.
    }
    \label{fig:environment}
\end{figure}

%% file: tables/main_table.tex
\begin{table}
    \caption{\textbf{Comparisons to offline IL methods.}
We report the mean and standard deviation of success rates across eight tasks, evaluated with 50 environment rollouts and 5 training seeds. 
We pair feedforward policies with BC, DWBC, DemoDICE, \rbc{}, and \method{} and diffusion policies with DP, LPB-Offline, \rdp{}, \methoddp{}.
Best results are in \textbf{bold}; second-best are \underline{underlined}. 
\method{} and \methoddp{} achieve competitive performance relative to all baselines across different environments. (\% omitted.)
}
      \label{tab:main_results}
      \centering
        \begin{small}
          \begin{sc}
            \scalebox{0.72}{
                \begin{tabular}{lcccccccc}
                    \toprule
                    \textbf{Method} & \maze & \parking & \sweep & \boxclose & \blockpushing & \peginsertionside & \hammer & \relocate \\
                    \midrule
                    BC 
                    & 79.6 $\pm$ 4.8
                    & 96.4 $\pm$ 3.0
                    & 76.0 $\pm$ 14.6
                    & 54.8 $\pm$ 8.4
                    & 34.4 $\pm$ 3.3
                    & \underline{66.4} $\pm$ 5.4
                    & 96.0 $\pm$ 3.7
                    & 70.8 $\pm$ 8.3 \\
                    
                    DWBC
                    & \textbf{85.2} $\pm$ 5.8
                    & \underline{98.0} $\pm$ 0.0
                    & 66.0 $\pm$ 6.3
                    & 64.8 $\pm$ 7.8
                    & 26.0 $\pm$ 4.2
                    & 60.8 $\pm$ 1.8
                    & 88.0 $\pm$ 9.9
                    & 48.0 $\pm$ 6.6 \\
                    
                    DemoDICE
                    & 78.0 $\pm$ 3.2
                    & 96.8 $\pm$ 1.1
                    & 68.4 $\pm$ 10.1
                    & \textbf{82.0} $\pm$ 6.3
                    & 3.6 $\pm$ 2.6
                    & \textbf{67.2} $\pm$ 6.1
                    & \underline{96.8} $\pm$ 7.2
                    & 65.2 $\pm$ 8.1 \\

                    ILID
                    & \underline{82.4} $\pm$ 3.6
                    & 93.6 $\pm$ 1.7
                    & \textbf{92.4} $\pm$ 7.3
                    & 69.6 $\pm$ 15.3
                    & 28.4 $\pm$ 4.1
                    & 59.6 $\pm$ 2.6
                    & 90.4 $\pm$ 9.1
                    & \textbf{80.8} $\pm$ 6.9 \\
                    
                    \rbc{}
                    & 79.6 $\pm$ 4.8
                    & 96.8 $\pm$ 2.7
                    & 88.4 $\pm$ 7.0
                    & 63.6 $\pm$ 12.1
                    & 38.8 $\pm$ 5.8
                    & 63.6 $\pm$ 3.3
                    & 90.0 $\pm$ 6.9
                    & 65.6 $\pm$ 11.2 \\

                    \methodcls{}
                    & 78.0 $\pm$ 3.2
                    & \textbf{100.0} $\pm$ 0.0
                    & 85.2 $\pm$ 3.3
                    & \textbf{82.0} $\pm$ 14.4
                    & \underline{40.4} $\pm$ 6.1
                    & 65.2 $\pm$ 3.6
                    & 75.6 $\pm$ 9.8
                    & 66.8 $\pm$ 12.6 \\
                    
                    LC-BC (Ours)
                    & 81.6 $\pm$ 6.1
                    & 96.8 $\pm$ 2.7
                    & \underline{90.0} $\pm$ 5.8
                    & \underline{80.4} $\pm$ 4.6
                    & \textbf{47.2} $\pm$ 4.1
                    & \textbf{67.2} $\pm$ 5.4
                    & \textbf{99.2} $\pm$ 1.8
                    & \underline{71.2} $\pm$ 5.9 \\

                    \midrule
                    
                    DP
                    & \textbf{94.4} $\pm$ 2.2
                    & 96.0 $\pm$ 0.0
                    & 94.0 $\pm$ 5.7
                    & 64.8 $\pm$ 5.0
                    & 83.6 $\pm$ 4.8
                    & \underline{53.6} $\pm$ 1.7
                    & \underline{79.2} $\pm$ 3.0
                    & 67.2 $\pm$ 24.2 \\
                    
                    LPB-Offline
                    & 93.6 $\pm$ 2.2
                    & 95.6 $\pm$ 0.9
                    & 93.6 $\pm$ 7.1
                    & \underline{67.2} $\pm$ 7.8
                    & 78.0 $\pm$ 4.9
                    & 51.6 $\pm$ 5.0
                    & 74.4 $\pm$ 3.6
                    & 59.6 $\pm$ 27.2 \\
                    
                    \rdp{}
                    & \underline{94.0} $\pm$ 2.0
                    & \underline{96.4} $\pm$ 0.9
                    & 94.0 $\pm$ 5.8
                    & 58.4 $\pm$ 5.7
                    & \underline{85.2} $\pm$ 1.1
                    & 53.2 $\pm$ 2.7
                    & 76.0 $\pm$ 1.4
                    & 66.8 $\pm$ 21.8 \\

                    \methoddpcls{}
                    & 92.4 $\pm$ 2.2
                    & \textbf{97.6} $\pm$ 0.9
                    & \underline{96.0} $\pm$ 5.8
                    & 66.8 $\pm$ 4.1
                    & 82.8 $\pm$ 2.3
                    & 52.8 $\pm$ 1.8
                    & 77.6 $\pm$ 0.9
                    & \underline{68.0} $\pm$ 24.8 \\
                    
                    LC-DP (Ours)
                    & \textbf{94.4} $\pm$ 1.7
                    & \textbf{97.6} $\pm$ 1.7
                    & \textbf{96.8} $\pm$ 3.6
                    & \textbf{69.2} $\pm$ 6.7
                    & \textbf{88.0} $\pm$ 1.4
                    & \textbf{56.8} $\pm$ 2.7
                    & \textbf{80.0} $\pm$ 3.2
                    & \textbf{77.6} $\pm$ 4.1 \\
                    \bottomrule
                \end{tabular}
             }
          \end{sc}
        \end{small}
\end{table}

%% file: tables/mix_quality_data.tex
\begin{table}
    \caption{\textbf{Learning from mixed-quality demonstrations.} We relax the clean expert dataset assumption and use language-based heuristic filtering to extract varying levels of expert-quality data from the general dataset. We evaluate two tasks across four quality levels \textit{N=\{2,3,5,10\}}, with expert accuracy shown in parentheses. Higher $N$ indicates stricter filtering; expert accuracy measures the fraction of selected samples from the ground-truth expert dataset as a proxy for data quality. We report success rates over 50 rollouts and 3 seeds (\% omitted). \method{} and \methoddp{} remain competitive with all baselines across both environments and policy architectures.}
      \label{tab:mix_quality_data}
      \centering
        \begin{small}
          \begin{sc}
            \scalebox{0.72}{
                \begin{tabular}{lcccccccc}
                    \toprule
                    \multirow{2.5}{*}{\textbf{Method}}
                    & \multicolumn{4}{c}{\sweep{}}
                    & \multicolumn{4}{c}{\boxclose{}}\\
                    \cmidrule(lr){2-5} \cmidrule(lr){6-9}
                    \multicolumn{1}{r}{}
                    & \textit{2 (22.7\%}) & \textit{3 (33.6\%}) & \textit{5 (51.9\%}) & \textit{10 (84.3\%}) & \textit{2 (32.9\%}) & \textit{3 (48.9\%}) & \textit{5 (68.4\%}) & \textit{10 (86.8\%}) \\
                    \midrule
                    BC 
                    & 53.3 $\pm$ 8.1
                    & 65.3 $\pm$ 7.0
                    & 72.7 $\pm$ 3.1
                    & \underline{94.7} $\pm$ 3.1
                    & 50.0 $\pm$ 4.0
                    & 55.3 $\pm$ 11.0
                    & \underline{87.3} $\pm$ 4.6
                    & 84.7 $\pm$ 2.3 \\
                    
                    DWBC
                    & 60.7 $\pm$ 8.1
                    & 69.3 $\pm$ 5.0
                    & 74.0 $\pm$ 4.0
                    & 84.0 $\pm$ 3.5
                    & 34.7 $\pm$ 2.3
                    & 52.7 $\pm$ 7.0
                    & 71.3 $\pm$ 3.1
                    & 80.0 $\pm$ 2.0 \\
                    
                    DemoDICE
                    & 66.7 $\pm$ 5.0
                    & 76.7 $\pm$ 8.3
                    & \textbf{87.3} $\pm$ 3.1
                    & 80.7 $\pm$ 8.3
                    & \underline{65.3} $\pm$ 15.5
                    & 61.3 $\pm$ 5.0
                    & \underline{87.3} $\pm$ 5.0
                    & 84.0 $\pm$ 5.3 \\

                    ILID
                    & 68.7 $\pm$ 10.1
                    & \textbf{86.7} $\pm$ 8.1
                    & 80.7 $\pm$ 8.3
                    & \textbf{95.3} $\pm$ 8.1
                    & 68.0 $\pm$ 7.2
                    & \underline{63.3} $\pm$ 6.4
                    & 50.7 $\pm$ 43.9
                    & 88.0 $\pm$ 2.0 \\
                    
                    \rbc{}
                    & 52.0 $\pm$ 7.2
                    & 60.7 $\pm$ 4.6
                    & 76.0 $\pm$ 10.6
                    & 91.3 $\pm$ 7.0
                    & 58.7 $\pm$ 5.0
                    & 56.7 $\pm$ 13.3
                    & 82.0 $\pm$ 8.0
                    & 85.3 $\pm$ 5.0 \\

                    \methodcls{}
                    & \textbf{88.7} $\pm$ 2.3
                    & \underline{85.3} $\pm$ 9.2
                    & \underline{85.3} $\pm$ 13.3
                    & 78.7 $\pm$ 5.0
                    & 64.7 $\pm$ 4.2
                    & 64.7 $\pm$ 5.8
                    & 80.0 $\pm$ 7.2
                    & \underline{87.3} $\pm$ 6.4 \\
                    
                    \method{} (Ours)
                    & \underline{72.0} $\pm$ 8.7
                    & 81.3 $\pm$ 4.6
                    & 84.0 $\pm$ 8.7
                    & 88.0 $\pm$ 5.3
                    & \textbf{71.3} $\pm$ 8.3
                    & \textbf{78.0} $\pm$ 3.5
                    & \textbf{88.7} $\pm$ 7.0
                    & \textbf{94.0} $\pm$ 2.0 \\

                    \midrule
                    
                    DP
                    & \textbf{66.0} $\pm$ 3.5
                    & \underline{70.7} $\pm$ 1.2
                    & 86.0 $\pm$ 0.0
                    & \textbf{100.0} $\pm$ 0.0
                    & 56.0 $\pm$ 4.0
                    & \underline{91.3} $\pm$ 1.2
                    & 92.0 $\pm$ 2.0
                    & 95.3 $\pm$ 2.3 \\
                    
                    LPB-Offline
                    & 47.3 $\pm$ 2.3
                    & 58.0 $\pm$ 2.0
                    & 73.3 $\pm$ 8.1
                    & 97.3 $\pm$ 2.3
                    & 54.0 $\pm$ 5.3
                    & 82.0 $\pm$ 4.0
                    & 86.7 $\pm$ 6.1
                    & 91.3 $\pm$ 4.6 \\
                    
                    \rdp{}
                    & \underline{64.7} $\pm$ 4.2
                    & \textbf{72.0} $\pm$ 2.0
                    & \underline{87.3} $\pm$ 1.2
                    & 98.7 $\pm$ 1.2
                    & 58.0 $\pm$ 3.5
                    & 86.0 $\pm$ 2.0
                    & 91.3 $\pm$ 1.2
                    & 95.3 $\pm$ 1.2\\

                    \methoddpcls{}
                    & \underline{64.7} $\pm$ 3.1
                    & 70.0 $\pm$ 2.0
                    & \textbf{88.0} $\pm$ 2.0
                    & \textbf{100.0} $\pm$ 0.0
                    & \underline{58.7} $\pm$ 3.1
                    & 88.7 $\pm$ 1.2
                    & \underline{92.7} $\pm$ 2.3
                    & \underline{96.0} $\pm$ 2.0 \\
                    
                    \methoddp{} (Ours)
                    & 62.7 $\pm$ 1.2
                    & \textbf{72.0} $\pm$ 5.3
                    & \textbf{88.0} $\pm$ 4.0
                    & \underline{99.3} $\pm$ 1.2
                    & \textbf{70.0} $\pm$ 2.0 
                    & \textbf{92.0} $\pm$ 2.0
                    & \textbf{94.0} $\pm$ 0.0
                    & \textbf{96.7} $\pm$ 1.2 \\
                    \bottomrule
                \end{tabular}
             }
          \end{sc}
        \end{small}
\end{table}

%% file: tables/lang_label_ablation.tex
\begin{table}
    \caption{\textbf{\Lang{} ablation.} We ablate \lang{} on \boxclose{}, reporting success rates averaged over 50 environment rollouts and 5 training seeds ({\%} omitted). \textbf{(Left)} Ablation over the three \lang{} components--\texttt{<T>}: Task Progress, \texttt{<A>}: Action Optimality, and \texttt{<M>}: Movement Guidance--showing that all three provide complementary supervision and achieve the best performance. \textbf{(Right)} Comparison of $\mu_g$ versus OpenAI \texttt{o4-mini} \lang{}s across different \llm{} backbones and post-processing strategies (concise, and $\mu_g$-style). 
    $\mu_g$-shuffled randomly permutes \lang{}s within $\mathcal{D}^\text{lang}_G$, preserving format but destroying semantic accuracy; $\mu_g$-verbose preserves semantic accuracy but replaces the format with verbose descriptions.}
    \centering
    \begin{subtable}[t]{0.36\linewidth}
        \centering
        \scalebox{0.73}{
        \begin{tabular}{ccc cc}
            \toprule
            \texttt{<T>} & \texttt{<A>} & \texttt{<M>} & \textbf{\method{}} & \textbf{\methoddp{}} \\
            \midrule
            $\times$ 
            & $\times$ 
            & $\checkmark$     
            & 67.6 $\pm$ 8.3
            & 56.0 $\pm$ 5.7 \\
            $\times$ 
            & $\checkmark$     
            & $\times$ 
            & 69.6 $\pm$ 10.5
            & 65.2 $\pm$ 8.1 \\
            $\checkmark$     
            & $\times$ 
            & $\times$ 
            & 51.6 $\pm$ 4.3
            & \underline{66.4} $\pm$ 7.1 \\
            $\checkmark$ 
            & $\checkmark$ 
            & $\times$     
            & 65.6 $\pm$ 9.5
            & 62.4 $\pm$ 15.5 \\
            $\checkmark$ 
            & $\times$     
            & $\checkmark$ 
            & \underline{79.6} $\pm$ 10.8
            & 62.0 $\pm$ 9.6 \\
            $\times$     
            & $\checkmark$ 
            & $\checkmark$ 
            & 74.4 $\pm$ 15.3
            & 62.8 $\pm$ 8.8 \\
            $\checkmark$ 
            & $\checkmark$ 
            & $\checkmark$ 
            & \textbf{80.4} $\pm$ 4.6
            & \textbf{69.2} $\pm$ 6.7 \\
            \bottomrule
        \end{tabular}
    }
    \end{subtable}
    \hfill
    \begin{subtable}[t]{0.6\linewidth}
        \centering
        \scalebox{0.73}{
        \begin{tabular}{ccc cc}
            \toprule
            \textbf{\Lang{}}
            & \textbf{\llm{} Backbone}
            & \textbf{\method{}}
            & \textbf{\methoddp{}} \\
            \midrule
            
            \texttt{o4-mini} 
            & \texttt{SmolLM2-135M-Instruct} 
            & 55.2 $\pm$ 15.1
            & 63.6 $\pm$ 6.1 \\
            \texttt{o4-mini} 
            & \texttt{SmolLM2-360M-Instruct} 
            & 61.6 $\pm$ 14.6
            & \textbf{71.2} $\pm$ 12.1 \\

            \texttt{o4-mini} concise 
            & \texttt{SmolLM2-135M-Instruct} 
            & 63.4 $\pm$ 10.5
            & 62.4 $\pm$ 7.1 \\
            \texttt{o4-mini} $\mu_g$-style 
            & \texttt{SmolLM2-135M-Instruct} 
            & 56.4 $\pm$ 13.4
            & 62.4 $\pm$ 12.8 \\

            $\mu_g$-shuffled 
            & \texttt{SmolLM2-135M-Instruct} 
            & 55.6 $\pm$ 13.5 
            & 62.4 $\pm$ 11.3 \\

            $\mu_g$-verbose 
            & \texttt{SmolLM2-135M-Instruct} 
            & \underline{68.0} $\pm$ 12.4 
            & 62.0 $\pm$ 10.7 \\
            
            $\mu_g$ 
            & \texttt{SmolLM2-135M-Instruct} 
            & \textbf{80.4} $\pm$ 4.6
            & \underline{69.2} $\pm$ 6.7 \\
            \bottomrule
        \end{tabular}
        }
    \end{subtable}
    \label{tab:lang_ablation}
\end{table}

%% file: tables/compare_offlineRL_table.tex
\begin{table}
    \caption{\textbf{Comparisons to offline RL methods.} We report the mean and standard deviation of success rates across eight tasks, evaluated with $50$ environment rollouts and $5$ seeds. We compare feedforward baselines (CQL, TD3+BC), a sequence-modeling baseline (Decision Transformer; DT), and diffusion baselines (EDP).
    Best results are in \textbf{bold}; second-best are \underline{underlined}. 
    \method{} achieves the best feedforward performance on all eight tasks,
    while \methoddp{} is best on five tasks.
    }
      \label{tab:comp_offlineRL_table}
      \centering
        \begin{small}
          \begin{sc}
            \scalebox{0.74}{
                \begin{tabular}{lcccccccc}
                    \toprule
                    \textbf{Method} & \maze & \parking & \sweep & \boxclose & \blockpushing & \peginsertionside & \hammer & \relocate \\
                    \midrule
                    CQL 
                    & 30.0 $\pm$ 1.4
                    & 25.2 $\pm$ 3.0
                    & 4.0 $\pm$ 3.2
                    & 10.0 $\pm$ 10.3
                    & 0.0 $\pm$ 0.0
                    & 1.6 $\pm$ 0.9
                    & 50.8 $\pm$ 5.2
                    & 2.4 $\pm$ 2.2 \\

                    TD3+BC 
                    & \underline{81.2} $\pm$ 3.0
                    & \underline{93.6} $\pm$ 0.9
                    & 38.8 $\pm$ 6.7
                    & 56.8 $\pm$ 10.4
                    & \underline{41.2} $\pm$ 5.8
                    & 58.4 $\pm$ 3.6
                    & \underline{97.6} $\pm$ 5.4
                    & 41.6 $\pm$ 7.0 \\
                    
                    
                    LC-BC (Ours)
                    & \textbf{81.6} $\pm$ 6.1
                    & \textbf{96.8} $\pm$ 2.7
                    & \textbf{90.0} $\pm$ 5.8
                    & \textbf{80.4} $\pm$ 4.6
                    & \textbf{47.2} $\pm$ 4.1
                    & \textbf{67.2} $\pm$ 5.4
                    & \textbf{99.2} $\pm$ 1.8
                    & \textbf{71.2} $\pm$ 5.9 \\

                    \midrule

                    DT 
                    & 41.6 $\pm$ 4.3
                    & 39.2 $\pm$ 2.3
                    & 59.2 $\pm$ 7.7
                    & 42.0 $\pm$ 5.5
                    & 2.4 $\pm$ 0.9
                    & 40.4 $\pm$ 3.3
                    & 61.6 $\pm$ 12.8
                    & 11.6 $\pm$ 2.6 \\

                    \midrule
                    
                    EDP
                    & 91.2 $\pm$ 1.1
                    & \underline{96.8} $\pm$ 1.1
                    & \textbf{98.4} $\pm$ 2.6
                    & \textbf{71.6} $\pm$ 5.2
                    & 82.0 $\pm$ 6.8
                    & \underline{56.0} $\pm$ 2.4
                    & \textbf{83.6} $\pm$ 3.3
                    & 44.0 $\pm$ 5.1 \\
                    

                    LC-DP (Ours)
                    & \textbf{94.4} $\pm$ 1.7
                    & \textbf{97.6} $\pm$ 1.7
                    & \underline{96.8} $\pm$ 3.6
                    & \underline{69.2} $\pm$ 6.7
                    & \textbf{88.0} $\pm$ 1.4
                    & \textbf{56.8} $\pm$ 2.7
                    & \underline{80.0} $\pm$ 3.2
                    & \textbf{77.6} $\pm$ 4.1 \\
                    \bottomrule
                \end{tabular}
             }
          \end{sc}
        \end{small}
    \vspace{-0.15in}
\end{table}

%% file: tex/6_discussion.tex
\vspacesection{Conclusion}
We introduced a language-critique framework for imitation learning from suboptimal demonstrations, which constructs critiques describing progress, failures, and corrective actions, and trains policies via a language-critique loss without reducing language to scalars, instantiated as LC-BC and LC-DP with theoretical guarantees on the expert-performance gap. Across diverse continuous-control tasks, our method consistently outperforms imitation and offline RL baselines. These results show that language provides an expressive and effective supervision signal for learning robust policies from suboptimal data.
Broader impacts and limitations of our work are discussed in \Cref{appx:impact} and \ref{appx:limitation}, respectively.

%% file: tex/ack.tex
\section*{Acknowledgments}
This work was supported in part by the National Science and Technology Council, Taiwan, under Grants 114-2628-E-002-021-, 114-2628-E-A49-002, 115-2634-F-002-012-, and 115-2223-E-002-005-MY3, and the Taiwan Centers of Excellence in Artificial Intelligence. 
Shao-Hua Sun was supported by the Yushan Fellow Program of the Ministry of Education, Taiwan.

%% file: tex/7_appendix.tex
\clearpage
\section*{Appendix}

\vspace{-1.8cm}
\begingroup
\hypersetup{colorlinks=false, linkcolor=black}
\hypersetup{pdfborder={0 0 0}}
\part{} 
\parttoc 
\endgroup

\input{tex/appendix/theory}

\input{tex/appendix/algorithm}
\input{tex/appendix/lang_details}
\input{tex/appendix/llm_captioner}
\input{tex/appendix/lcdp_details}

\input{tex/appendix/env_details}

\input{tex/appendix/training_details}
\input{tex/appendix/baseline_details}

\input{tex/appendix/add_experiments}
\input{tex/appendix/lc_loss_rollouts}

\input{tex/appendix/impact}
\input{tex/appendix/limitation}

%% file: tex/appendix/theory.tex
\section{Theoretical Analysis}
\label{appx:theory}

In this section, we provide theoretical justification for the language-critique objective introduced in \Cref{approach:lc_imitation}. We show that the expert-state language-critique gap naturally reflects the sub-optimality gap. Specifically, under mild assumptions, minimizing the language-critique gap $\Delta_\text{LC}(\pi_\theta)$ decreases an upper bound on the gap $J_{r^*}(\pi_E) - J_{r^*}(\pi_\theta)$. This provides a theoretical basis for using \lang{} as an auxiliary imitation signal beyond action matching.

\subsection{Expert-state language-critique objective}

Our method uses both expert and suboptimal demonstrations, where suboptimal data may contain useful near-optimal behaviors as well as noisy rollouts that should be discounted. To identify useful behaviors, we need a unified measure of policy quality for both $\pi_\theta$ and the behavior policies underlying the demonstrations. 

Formally, let $d^t_E(s) \coloneq \Pr(s_t = s \mid \pi_E)$ be the expert state distribution at timestep $t$, and let $A^{\pi}_{r,t}(s, a) := Q^{\pi}_{r,t}(s, a) - V^{\pi}_{r,t}(s)$ be the time-dependent advantage, where $Q^{\pi}_{r,t}$ and $V^{\pi}_{r,t}$ are the action- and state-value functions of $\pi$ under reward $r$ over the remaining horizon.

We build on the following performance difference lemma:
\begin{lemma}[Performance difference lemma.]
    Following \citet{Kakade2002ApproximatelyOA}, the return difference between the expert policy and an arbitrary policy $\pi$, evaluated under the optimal reward $r^*$, admits the following decomposition:
    \begin{align}
        J_{r^*}(\pi_E) - J_{r^*}(\pi)
        =
        \sum_{t=0}^{T-1}
        \gamma^t
        \mathbb{E}_{
            s_t \sim d^t_E,\,
            a^E_t \sim \pi_E(\cdot\mid s_t)
        }
        \!\left[A^{\pi}_{r^*, t}(s_t, a^E_t)\right].
        \label{eq:perf_objective_gap}
    \end{align}
    \label{lemma:pdl}
\end{lemma}
This provides such a measure and motivates the objective introduced in \Cref{approach:lc_imitation}, which we derive step by step in the remainder of this section.

For an expert state $s_t \sim d^t_E$, let $a^E_t \sim \pi_E(\cdot\,|\, s_t)$ denote the expert action. Given a \lang{} function $\mu$, the corresponding expert \lang{} is sampled as
\[
    l^E_t \sim \mu(\cdot\,|\, s_t, a^E_t).
\]
For a learned policy $\pi_\theta$, we evaluate the policy action $\hat{a}_t \sim \pi_\theta(\cdot\,|\, s_t)$ by the likelihood with which it explains the same expert \lang{}. We define the expert-state language-critique objective as
\begin{align}
    J_\text{LC}(\pi_\theta) := 
    \sum_{t=0}^{T-1}
        \gamma^t
        \mathbb{E}_{
            s_t\sim d^t_E, 
            a^E_t\sim \pi_E(\cdot\,|\, s_t),
            l_t^E \sim \mu(\cdot \,|\, s_t, a_t^E), 
            \hat{a}_t\sim \pi_\theta(\cdot\,|\, s_t)
        }
        \left[
            \log \mu (l_t^E \,|\, s_t, \hat{a}_t)
        \right].
\end{align}
where $T$ is the horizon of each episode. Similarly, the expert policy obtains
\begin{align}
    J_\text{LC}(\pi_E) := 
    \sum_{t=0}^{T-1}
        \gamma^t
        \mathbb{E}_{
            s_t\sim d^t_E, 
            a^E_t\sim \pi_E(\cdot\,|\, s_t),
            l_t^E \sim \mu(\cdot \,|\, s_t, a_t^E)
        }
        \left[
            \log \mu (l_t^E \,|\, s_t, a^E_t)
        \right].
\end{align}
The language-critique gap is defined as
\begin{align}
    \Delta_\text{LC}(\pi_\theta)
    \coloneq
    J_\text{LC}(\pi_E) - J_\text{LC}(\pi_\theta)
\end{align}
Since $J_\text{LC}(\pi_E)$ is independent of the learned policy, minimizing $\Delta_\text{LC}(\pi_\theta)$ is equivalent to maximizing $J_\text{LC}(\pi_\theta)$. Expanding $\Delta_\text{LC}(\pi_\theta)$, we obtain
\begin{align}
    \Delta_\text{LC}(\pi_\theta) = 
    \sum_{t=0}^{T-1}
    \gamma^t
    \mathbb{E}_{
        s_t\sim d^t_E, 
        a^E_t\sim \pi_E,
        l_t^E \sim \mu(\cdot | s_t, a_t^E), 
        \hat{a}_t\sim \pi_\theta
    }
    \left[
        \log \mu (l_t^E | s_t, a^E_t)
        -
        \log \mu (l_t^E | s_t, \hat{a}_t)
    \right].
\end{align}

\subsection{Language-critique gap as distribution matching}

First, we show that the language-critique gap $\Delta_\text{LC}(\pi_\theta)$ is exactly a KL divergence between the expert \lang{} distribution and the policy-yielded \lang{} distribution. Following the definition of KL divergence, we have:
\[
    D_\text{KL}(
        \mu(\cdot\,|\, s_t, a^E_t)
        ||
        \mu(\cdot\,|\, s_t, \hat{a}_t)
    )
    = \mathbb{E}_{
        l^E_t\sim \mu(\cdot\,|\, s_t, a^E_t)
    }\left[
        \log \frac{\mu(l^E_t\,|\, s_t, a^E_t)}{\mu(l^E_t\,|\, s_t, \hat{a}_t)}
    \right].
\]
Therefore,
\begin{align}
    \Delta_\text{LC}(\pi_\theta) = 
    \sum_{t=0}^{T-1}
    \gamma^t
    \mathbb{E}_{
        s_t\sim d^t_E, 
        a^E_t\sim \pi_E(\cdot\,|\, s_t),
        \hat{a}_t\sim \pi_\theta(\cdot\,|\, s_t)
    }
    \left[
        D_\text{KL}(
            \mu(\cdot\,|\, s_t, a^E_t)
            ||
            \mu(\cdot\,|\, s_t, \hat{a}_t)
        )
    \right].
\end{align}
Therefore, minimizing $\Delta_\text{LC}(\pi_\theta)$ encourages the policy-induced language distribution $\mu(\cdot\,|\, s_t, \hat{a}_t)$ to match the expert-induced language distribution $\mu(\cdot\,|\, s_t, a^E_t)$ on expert states $s_t\sim d^t_E$.

\subsection{Language-critique gap bounds feature gap}

We now connect language distribution matching to feature matching. We assume the following:
\begin{assumption}[Linear realizability]
\label{assump:linear_feature}
    For any policy $\pi$, the action-value function under the optimal task reward $r^*$ is linear in a reward-relevant feature representation: there exist weights $w^\pi_t \in \mathbb{R}^d$ for each $t$ such that $Q^{\pi}_{r^*, t}(s, a) = {w^\pi_t}^\top \psi(s, a)$, where $\psi: \mathcal{S} \times \mathcal{A} \to \mathbb{R}^d$. We assume $w_{\max} \coloneq \sup_{\pi,t} \|w^\pi_t\|_2$ is finite. 
\end{assumption}

Notably, this assumption is standard in the literature of linear MDPs~\citep{jin2023provably} and successor features~\citep{barreto2017successor}, and underlies feature-expectation matching in apprenticeship learning~\citep{abbeel2004apprenticeship}.
    
\begin{assumption}[Language-critique sufficiency]
    The \lang{} preserves reward-relevant distinctions between actions: there exists $c>0$ such that, for any state $s$ and actions $a,a'$,
    \begin{equation}
        D_\text{KL}
            \left(
                \mu(\cdot\,|\, s,a')
                \;\middle\|\;
                \mu(\cdot\,|\, s,a)
            \right)
        \ge
        c
        \left\|
            \psi(s,a')-\psi(s,a)
        \right\|_2^2 .
    \end{equation}
    \label{assump:language_sufficiency}
\end{assumption}
This assumption states that reward-relevant differences in the state-action feature space are reflected by the induced language distributions. 
The condition rules out uninformative labels and requires $\mu$ to preserve distinctions in the features of state-action $\psi(s, a):\mathcal{S}\times\mathcal{A}\rightarrow \mathbb{R}^d$: if two actions $a, a'$ in the same state differ, $\mu(\cdot\mid s, a)$ and $\mu(\cdot\mid s, a')$ shall be distinguishable.
In our setting, this assumption is realized by the structured design of \lang{}: the generated labels explicitly encode task progress, action optimality, and movement guidance, rather than arbitrary free-form descriptions. See \Cref{approach:lang_design}.

Under Assumption~\ref{assump:language_sufficiency}, the expected feature gap is bounded by the language-critique gap:
\begin{align}
    \sum_{t=0}^{T-1}
    \gamma^t
    \mathbb{E}_{
        s_t\sim d^t_E, 
        a^E_t\sim \pi_E(\cdot\,|\, s_t),
        \hat{a}_t\sim \pi_\theta(\cdot\,|\, s_t)
    }
    \left[
        \left\|
            \psi(s_t,a^E_t)-\psi(s_t,\hat{a}_t)
        \right\|_2^2
    \right]
    \le
    \frac{1}{c}
    \Delta_\text{LC}(\pi_\theta).
    \label{eq:lc_gap_feature_mismatch}
\end{align}

\subsection{Proof of the performance difference bound}
\label{appx:proof_lc_gap_perf_bound}

Below we formally state the performance difference bound in~\Cref{theo:lc_gap_performance_bound} with the detailed proof.
\begin{theorem}[Language-critique gap bounds the performance difference gap]
    Under Assumptions~\ref{assump:linear_feature} and~\ref{assump:language_sufficiency}, the performance difference between the expert policy $\pi_E$ and the learned policy $\pi_\theta$ can be upper bounded by the expert-referenced language-critique gap:
    \begin{align}
        \left|J_{r^*}(\pi_E) - J_{r^*}(\pi_\theta)\right|
        \le w_{\max}
        \sqrt{
            \frac{1}{c}\frac{1-\gamma^T}{1-\gamma}
            \Delta_\text{LC}(\pi_\theta)
        }.
        \label{eq:lc_gap_performance_bound}
    \end{align}
    \label{theo:lc_gap_performance_bound}
\end{theorem}
\begin{proof}
By the performance difference lemma~\citep{Kakade2002ApproximatelyOA}, we have
\begin{align}
    J_{r^*}(\pi_E) - J_{r^*}(\pi_\theta)
    =
    \sum_{t=0}^{T-1}
    \gamma^t
    \mathbb{E}_{s_t \sim d^t_E,\, a^E_t \sim \pi_E}
    \!\left[A^{\pi_\theta}_{r^*, t}(s_t, a^E_t)\right].
    \label{eq:appx_pdl_recall}
\end{align}
Since $A^{\pi_\theta}_{r^*, t}(s, a) = Q^{\pi_\theta}_{r^*, t}(s, a) - V^{\pi_\theta}_{r^*, t}(s)$ and $V^{\pi_\theta}_{r^*, t}(s) = \mathbb{E}_{\hat a \sim \pi_\theta(\cdot\,|\, s)}[Q^{\pi_\theta}_{r^*, t}(s, \hat a)]$, we have
\begin{align}
    \mathbb{E}_{a^E \sim \pi_E}\!\left[A^{\pi_\theta}_{r^*, t}(s, a^E)\right]
    =
    \mathbb{E}_{a^E \sim \pi_E,\, \hat a \sim \pi_\theta}
    \!\left[Q^{\pi_\theta}_{r^*, t}(s, a^E) - Q^{\pi_\theta}_{r^*, t}(s, \hat a)\right].
\end{align}

Under Assumption~\ref{assump:linear_feature}, $Q^{\pi_\theta}_{r^*, t}(s, a) = {w^{\pi_\theta}_t}^\top \psi(s, a)$, so
\begin{align}
    Q^{\pi_\theta}_{r^*, t}(s, a^E) - Q^{\pi_\theta}_{r^*, t}(s, \hat a)
    =
    {w^{\pi_\theta}_t}^\top\left(\psi(s, a^E) - \psi(s, \hat a)\right)
    =
    {w^{\pi_\theta}_t}^\top X_t,
\end{align}
where $X_t \coloneq \psi(s_t, a^E_t) - \psi(s_t, \hat{a}_t)$. Plugging this back into~\Cref{eq:appx_pdl_recall}, we have
\begin{align}
    J_{r^*}(\pi_E) - J_{r^*}(\pi_\theta)
    =
    \sum_{t=0}^{T-1}
    \gamma^t\,
    \mathbb{E}\left[{w^{\pi_\theta}_t}^\top X_t\right].
    \label{eq:appx_gap_as_feature_inner_product}
\end{align}

By Jensen's inequality, $\|\mathbb{E}[X_t]\|_2^2 \le \mathbb{E}[\|X_t\|_2^2]$. Applying weighted Cauchy--Schwarz to the discounted sum,
\begin{equation}
    \left\|\sum_{t=0}^{T-1}\gamma^t\mathbb{E}[X_t]\right\|^2_2
    \le
    \left(\sum_{t=0}^{T-1}\gamma^t\right)
    \left(\sum_{t=0}^{T-1}\gamma^t\|\mathbb{E}[X_t]\|_2^2\right)
    \le
    \frac{1-\gamma^T}{1-\gamma}
    \sum_{t=0}^{T-1}\gamma^t\,\mathbb{E}\!\left[\|X_t\|_2^2\right].\label{eq:cauchy schwarz}
\end{equation}

By~\Cref{eq:lc_gap_feature_mismatch} (following Assumption~\ref{assump:language_sufficiency}),
\begin{equation}
    \sum_{t=0}^{T-1}\gamma^t\,\mathbb{E}\!\left[\|X_t\|_2^2\right]
    \le
    \frac{1}{c}\,\Delta_{LC}(\pi_\theta).\label{eq:norm bound}
\end{equation}
Combining~\Cref{eq:cauchy schwarz} and~\Cref{eq:norm bound},
\begin{align}
    \sum_{t=0}^{T-1}\gamma^t\,\|\mathbb{E}[X_t]\|_2
    \le
    \sqrt{\frac{1}{c}\,\frac{1-\gamma^T}{1-\gamma}\,\Delta_{LC}(\pi_\theta)}.
\end{align}

Putting everything together,
\begin{align}
    \left|J_{r^*}(\pi_E) - J_{r^*}(\pi_\theta)\right|
    &=
    \left|
        \sum_{t=0}^{T-1}
        \gamma^t
        \mathbb{E}\left[{w^{\pi_\theta}_t}^\top X_t\right]
    \right|
    \\
    &\le
    w_{\max}
    \left\|
        \sum_{t=0}^{T-1}
        \gamma^t
        \mathbb{E}\left[X_t\right]
    \right\|_2
    \\
    &\le
    w_{\max}\sum_{t=0}^{T-1}\gamma^t\,\|\mathbb{E}[X_t]\|_2
    \\
    &\le
    w_{\max}
    \sqrt{\frac{1}{c}\,\frac{1-\gamma^T}{1-\gamma}\,\Delta_{LC}(\pi_\theta)},
\end{align}
which proves~\Cref{theo:lc_gap_performance_bound}.

\end{proof}

\subsection{Connection to the practical \lcloss{}}
\label{appx:lc_gap_to_lc_loss}

The theoretical objective above is defined using the language-critique function $\mu$. In practice, $\mu$ is instantiated in two stages. First, the \lang{} generator $\mu_g$ produces structured \lang{}s for the expert and general offline datasets. Second, the differentiable \llm{} $\mu_\phi$ is trained on $\mathcal{D}^\text{lang}_G$ to approximate these labels and is then frozen during policy training. In this subsection, we identify the relationship between the theoretical objective and the practical objective. First, we go through the details of the practical \lcloss{}.

For a language label $l_t = \{l_{t,1}, \dots , l_{t, N_t}\}$ of token length $N_t$, the token-level cross entropy loss under $\mu_\phi$ is
\[
    \ell^\phi_\text{CE}(l_t, s_t, a_t) = 
    -\frac{1}{N_t}
    \sum_{n=1}^{N_t}
    \log \mu_\phi(l_{t,n} \,|\, l_{t, <n}, s_t, a_t).
\]
By the chain rule of conditional probability, this satisfies the exact identity:
\begin{align}
    \ell^\phi_\text{CE}(l_t, s_t, a_t) = 
    -\frac{1}{N_t}
    \log \mu_\phi(l_t \,|\, s_t, a_t),
    \label{eq:ce_loss_to_log_mu_phi}
\end{align}
\ie the token-averaged loss is exactly the sequence log-likelihood scaled by $1/N_t$. The dataset-level losses are: 
\begin{align}
    &\mathcal{L}_{\text{CE}}(\pi_\theta)
    =
    \mathbb{E}_{
        (s_t,l_t)\sim \mathcal{D}^\text{lang}_E,
        \hat{a}_t\sim \pi_\theta
    }
    \left[
        \ell^\phi_\text{CE}(l_t, s_t, \hat{a}_t)
    \right]
    \\
    &\mathcal{L}_{\text{CE}}(\pi_E)
    =
    \mathbb{E}_{
        (s_t,a^E_t,l_t)\sim \mathcal{D}^\text{lang}_E
    }
    \left[
        \ell^\phi_\text{CE}(l_t, s_t, a^E_t)
    \right]
    \\
    &\mathcal{L}_\text{LC}(\pi_\theta, \pi_E)
    =
    \left[
        \mathcal{L}_\text{CE}(\pi_\theta)-
        \mathrm{sg}(\mathcal{L}_\text{CE}(\pi_E))
    \right]_+
\end{align}
where $[x]_+=\max(x, 0)$, and $\mathrm{sg}(\cdot)$ denotes stop-gradient, and $\mathcal{L}_\text{LC}(\pi_\theta, \pi_E)$ is the token-level practical \lcloss{}.
To formally connect $\mathcal{L}_\text{LC}$ to $\Delta_\text{LC}$, we introduce the following assumptions corresponding to three key aspects in terms of approximation: bounded \lang{} lengths, $\mu_\phi$ fidelity, and time-step $t$ weighting.

\begin{assumption}[Bounded \lang{} lengths]
    All the \lang{}s yielded by $\mu_g$ have a limited length under some tokenization strategy. There exist constants $0 < N_{\min} \le N_{\max} < \infty$ such that $N_{\min} \le N_t \le N_{\max}$ for all $l_t$ in the $\mathcal{D}^\text{lang}_E$ and every policy sample.
    \label{prop:bounded_lang_len}
\end{assumption}

\begin{assumption}[\llm{} $\mu_\phi$ approximation]
    Since the $\mu_\phi$ approximation is not perfect, there exists $\epsilon_\phi \ge 0$ such that, for every $s_t\sim d^t_E$, $a_t\sim\pi(\cdot\mid s_t)$, and $l_t\sim\mu_g(\cdot\mid s_t, a_t)$,
    \begin{align}
        \left|
            \log \mu_g(l_t \,|\, s_t, a_t) - 
            \log \mu_\phi(l_t \,|\, s_t, a_t)
        \right|
        \le
        \epsilon_\phi.
    \end{align}
    The constant $\epsilon_\phi$ captures the distillation error introduced by replacing the generator $\mu_g$ with the differentiable \llm{} $\mu_\phi$. It is small once $\mu_\phi$ is trained to high fidelity on $\mathcal{D}^\text{lang}_G$.
    \label{prop:llm_captioner_approx}
\end{assumption}

\begin{assumption}[Time-step $t$ coverage]
    Let $p_t$ denote the empirical frequency of timestep $t$ in $\mathcal{D}^\text{lang}_E$, and define $q_t\coloneq \gamma^t/p_t$. There exist constants $0<q_{\min} \le q_{\max} < \infty$ such that $q_{\min} \le q_t \le q_{\max}$ for all $t\in\{0,\dots,T-1\}$. In the undiscounted finite-horizon setting, where $\gamma=1$ with uniform timestep coverage $p_t=1/T$, $q_t=T$ exactly.
    \label{prop:timestep_coverage}
\end{assumption}

The above three assumptions are fairly mild and can be satisfied by our implementation in practice.
We now combine these assumptions to bound the theoretical gap $\Delta_\text{LC}(\pi_\theta)$ by the practical loss $\mathcal{L}_\text{LC}(\pi_\theta, \pi_E)$.
First, we define the captioner-based language critique gap, identical in form to $\Delta_\text{LC}(\pi_\theta)$, but using $\mu_\phi$ instead of a general $\mu$.
\begin{align}
    \Delta_{r,\phi}(\pi_\theta)
    \coloneq
    \sum_{t=0}^{T-1}
    \gamma^t
    \mathbb{E}_{
        s_t\sim d^t_E,
        a^E_t\sim \pi_E(\cdot| s_t),
        l^E_t\sim \mu_g,
        \hat{a}_t \sim \pi_\theta(\cdot| s_t)
    }
    \left[
        \log \mu_\phi(l^E_t| s_t, a^E_t)
        -
        \log \mu_\phi(l^E_t| s_t, \hat{a}_t)
    \right].
\end{align}
We now bound $\Delta_\text{LC}(\pi_\theta)$ — instantiated with the generator $\mu = \mu_g$ — by $\Delta_{r,\phi}(\pi_\theta)$ using Assumption~\ref{prop:llm_captioner_approx}.
For every $(s_t, a_t, l_t)\in\mathcal{D}^\text{lang}_E$ and every policy sample $\hat{a}_t$, the assumption gives
\begin{align}
    \left|\log \mu_g(l^E_t \,|\, s_t, a^E_t) - \log \mu_\phi(l^E_t \,|\, s_t, a^E_t)\right| \le \epsilon_\phi,
    \\
    \left|\log \mu_g(l^E_t \,|\, s_t, \hat{a}_t) - \log \mu_\phi(l^E_t \,|\, s_t, \hat{a}_t)\right| \le \epsilon_\phi.
\end{align}
Combining the two bounds we have,
\[
    \left[
        \log \mu_g(l^E_t \,|\, s_t, a^E_t) - \log \mu_g(l^E_t \,|\, s_t, \hat{a}_t)
    \right]
    \le
    \left[
        \log \mu_\phi(l^E_t \,|\, s_t, a^E_t) - \log \mu_\phi(l^E_t \,|\, s_t, \hat{a}_t)
    \right]
    + 2\epsilon_\phi.
\]
Taking the expectation under $s_t \sim d^t_E$, $a^E_t \sim \pi_E$, $l^E_t \sim \mu_g$, 
$\hat{a}_t \sim \pi_\theta$, and the discounted sum $\sum_t \gamma^t$ on both sides, we obtain
\begin{align}
    \Delta_\text{LC}(\pi_\theta)
    \le
    \Delta_{r,\phi}(\pi_\theta)
    + 2\epsilon_\phi \sum_{t=0}^{T-1}\gamma^t
    =
    \Delta_{r,\phi}(\pi_\theta)
    + 2\epsilon_\phi \frac{1-\gamma^T}{1-\gamma}.
    \label{eq:delta_lc_to_delta_rphi}
\end{align}

According to \Cref{eq:ce_loss_to_log_mu_phi} and Assumption~\ref{prop:bounded_lang_len}, we define a length-weighted gap for the LLM-style log-likelihood we have in practice:
\begin{align}
    \bar{\Delta}_{r,\phi}(\pi_\theta)
    \coloneq
    \sum_{t=0}^{T-1}
    \gamma^t
    \mathbb{E}_{
        s_t\sim d^t_E,
        a^E_t\sim \pi_E,
        l^E_t\sim \mu_g,
        \hat{a}_t \sim \pi_\theta
    }
    &\left[
        \frac{1}{N_t}
        \left(
            \log \mu_\phi(l^E_t| s_t, a^E_t)
            -
            \log \mu_\phi(l^E_t| s_t, \hat{a}_t)
        \right)
    \right],
    \\
    \Delta_{r,\phi}(\pi_\theta) &\le N_{\max} \bar{\Delta}_{r,\phi}(\pi_\theta)
\end{align}

Then, under Assumption~\ref{prop:timestep_coverage}, we can rewrite the 
\begin{align}
    \bar{\Delta}_{r,\phi}(\pi_\theta)
    &=
    \sum_{t=0}^{T-1}
    \gamma^t
    \mathbb{E}_{
        s_t\sim d^E_t,
        a^E_t\sim \pi_E,
        l^E_t\sim \mu_g,
        \hat{a}_t \sim \pi_\theta
    }
    \left[
        \frac{1}{N_t}
        \left(
            \log \mu_\phi(l^E_t| s_t, a^E_t)
            -
            \log \mu_\phi(l^E_t| s_t, \hat{a}_t)
        \right)
    \right]
    \\
    &=\sum_{t=0}^{T-1}
    q_t p_t
    \mathbb{E}_{
        (s_t, a_t, l^E_t)\sim \mathcal{D}^\text{lang}_E,
        \hat{a}_t \sim \pi_\theta(\cdot\,|\, s_t)
    }
    \left[
        \frac{1}{N_t}
        \left(
            \log \mu_\phi(l^E_t| s_t, a^E_t)
            -
            \log \mu_\phi(l^E_t| s_t, \hat{a}_t)
        \right)
    \right]
    \\
    &\le
    q_{\max}
    \sum_{t=0}^{T-1}
    p_t
    \mathbb{E}_{
        (s_t, a_t, l^E_t)\sim \mathcal{D}^\text{lang}_E,
        \hat{a}_t \sim \pi_\theta(\cdot\,|\, s_t)
    }
    \left[
        \frac{1}{N_t}
        \left(
            \log \mu_\phi(l^E_t| s_t, a^E_t)
            -
            \log \mu_\phi(l^E_t| s_t, \hat{a}_t)
        \right)
    \right]
    \\
    &=
    q_{\max} (\mathcal{L}_\text{CE}(\pi_\theta) - \mathcal{L}_\text{CE}(\pi_E))
\end{align}
Also, we have the fact that the clipping has that $x\le [x]_+$, so we can write:
\begin{align}
    \mathcal{L}_\text{CE}(\pi_\theta) - \mathcal{L}_\text{CE}(\pi_E) 
    \le
    \left[
        \mathcal{L}_\text{CE}(\pi_\theta) - \mathcal{L}_\text{CE}(\pi_E) 
    \right]_+
    = 
    \mathcal{L}_\text{LC}(\pi_\theta, \pi_E).
\end{align}
Finally, by putting everything together, we have:
\begin{align}
    \Delta_\text{LC}(\pi_\theta)
    &\le \Delta_{r,\phi}(\pi_\theta) + 2\epsilon_\phi \frac{1-\gamma^T}{1-\gamma}
    \\
    &\le N_{\max} \bar{\Delta}_{r,\phi}(\pi_\theta) + 2\epsilon_\phi \frac{1-\gamma^T}{1-\gamma}
    \\
    &\le N_{\max} q_{\max} \bigl(\mathcal{L}_\text{CE}(\pi_\theta) - \mathcal{L}_\text{CE}(\pi_E)\bigr) + 2\epsilon_\phi \frac{1-\gamma^T}{1-\gamma}
    \\
    &\le N_{\max} q_{\max}\, \mathcal{L}_\text{LC}(\pi_\theta, \pi_E) + 2\epsilon_\phi \frac{1-\gamma^T}{1-\gamma}.
    \label{eq:lc_loss_ge_detla_r_phi}
\end{align}
Combining~\Cref{theo:lc_gap_performance_bound} and~\Cref{eq:lc_loss_ge_detla_r_phi}, we establish the following bound:
\begin{theorem}[\lcloss{} bounds the performance difference gap]
    Combining the result in \Cref{eq:lc_loss_ge_detla_r_phi} with \Cref{theo:lc_gap_performance_bound}, under Assumptions~\ref{assump:linear_feature}, \ref{assump:language_sufficiency} and Assumptions~\ref{prop:bounded_lang_len}, \ref{prop:llm_captioner_approx}, \ref{prop:timestep_coverage}, the performance difference gap between the expert policy and the learned policy is bounded by the realized \lcloss{} up to an \llm{} distillation residual:
    \begin{align}
        |J_{r^*}(\pi_E) - J_{r^*}(\pi_\theta)|
        \le w_{\max}
        \sqrt{
            \frac{1}{c}\frac{1-\gamma^T}{1-\gamma}
            \left(
                N_{\max} q_{\max}\, \mathcal{L}_\text{LC}(\pi_\theta, \pi_E)
                + 2\epsilon_\phi \frac{1-\gamma^T}{1-\gamma}
            \right)
        }.
        \label{eq:lc_loss_performance_bound}
    \end{align}
    \label{theo:lc_loss_performance_bound}
\end{theorem}

%% file: tex/appendix/algorithm.tex
\section{Algorithm}
\label{appx:algo}

The detailed procedures of \method{} and \methoddp{} are presented in \myalgo{algo:LLM-BC} and \myalgo{algo:LLM-DP}, respectively.
First, we annotate both the expert dataset $\mathcal{D}_E$ and the general dataset $\mathcal{D}_G$ with \lang{}s by the \lang{} generator $\mu_g$ to form $\mathcal{D}^\text{lang}_E$ and $\mathcal{D}^\text{lang}_G$.
Then, the \llm{} $\mu_\phi$, initialized from a pretrained LLM $\Psi$ and projector weight $W$, is finetuned from a pretrained LLM on the general dataset $\mathcal{D}^\text{lang}_G$ to distill $\mu_g$. In our main results, we use SmolLM2-135M-Instruct model~\cite{allal2025smollm2} as the \llm{}'s backbone.
Finally, the policy is trained using the proposed \lcloss{} $\mathcal{L}_\text{LC}$ computed via the finetuned \llm{}, simultaneously optimizing the original behavior cloning objective $\mathcal{L}_\text{BC}$ on the expert dataset $\mathcal{D}^\text{lang}_E$. This forms the algorithm of language-critique behavior cloning (\method{}).

For diffusion policies~\cite{chi2023diffusionpolicy}, we additionally incorporate a reweighted \lcloss{} to balance contributions of \lcloss{} across different diffusion timesteps $k$, where the reweighting factor is derived in \mysecref{appx:reweighting_factor_derive}.
Following \citet{kang2023efficient}, we apply one-step reconstruction to recover the clean action from the predicted noise at each timestep, and compute the reweighted \lcloss{} $\tilde{\mathcal{L}}_\text{LC}$ on the reconstructed action $\hat{a}^0_t$.
The diffusion policy is then optimized jointly with this reweighted \lcloss{} and the standard diffusion objective $\mathcal{L}_{\text{DP}}$ on the expert dataset $\mathcal{D}^\text{lang}_E$.

\FloatBarrier
\begin{figure}[htbp]
\centering

\noindent
\begin{minipage}{\linewidth}
\centering

\begin{minipage}[t]{0.493\linewidth}
\begin{algorithm}[H]
   \caption{Language-critique Behavior Cloning}
   \label{algo:LLM-BC}
   \begin{algorithmic}[1]
   \STATE {\bfseries Input:} 
   Expert dataset $\mathcal{D}_E$, General dataset $\mathcal{D}_G$, \lcloss{} weight $\lambda$, 
   Feedforward policy $\pi_\theta$, \Lang{} generator $\mu_g$, and \llm{} $\mu_\phi$

   \STATE Generate $\mathcal{D}^\text{lang}_E$ by $\mu_g$ with $\mathcal{D}_E$
   \STATE Generate $\mathcal{D}^\text{lang}_G$ by $\mu_g$ with $\mathcal{D}_G$
   
   \STATE {\color{gray}// Finetune \llm{} $\mu_\phi$}
   \STATE Initialize $\mu_\phi$ with a pretrained LLM $\Psi$ and a projector $W$
   \FOR{each $W$ finetuning iteration}
        \STATE Sample $(s_t, a_t, l_t) \sim \mathcal{D}^\text{lang}_G$
        \STATE Update $W$ with $\ell^\phi_\text{CE}$ from \Cref{eq:ce_loss}
   \ENDFOR
   \FOR{each $\mu_\phi$ finetuning iteration}
        \STATE Sample $(s_t, a_t, l_t) \sim \mathcal{D}^\text{lang}_G$
        \STATE Update $\mu_\phi$ with $\ell^\phi_\text{CE}$ from \Cref{eq:ce_loss}
   \ENDFOR
   
   \STATE {\color{gray}// Learn policy $\pi_\theta$}
   \STATE Randomly initialize $\pi_\theta$ 
   \FOR{each policy iteration}
        \STATE Sample $(s_t, a^E_t, l_t) \sim \mathcal{D}^\text{lang}_E$
        \STATE $\hat{a}_t \sim \pi_\theta(\cdot\mid s_t)$
        \STATE Compute $\mathcal{L}_{\text{BC}}$ from \Cref{eq:bc_loss} 
        \STATE Compute $\mathcal{L}_{\text{LC}}$ from \Cref{eq:lc_loss} 
        \STATE Update $\pi_\theta$ with $\mathcal{L}_{\text{BC}}+\lambda\mathcal{L}_{\text{LC}}$
   \ENDFOR
   \STATE \textbf{return} $\pi_\theta$
   \end{algorithmic}
\end{algorithm}
\end{minipage}
\hfill
\begin{minipage}[t]{0.493\linewidth}
\begin{algorithm}[H]
   \caption{Language-critique Diffusion Policy}
   \label{algo:LLM-DP}
   \begin{algorithmic}[1]
   \STATE {\bfseries Input:} 
   Expert dataset $\mathcal{D}_E$, General dataset $\mathcal{D}_G$, \lcloss{} weight $\lambda$, 
   Diffusion policy $\epsilon_\theta$, timestep $K$, \Lang{} generator $\mu_g$, and \llm{} $\mu_\phi$
   
   \STATE Generate $\mathcal{D}^\text{lang}_E$ by $\mu_g$ with $\mathcal{D}_E$
   \STATE Generate $\mathcal{D}^\text{lang}_G$ by $\mu_g$ with $\mathcal{D}_G$
   
   \STATE {\color{gray}// Finetune \llm{} $\mu_\phi$}
   \STATE Initialize $\mu_\phi$ with a pretrained LLM $\Psi$ and a projector $W$
   \FOR{each $W$ finetuning iteration}
        \STATE Sample $(s_t, a_t, l_t) \sim \mathcal{D}^\text{lang}_G$
        \STATE Update $W$ with $\ell^\phi_\text{CE}$ from \Cref{eq:ce_loss}
   \ENDFOR
   \FOR{each $\mu_\phi$ finetuning iteration}
        \STATE Sample $(s_t, a_t, l_t) \sim \mathcal{D}^\text{lang}_G$
        \STATE Update $\mu_\phi$ with $\ell^\phi_\text{CE}$ from \Cref{eq:ce_loss}
   \ENDFOR
   
   \STATE {\color{gray}// Learn diffusion policy $\epsilon_\theta$}
   \STATE Randomly initialize $\epsilon_\theta$ 
   \FOR{each policy iteration}
        \STATE Sample $(s_t, a^E_t, l_t) \sim \mathcal{D}^\text{lang}_E$
        \STATE Sample $k \sim \mathrm{Uniform}(\{1,\dots,K\})$
        \STATE Sample $\epsilon \sim \mathcal{N}(0,\mathbf{I})$
        \STATE $a_t^k = \sqrt{\bar{\alpha}^k}a^E_t 
               + \sqrt{1-\bar{\alpha}^k}\epsilon$
        \STATE $\hat{\epsilon} \sim \epsilon_\theta(\cdot\mid s_t, a_t^k, k)$
        \STATE Compute $\mathcal{L}_{\text{DP}}$ from \Cref{eq:dp_loss} 
        \STATE Reconstruct $\hat{a}^0_t$ from \Cref{eq:dp_1step_reconstruct}
        \STATE Compute $\tilde{\mathcal{L}}_{\text{LC}}$ with $\omega^k$
        \STATE Update $\epsilon_\theta$ with 
        $\mathcal{L}=\mathcal{L}_{\text{DP}}+\lambda\tilde{\mathcal{L}}_{\text{LC}}$
   \ENDFOR
   \STATE \textbf{return} $\epsilon_\theta$
   \end{algorithmic}
\end{algorithm}
\end{minipage}

\end{minipage}

\end{figure}

%% file: tex/appendix/lang_details.tex
\section{\Lang{} details}
\label{appx:lang_details}

This section provides additional details on the design and sources of \lang{}s. We first describe the semantic structure of the labels, and then discuss how the language-generation function $\mu$ can be instantiated in practice.

\paragraph{\Lang{} design.}
Assumption~\ref{assump:language_sufficiency} requires that reward-relevant differences in feature space are reflected in the induced language distributions. In other words, when two state-action pairs differ in their reward-relevant semantics, their corresponding \lang{}s should also be distinguishable. This requirement motivates labels that expose task-relevant information rather than unconstrained descriptions of every transition.

Motivated by reward design principles~\cite{zhang2025rewind, ng1999policy}, we decompose each \lang{} into three components: task progress \texttt{<T>}, action optimality \texttt{<A>}, and movement guidance \texttt{<M>}. The goal of this decomposition is not to exhaustively describe all aspects of a transition, but to provide structured and compact supervision for learning from mixed-quality demonstrations.

The task-progress component \texttt{<T>} describes the semantic stage of the current state. This component identifies what part of the task has been achieved and which subgoal is currently relevant. For example, in manipulation tasks, \texttt{<T>} may distinguish whether the end-effector is approaching an object, whether the object has been contacted, whether the object is being moved toward the target, or whether the task is completed. By explicitly representing task progress, \texttt{<T>} helps distinguish states that may be close in raw observation space but correspond to different stages of the long-horizon objective.

The action-optimality component \texttt{<A>} characterizes the quality of the demonstrated action under the current task context. Instead of assigning a scalar reward or confidence weight, \texttt{<A>} expresses whether the action is beneficial, ineffective, or harmful for task completion. This component is particularly important in mixed-quality datasets, where suboptimal demonstrations may contain both useful intermediate states and undesirable actions. \texttt{<A>} provides supervision for deciding which behaviors should be imitated and which should be down-weighted or corrected.

The movement-guidance component \texttt{<M>} provides fine-grained corrective information about how the agent should move. While \texttt{<T>} captures the high-level task stage and \texttt{<A>} captures action quality, \texttt{<M>} describes the direction or adjustment needed to make progress. For instance, it may indicate that the agent should move toward a target, align with an object, push in a particular direction, reduce excessive motion, or correct an overshooting behavior. This component is especially useful for continuous-control tasks, where small differences in actions can lead to different outcomes but may not be fully captured by coarse task-stage labels alone.

Together, \texttt{<T>}, \texttt{<A>}, and \texttt{<M>} provide complementary supervision. \texttt{<T>} specifies where the agent is in the task, \texttt{<A>} evaluates whether the demonstrated action is appropriate, and \texttt{<M>} describes how the behavior should be adjusted. This structured design encourages the induced language distribution to distinguish behaviors that matter for task success while reducing irrelevant variation caused by unconstrained natural language. We evaluate the contribution of these components in \Cref{exp:lang_ablation_comp}.

For example, in the \peginsertionside{} task~(\Cref{fig:5-5-peginsertionside-v1}), the robot arm must grasp a peg and insert it into the side of a box. A generated \lang{} may include task progress, such as \textit{Align the peg to the hole.}; action optimality, such as \textit{The action is bad for the state, since the peg misaligns with the hole.}; and movement guidance, such as \textit{Plus, you can pitch to the top softly.} Combined, the final label is: \textit{Align the peg to the hole. But the action is bad for the state, since the peg misaligns with the hole. Plus, you can pitch to the top softly.} Additional examples across tasks are provided in \Cref{tab:lang_samples}.

\input{figures/vlm_prompt_template}

\input{tables/lang_label_samples}

\paragraph{Sources of \lang{}s.}
Given the desired label structure, the remaining question is how the language-generation function $\mu$ is instantiated. In general, $\mu$ can take several forms. Human annotators can provide semantically rich feedback, but annotation is labor-intensive, expensive to scale, and may introduce inconsistency. LLMs and VLMs can generate language labels from symbolic states, actions, or visual observations, and have been widely used for language-based reasoning and decision making in embodied agents~\cite{sumers2023distilling, driess2023palm, jones2025beyond}. However, the induced language distribution can vary substantially with the model, prompt template, decoding strategy, and input representation. In addition, LLM-based or VLM-based generation may introduce computational overhead when applied to large offline datasets. Our proposed \lang{} generator provides a controlled alternative: it is less flexible than human or model-generated feedback, but offers reproducible and consistent labels when task states and relevant geometric relations are available.

In this work, we mainly instantiate $\mu$ with a \lang{} generator $\mu_g$ built on LLF-Bench~\citep{cheng2024llfbench}. We implement a customized version of the LLF-Bench generation pipeline to produce task-specific labels following the \texttt{<T>}, \texttt{<A>}, and \texttt{<M>} structure described above. This choice allows us to study whether structured language supervision improves offline imitation learning without variation from model sampling, prompt design, or human annotator inconsistency. Although our generator is built on LLF-Bench~\citep{cheng2024llfbench}, it differs from the original benchmark in two important ways. First, we propose and implement the structured \texttt{<T>}, \texttt{<A>}, and \texttt{<M>} decomposition, whereas the original LLF-Bench feedback mainly indicates whether the behavior is good or bad and often contains larger language variations. Our structure is designed to make task progress, action quality, and corrective movement guidance explicit and consistent across the dataset. Second, we extend the generation pipeline beyond the original LLF-Bench domains to cover a broader set of embodied decision-making tasks, including \maze{}, \blockpushing{}, \peginsertionside{}, and Adroit hand tasks such as \hammer{} and \relocate{}. The overall generation details of $\mu_g$ are provided in \Cref{fig:lc_generate_pipeline}.

We additionally evaluate alternative \lang{} sources based on LLM/VLM generation in \Cref{exp:lang_ablation_comp}. Specifically, we collect language labels by prompting \texttt{o4-mini}, a multimodal vision-language reasoning model, to generate \lang{}s from trajectory observations. The prompting template is provided in \Cref{fig:vlm_prompt_temp}, and generated examples are included in \Cref{tab:lang_samples} and \Cref{tab:lang_samples_variant}. These experiments allow us to examine whether the proposed framework can benefit from more flexible model-generated language, while the main results use $\mu_g$ labels for consistency.

%% file: figures/vlm_prompt_template.tex
\begin{figure}
    \centering
    \caption{\textbf{VLM Prompting.} The prompt we use to generate \lang{}s with \texttt{o4-mini}.}
    \begin{myblock}{VLM Prompt with \texttt{o4-mini}.}
        You are shown two images:
        \begin{itemize}
            \item Image 1 is the state BEFORE the action (current state).
            \item Image 2 is the state AFTER the action (next state).
        \end{itemize}
        
        Below is additional information about the system:
        
        \quad (STATE SPACE)
        \{state space description\}

        \quad (ACTION SPACE)
        \{action space description\}

        You are also given the numeric representations of both states and the action taken:
        
        \quad (NUMERIC STATE)
        current state: \{state\}
        ; next state: \{next state\}
        ; action taken: \{action\}

        The agent's task is described as follows (in second-person view):
        \{task instruction\}

        Use BOTH the images and the numeric states to evaluate the transition.
        Before anything else, identify the main objects in the scene (robot arm, gripper, target object, goal region, obstacles) and reason about their spatial relationships (left/right, nearer/farther, higher/lower, touching/not touching).
        Use the numeric state values to clarify small or ambiguous visual changes, such as exact distances, object displacement, or contact.
        Focus on how positions, distances, and contacts changed from the previous state to the next state.
        Answer the following questions:
        \begin{enumerate}
            \item Using BOTH images and numeric states, describe the relative positions of the agent, the main object, and the goal.
            \item In the previous state, determine the agent's task stage based on the gripper–object–goal configuration.
            \item Describe the key spatial differences between the two states, focusing ONLY on task-relevant objects such as changes in distance, movement direction, contact, or orientation.
            \item Decide whether the action was GOOD or BAD for accomplishing the task goal, tolerating small errors or minor regressions unless the action clearly moves the agent meaningfully farther from success.
            \item Provide a concise recommendation for what the agent should do NEXT, specifying clear spatial directions such as move left/right/forward/backward/up/down, open or close the gripper, or push or pull.
        \end{enumerate}
        
        First, write a few sentences of reasoning that:
        \begin{itemize}
            \item Describe the relative positions of the agent, the main object, and the goal using both images and numeric states.
            \item Identify the task stage in the previous state.
            \item Explain what visually and numerically changed between the two states.
            \item Justify why the action is good or bad in terms of progress toward the goal.
            \item Explain why your recommended next movement is appropriate based on spatial reasoning and numeric-state dynamics.
        \end{itemize}
        
    Then, on a NEW LINE at the end, output a single human-style message in the following EXACT format:
    RESULT: followed by 1–5 short sentences addressed directly to the agent in second person that describe the current task stage, clearly state whether the last action was good or bad, and give a specific spatial recommendation for what to do next.

    In the RESULT message:
    \begin{itemize}
        \item Speak directly to the agent using you and your.
        \item Do NOT mention images, numeric states, coordinates, or technical terms.
        \item Make it sound like natural human advice rather than a formal report.
    \end{itemize}
    
    You MUST follow the format.
    \end{myblock}
    \label{fig:vlm_prompt_temp}
\end{figure}

%% file: tables/lang_label_samples.tex
\begin{table}
    \centering
    \caption{\textbf{\Lang{} samples.} We listed expert and suboptimal samples for all tasks generated with $\mu_g$.}
    \renewcommand{\arraystretch}{1.2}
    
    \begin{tabularx}{\textwidth}{l X X}
        \toprule
        \textbf{Task} & Expert Samples & Suboptimal Samples \\
        \midrule
        \textbf{\maze{}}
            & You should go in the west direction. Furthermore, the action is good for the state. Plus, you should make a left here.
            & You should continue in the north direction. Yet, the action is not working for the state. Moreover, you should go slower. Also, you should take a left here.\\
        \midrule
        \textbf{\parking{}}
            & You should now gently park your car into the spot. Moreover, the action is helpful for the state. Furthermore, brake a bit. Also, steer a bit left.
            & You should first align your car to the parking spot. Even so, the action is not helpful for the state. Plus, slow down a bit. Moreover, steer a bit left.\\
        \midrule
        \textbf{\sweep{}}
            & You should go sideways to the goal. Plus, the action is good for the state. Yet, make a move to the front softly. Plus, please move downward strongly.
            & You should move to the cube. Yet, the action is bad for the state. Furthermore, the gripper should be opened. Plus, move toward the left firmly.\\
        \midrule
        \textbf{\boxclose{}}
            & You should fetch the lid. Furthermore, the action is useful for the state. Even so, you should move to the front calmly.
            & You need to raise the lid. Yet, the action is not good for the state. Plus, you should move to the front quickly. Moreover, you need to move to the top quickly.\\
        \midrule
        \textbf{\blockpushing{}}
            & Push the second block to the second target. Furthermore, the action is good for the state.
            & Move to the first block. Plus, the action is bad for the state. Plus, you should move to the front strongly.\\
        \midrule
        \textbf{\peginsertionside{}}
            & Insert the peg into the hole. Furthermore, the action is good for the state because the peg is being pushed into the hole. Also, take a move to the bottom gently.
            & Go to the peg. Yet, the action is not right for the state, since the gripper moves away from the peg. Plus, make a move toward the right gently. Furthermore, make a move to the bottom softly.\\
        \midrule
        \textbf{\hammer{}}
            & You should swing the hammer to the nail. Furthermore, the action is helpful for the state.
            & You need to move to the hammer. But, the action is not useful for the state. Also, you need to move to the back.\\
        \midrule
        \textbf{\relocate{}}
            & You need to go to the ball. Furthermore, the action is correct for the state.
            & You need to move to the ball. Even so, the action is not helpful for the state. Also, you need to move right. Furthermore, you should move toward the front. Plus, you need to reach down.\\
        \bottomrule
    \end{tabularx}
    \label{tab:lang_samples}
\end{table}

\begin{table}
    \centering
    \caption{\textbf{\Lang{} variant samples.} We listed expert and suboptimal samples for all tasks generated with \texttt{o4-mini} and other variants in \Cref{exp:lang_ablation_comp}.}
    \renewcommand{\arraystretch}{1.2}
    
    \begin{tabularx}{\textwidth}{l X X}
        \toprule
        \textbf{Task} & Expert Samples & Suboptimal Samples \\
        \midrule
        \textbf{\boxclose{} (\texttt{o4-mini})}
            & You’re in the lift stage and that last move was good. Now move the lid forward and to the right until it’s centered above the box opening.
            & You’re still trying to grab the lid, and lifting away was a bad move. Lower down and move forward and right to center on the handle. Then close your gripper to secure the lid.\\
        \midrule
        \textbf{\boxclose{} (\texttt{o4-mini} concise)}
            & Lower gripper straight onto handle, open fully and center, clamp down firmly, then lift up.
            & Maintain the handle alignment, lower gripper, then move forward and close. Last move was unhelpful.\\
        \midrule
        \textbf{\boxclose{} (\texttt{o4-mini} $\mu_g$-style)}
            & You should center the lid. Moreover, the action is good for the state. Yet, lower your gripper straight onto the handle. Also, lift up firmly.
            & You need to lower the gripper onto the handle and shift forward. However, the action is suboptimal for the state. Moreover, line up before closing the grip.\\
        \midrule
        \textbf{\boxclose{} ($\mu_g$-verbose)}
            & It's great that you're trying to retrieve the lid. This action will indeed help with the current situation. Just remember to approach it calmly and smoothly from the front. Keep up the good work!
            & It seems like you're trying to get the lid, but at the moment, that action doesn't seem to be advancing our goal. Let's adjust our approach here. Instead, move towards the left more decisively, and then slowly inch forward.\\
        \bottomrule
    \end{tabularx}
    \label{tab:lang_samples_variant}
\end{table}

%% file: tex/appendix/llm_captioner.tex
\section{\llm{} details}
\label{appx:llm_captioner_details}

\begin{figure}[t]
    \centering
        \centering
        \includegraphics[width=0.35\linewidth]{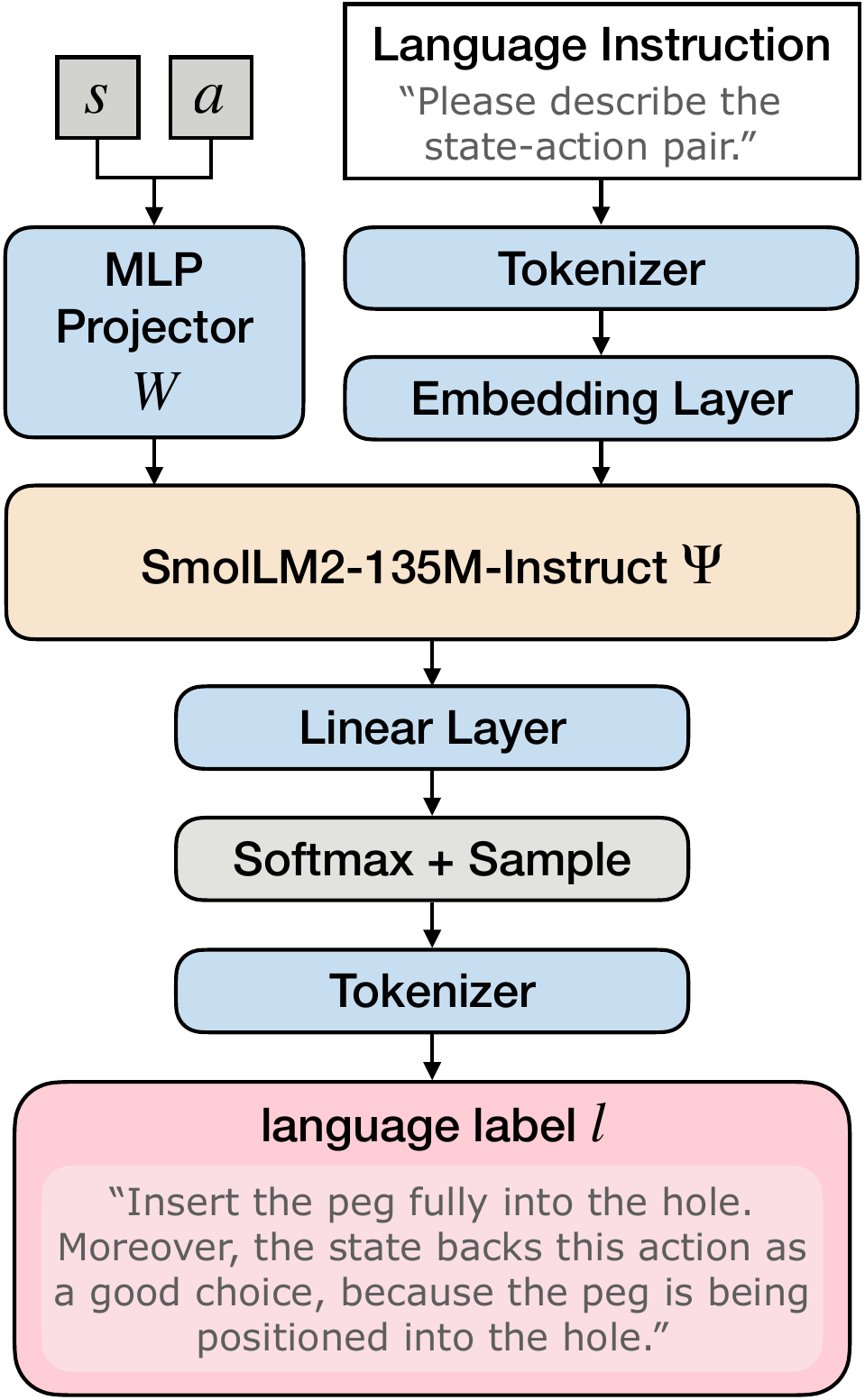}
        \caption{\textbf{Model Architecture of \llm{}.} The state-action pair $(s, a)$ is projected into the LLM’s hidden space as a single token via the MLP projector $W$. This token, concatenated with the language instruction embeddings, serves as input to the decoder transformer $\Psi$. The final output layer applies a softmax function to sample tokens, generating the resulting \lang{}s $l$.}
    \label{fig:llm_captioner_arch}
\end{figure}

In this section, we introduce the model architecture and the detailed training procedure of the \llm{} $\mu_\phi$.
Our objective is to enable end-to-end policy learning to optimize the language-critique imitation objective.
To achieve this, the differentiable \llm{} $\mu_\phi$ must distill the \lang{} generator $\mu_g$ and provide gradient signals that encourage the policy to match the objective.

The \llm{} $\mu_\phi$ consists of a pretrained transformer backbone $\Psi$—specifically SmolLM2-135M-Instruct~\cite{allal2025smollm2}—and an MLP projector $W$ that embeds state–action pairs into the hidden space of the transformer, as illustrated in \Cref{fig:llm_captioner_arch}. Further details for the model architecture are provided in \Cref{appx:model_architecture}. 
The resulting model defines a conditional distribution $\mu_{\phi}(\cdot \mid s_t, a_t)$ over \lang{}s, where $\phi = \{\Psi, W\}$ denotes all trainable parameters.

We finetune the \llm{} on the labeled general dataset $\mathcal{D}^\text{lang}_G$ using an instruction-tuning pipeline.
Training proceeds in two stages:
(1) we first update the MLP projector $W$ while keeping the LLM backbone frozen, and
(2) we finetune the entire \llm{} $\mu_\phi$ end-to-end, jointly updating both $W$ and $\Psi$.
In both stages, the objective is the cross-entropy loss in \Cref{eq:ce_loss}:
\[
    \mathbb{E}_{
        (s_t, a_t, l_t) \sim \mathcal{D}^\text{lang}_G
    }
    \left[
        \ell^\phi_\text{CE}(l_t, s_t,a_t)
    \right]
\]

The purpose of the finetuned \llm{} $\mu_\phi$ is to distill the generator $\mu_g$ for end-to-end policy optimization via gradient descent.
Then, the \lcloss{} constructed by \llm{} encourages the policy to produce actions that align with the expert policy's behavior with respect to the expert language-critique objective. We ablate backbone choices and finetuning modes in \Cref{appx:llm_ablation}.

%% file: tex/appendix/lcdp_details.tex
\section{\methoddp{} details}
\label{appx:lcdp_details}

In this section, we provide detailed derivations and implementation details for applying the proposed reweighted \lcloss{} $\tilde{\mathcal{L}}_\text{LC}$ to diffusion policies~\cite{chi2023diffusionpolicy} of our method \methoddp{}. 
We first review the diffusion model backbone underlying diffusion policies, including the forward diffusion process and the reverse generative process, in \Cref{appx:diffusion_forward_reverse}. 
Next, we describe the diffusion model training algorithm used for policy learning in \Cref{appx:diffusion_training}. 
Finally, in \Cref{appx:reweighting_factor_derive}, we derive the timestep-dependent reweighting factor $\omega^k$ for \methoddp{} from the diffusion training algorithm and provide experimental evidence justifying its design.

\subsection{Diffusion model: forward and reverse process}
\label{appx:diffusion_forward_reverse}

We first review the probabilistic formulation of diffusion models, which define paired forward and reverse Markov processes over a sequence of latent variables.
As illustrated in \Cref{fig:diffusion_forward_reverse}, diffusion models can be interpreted as latent variable models, where the data distribution over clean samples $x^0$ is augmented with a sequence of latent variables $\{x^k\}_{k=1}^K$.
The joint distribution is defined through a reverse generative process, while the forward process serves as an approximate posterior.

\begin{figure}[b]
    \centering
    \begin{minipage}{0.8\textwidth}
        \centering
        \includegraphics[width=0.8\linewidth]{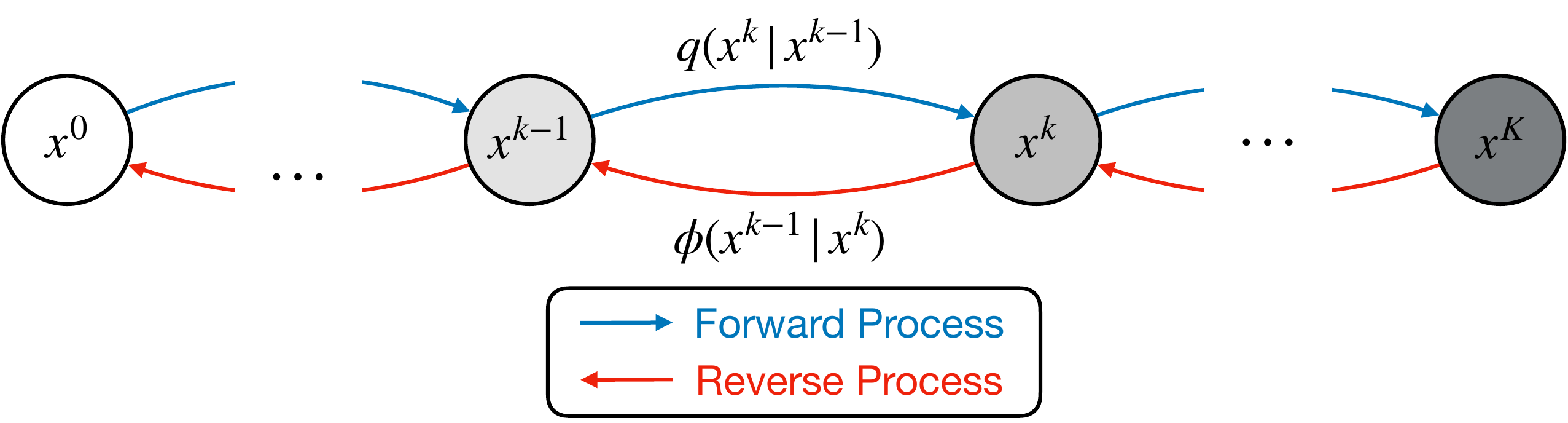}
        \caption{\textbf{Illustration of the probabilistic diffusion process.}
        Latent variables $x^1, \dots, x^K$ are generated from the clean data sample $x^0$ through a forward diffusion process defined by a probabilistic Markov chain $q(x^k \mid x^{k-1})$.
        The reverse process $\phi$ samples the latent variables from the previous ones to reconstruct the original data sample $x^0$ via $\phi(x^{k-1} \mid x^k)$.
        }
        \label{fig:diffusion_forward_reverse}
    \end{minipage}
\end{figure}

Concretely, in the case of denoising diffusion probabilistic model (DDPM)~\cite{ho2020denoising}, the forward diffusion process $q(x_{1:K} \mid x_0)$ is defined as a Markov chain that gradually adds Gaussian noise to the data according to a predefined variance schedule. The forward transitions take the form:
\begin{equation}
q(x^k \mid x^{k-1}) =
\mathcal{N}\!
    \left(
        x^k;\, \sqrt{\alpha^k} x^{k-1}, (1-\alpha^k)\mathbf{I}
    \right),
\end{equation}
where $\{\alpha^k\}_{k=1}^K$ schedules the noise magnitude at each timestep and $\bar{\alpha}^k = \prod_{i=1}^k \alpha^i$.
Due to the linear-Gaussian structure of this process, the latent variable $x^k$ admits a closed-form expression conditioned directly on the clean sample $x^0$:
\begin{equation}
    x^k = 
        \sqrt{\bar{\alpha}^k} x^0 
        + \sqrt{1 - \bar{\alpha}^k}\,\epsilon,
    \quad \epsilon \sim \mathcal{N}(0, \mathbf{I}).
\end{equation}
As the timestep $k$ increases, the signal-to-noise ratio decreases and $x^k$ approaches an isotropic Gaussian distribution.

The reverse diffusion process aims to invert the forward diffusion process in order to generatively sample clean data.
This is achieved by modeling the conditional distributions $\phi(x^k-1 \mid x^k)$, such that samples drawn from the resulting reverse Markov chain approximate the original data distribution.
However, directly parameterizing these reverse transitions is generally intractable.

To address this challenge, DDPMs adopt a denoising-based parameterization, in which a neural network is trained to predict the noise injected at each diffusion timestep.
Specifically, a model $\epsilon_\theta(x_k, k)$ is optimized to estimate the noise $\epsilon$, allowing recovery of the clean sample from its noisy latent representation.
This denoising formulation leads to a simple and stable training objective based on noise prediction, while implicitly defining the reverse generative process.

Diffusion policies~\cite{chi2023diffusionpolicy} adapt the denoising-based diffusion framework to continuous control by modeling a reverse generative process over actions.
As illustrated in \Cref{fig:dp_reverse_process}, conditioned on the current state $s_t$, the diffusion policy samples a clean action $\hat{a}_t^0$ starting from Gaussian noise $a^K_t\sim\mathcal{N}(0,\mathbf{I})$ and iteratively applying the learned reverse diffusion process.
At each diffusion timestep $k$, the policy predicts the noise component $\hat{\epsilon}\sim\epsilon_\theta(\cdot \mid s_t, a_t^k, k)$ based on the current state $s_t$, the intermediate noisy action $a_t^k$, and the timestep index $k$, and uses this prediction to produce the next action sample according to the reverse transition:
\begin{align}
    \hat{a}_t^{k-1} = \frac{1}{\sqrt{\alpha^k}}
        \left(
            a_t^k -
            \frac{1-\alpha^k}{\sqrt{1-\bar{\alpha}^k}}
            \epsilon_\theta(\cdot\mid s_t, a_t^k, k)
        \right)
    + \sigma^k z,
\end{align}
where $z \sim \mathcal{N}(0, \mathbf{I})$ for $k > 1$ and $z = 0$ for $k = 1$.
By iteratively applying this reverse update from $k=K$ to $k=1$, the diffusion policy progressively denoises the action sample, yielding a clean action $\hat{a}_t^0$.
This reverse-process formulation implicitly defines a conditional action distribution given the state.

\subsection{Diffusion model: training algorithm}
\label{appx:diffusion_training}

Training a diffusion policy reduces to learning a conditional denoising model from expert demonstrations.
Given a state-action pair $(s_t, a_t^0)$ sampled from the expert dataset $\mathcal{D}_E$, where $a_t^0$ denotes the clean expert action, the training procedure follows the forward--reverse diffusion formulation described in \Cref{appx:diffusion_forward_reverse}.

Specifically, a diffusion timestep $k$ is first sampled uniformly from $\{1, \dots, K\}$.
Gaussian noise $\epsilon \sim \mathcal{N}(0, \mathbf{I})$ is then drawn and applied to the clean action $a_t^0$ according to the forward diffusion process,
\begin{equation}
a_t^k = \sqrt{\bar{\alpha}^k} a_t^0 + \sqrt{1 - \bar{\alpha}^k}\,\epsilon,
\end{equation}
where $\bar{\alpha}^k = \prod_{i=1}^k \alpha^i$ denotes the cumulative noise schedule.
This operation yields a noisy action $a_t^k$ whose signal-to-noise ratio decreases as the timestep $k$ increases.

The diffusion policy $\epsilon_\theta$ is trained to predict the injected noise conditioned on the current state, the noisy action, and the diffusion timestep.
In practice, the model receives $(s_t, a_t^k, k)$ as input and outputs a noise estimate $\hat{\epsilon}\sim\epsilon_\theta(a_t^k, s_t, k)$.
The training objective minimizes the mean squared error between the predicted and true noise,
\begin{equation}
    \mathcal{L}_{\text{DP}} =
    \mathbb{E}_{(s_t, a^E_t) \sim \mathcal{D}_E,\, k,\, \epsilon\sim \mathcal{N}(0, \mathbf{I})}
        \left[
            \left\lVert 
                \epsilon - \epsilon_\theta(\cdot\mid s_t, a_t^k, k)
            \right\rVert^2
        \right].
\end{equation}
By training the policy to denoise actions at different noise levels, the model implicitly learns a conditional distribution over clean actions given the state.
At inference time, action generation begins from a Gaussian noise sample and iteratively applies the learned reverse process to produce a clean action conditioned on the current state, as illustrated in \Cref{fig:dp_reverse_process}.

\begin{figure}
    \centering
    \begin{minipage}{0.9\textwidth}
        \centering
        \includegraphics[width=0.6 \linewidth]{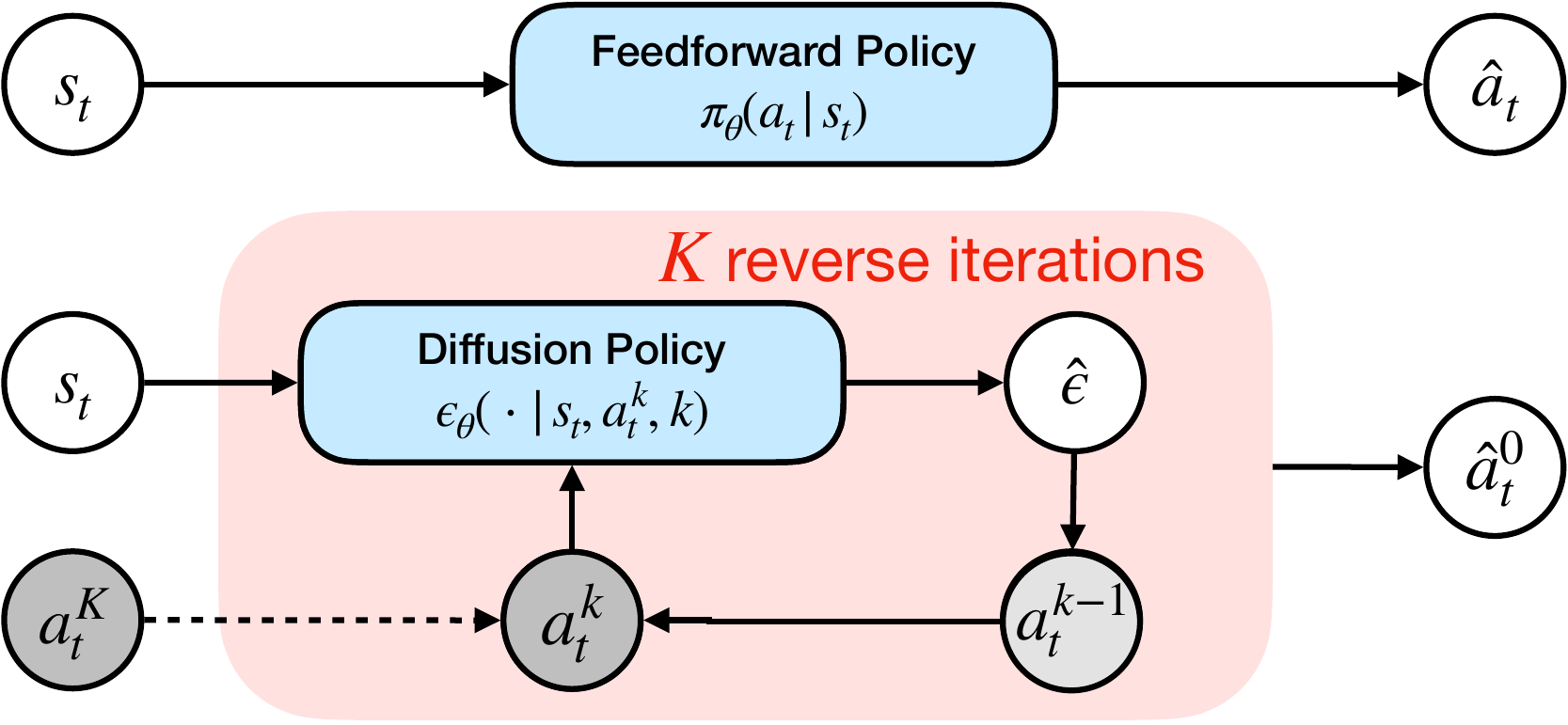}
        \caption{\textbf{Action sampling process of a diffusion policy.}
        Starting from Gaussian noise $a^K_t$, the diffusion policy iteratively denoises actions conditioned on the current state $s_t$, producing a clean action sample $\hat{a}^0_t$ through the learned reverse diffusion process. On the other hand, the feedforward policy predicts an action $\hat{a}_t$ with a single-step inference conditioned on the state.}
        \label{fig:dp_reverse_process}
    \end{minipage}
\end{figure}

\subsection{Reweighting factor derivation}
\label{appx:reweighting_factor_derive}

To apply \lang{} supervision to diffusion policies, we leverage one-step reconstruction to obtain a clean action estimate from the predicted noise and compute the \lcloss{} on the reconstructed action.
Concretely, although the diffusion policy is trained to predict noise, the predicted noise $\hat{\epsilon}\sim\epsilon_\theta$ can be converted into a reconstructed clean action $\hat{a}_t^0$ via one-step reconstruction. This allows us to reuse the same \lcloss{} formulation as in feedforward policies.

A straightforward approach is to directly apply the \lcloss{} on the reconstructed action without any additional design. However, as shown in \Cref{fig:reweighting_mse_inspect_perf}, this naive strategy yields limited improvement compared to feedforward policies, suggesting a mismatch between diffusion training dynamics and the action-level supervision induced by the \lcloss{}.

To better understand this behavior, we analyze the properties of one-step reconstruction across diffusion timesteps. Empirically, we observe that the quality of the reconstructed action $\hat{a}_t^0$ strongly depends on the diffusion timestep $k$. Reconstructions from small diffusion steps (small $k$) are generally accurate, while reconstructions from large diffusion steps (large $k$) are substantially noisier. This introduces an inherent imbalance: errors from large-$k$ steps dominate the action-level objective, while informative signals from small-$k$ steps are underrepresented.

The standard diffusion policy training objective optimizes noise prediction via the mean squared error
\begin{align}
    \mathcal{L}_\text{DP} =
    \mathbb{E}_{(s_t, a^E_t)\sim\mathcal{D}_E, \epsilon\sim \mathcal{N}(0,\mathbf{I}),k}
    \left[
        \left\lVert
           \epsilon - \epsilon_\theta(\cdot\mid s_t, a_t^k, k)
        \right\rVert^{2}
    \right]
\end{align}
as defined in \Cref{eq:dp_loss}. In contrast, when applying supervision on reconstructed actions, the objective becomes an action-level distance:
\begin{align}
    \mathcal{L}^a_\text{DP}=
    \mathbb{E}_{(s_t, a^E_t)\sim\mathcal{D}_E, \epsilon\sim \mathcal{N}(0,\mathbf{I}),k}
    \left[
        \left\lVert
            a_t^0 - \hat{a}_t^0
        \right\rVert^{2}
    \right]
\end{align}
where $\hat{a}_t^0$ is obtained via one-step reconstruction in \Cref{eq:dp_1step_reconstruct}. Expanding this action-level objective yields:

\begin{align}
    &\left\lVert
        a_t^0 - \hat{a}_t^0
    \right\rVert^{2}\\
    = &\left\lVert
        a_t^0 - \left(
            \frac{1}{\sqrt{\bar{\alpha}^k}}a^k_t 
            - \frac{\sqrt{1-\bar{\alpha}^k}}{\sqrt{\bar{\alpha}^k}}
            \epsilon_\theta(\cdot\mid s_t, a_t^k, k)
        \right)
    \right\rVert^{2}\\
    = &\left\lVert
        a_t^0 - \left(
            \frac{1}{\sqrt{\bar{\alpha}^k}}
                \left(
                    \sqrt{\bar{\alpha}^k}a^0_t 
                    + \sqrt{1-\bar{\alpha}^k} \epsilon 
                \right)
            - \frac{\sqrt{1-\bar{\alpha}^k}}{\sqrt{\bar{\alpha}^k}}
            \epsilon_\theta(\cdot\mid s_t, a_t^k, k)
        \right)
    \right\rVert^{2}\\
    = &\left\lVert
        a_t^0 - \left(
            a^0_t 
            + \frac{\sqrt{1-\bar{\alpha}^k}}{\sqrt{\bar{\alpha}^k}} \epsilon 
            - \frac{\sqrt{1-\bar{\alpha}^k}}{\sqrt{\bar{\alpha}^k}}
            \epsilon_\theta(\cdot\mid s_t, a_t^k, k)
        \right)
    \right\rVert^{2}\\
    = &\left\lVert
        a_t^0 - a^0_t 
        +  \frac{\sqrt{1-\bar{\alpha}^k}}{\sqrt{\bar{\alpha}^k}} \left(
            \epsilon 
            - \epsilon_\theta(\cdot\mid s_t, a_t^k, k)
        \right)
    \right\rVert^{2}\\
    = &\frac{1-\bar{\alpha}^k}{\bar{\alpha}^k}
        \left\lVert
            \epsilon 
            - \epsilon_\theta(\cdot\mid s_t, a_t^k, k)
        \right\rVert^{2}
    \label{eq:dp_action_loss}
\end{align}

showing that the action reconstruction error is equivalent to the original noise-prediction error weighted by a timestep-dependent factor $\frac{1-\bar{\alpha}^k}{\bar{\alpha}^k}$.

\input{figures/reweighting_factor_plot}

As illustrated in \Cref{fig:reweighting_factor_plot}, this weighting scheme disproportionately amplifies contributions from large diffusion $k$-steps at the expense of small-$k$ steps. Directly optimizing this action-level objective introduces unstable training loss convergence; as seen in Figure \ref{fig:reweighting_mse_inspect_loss}, the loss exhibits stronger fluctuation, leading to performance degradation. While the resulting policy remains functional, it ultimately underperforms compared to a standard diffusion policy trained with the original noise-prediction loss (Figure \ref{fig:reweighting_mse_inspect_perf}).

The above analysis reveals that the action-level objective $\mathcal{L}^a_\text{DP}$ implicitly reweights the original diffusion objective $\mathcal{L}_\text{DP}$. To recover the balanced training behavior of the standard diffusion loss with one-step reconstructed actions $\hat{a}^0_t$, we apply the reciprocal of this $k$-dependent weight $\frac{1-\bar{\alpha}^k}{\bar{\alpha}^k}$.Formally, we define the reweighting factor
\begin{align}
    \omega^k \coloneq \frac{\bar{\alpha}^k}{1-\bar{\alpha}^k}.
\end{align}
Applying this factor to the action-level objective yields the reweighted action loss
\begin{align}
    \tilde{\mathcal{L}}^a_\text{DP}&=
    \mathbb{E}_{(s_t, a^E_t)\sim\mathcal{D}_E, \epsilon\sim \mathcal{N}(0,\mathbf{I}),k}
    \left[
        \frac{\bar{\alpha}^k}{1-\bar{\alpha}^k}
        \left\lVert
            a_t^0 - \hat{a}_t^0
        \right\rVert^{2}
    \right]\\
    &=
    \mathbb{E}_{(s_t, a^E_t)\sim\mathcal{D}_E, \epsilon\sim \mathcal{N}(0,\mathbf{I}),k}
    \left[
        \left\lVert
           \epsilon - \epsilon_\theta(\cdot\mid s_t, a_t^k, k)
        \right\rVert^{2}
    \right]\\
    &=\mathcal{L}_\text{DP}
    \label{eq:dp_reweighted_action_loss}
\end{align}
showing that the reweighted action objective $\tilde{\mathcal{L}}^a_\text{DP}$ is exactly equivalent to the original noise-prediction objective $\mathcal{L}_\text{DP}$.

\input{figures/reweighting_mse_inspect}

We empirically validate this equivalence by training diffusion policies under three objectives: the standard noise-prediction loss $\mathcal{L}_\text{DP}$ (red), the unweighted action-level loss $\mathcal{L}^a_\text{DP}$ (gray), and the reweighted action-level loss $\tilde{\mathcal{L}}^a_\text{DP}$ (blue).
As shown in \Cref{fig:reweighting_mse_inspect_loss} and \Cref{fig:reweighting_mse_inspect_perf}, applying the reweighting factor $\omega^k$ aligns both the training loss trajectories and the final performance with those of $\mathcal{L}_{\text{DP}}$, whereas the unweighted action-level objective exhibits noticeable discrepancies.
These results confirm the effect and necessity of the proposed reweighting factor $\omega^k$.

With this derivation, we apply the same reweighting factor $\omega^k$ to the \lcloss{} when training diffusion policies. Intuitively, this reweighting compensates for the $k$-dependent reconstruction bias: without reweighting, inaccuracies at large diffusion steps are overemphasized, while informative gradients from small diffusion steps, where the signal-to-noise ratio is high, are diminished.

This behavior is visualized in \Cref{fig:reweighting_factor_plot_alphas} and \Cref{fig:reweighting_factor_plot_w&reci}. The original $k$-dependent weight $\frac{1-\bar{\alpha}^k}{\bar{\alpha}^k}$ (orange curve in \Cref{fig:reweighting_factor_plot_alphas}) exponentially suppresses contributions from small diffusion steps while amplifying those from large steps.
In contrast, its reciprocal $\omega^k$ (blue curve in \Cref{fig:reweighting_factor_plot_alphas}) restores balance by emphasizing small-$k$ steps—where the signal-to-noise ratio is high—and down-weighting large-$k$ steps.

%% file: figures/reweighting_factor_plot.tex
\begin{figure}[H]
\centering
    \begin{minipage}{0.9\textwidth}
    \centering
        \begin{subfigure}[b]{0.475\textwidth}
            \centering
            \includegraphics[width=\textwidth]{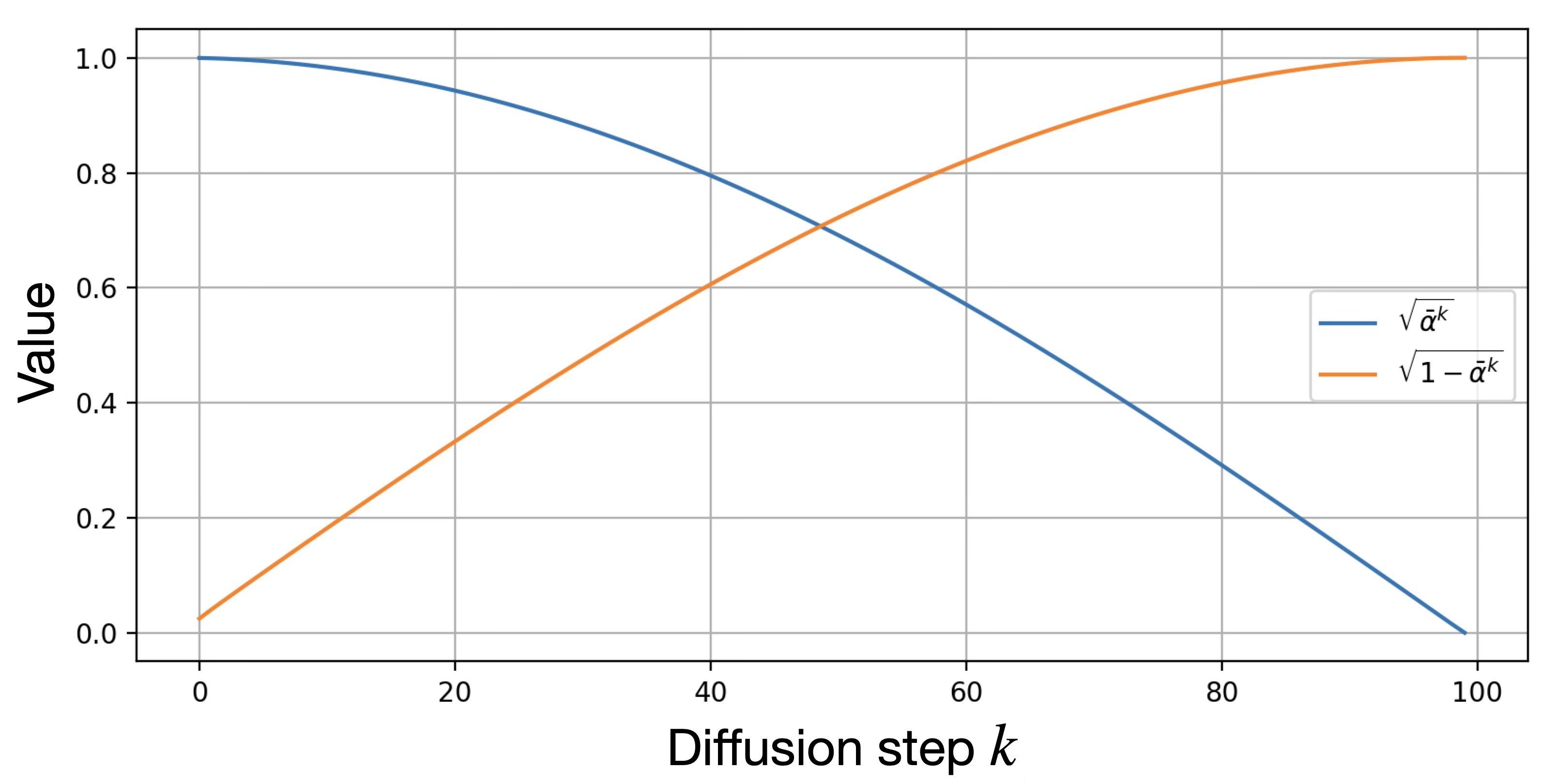}
            \caption{Diffusion schedule components over timesteps $k$: $\sqrt{\bar{\alpha}^k}$ (signal term) and $\sqrt{1-\bar{\alpha}^k}$ (noise term).}
            \label{fig:reweighting_factor_plot_alphas}
        \end{subfigure}
        \hfill
        \begin{subfigure}[b]{0.475\textwidth}
            \centering
            \includegraphics[width=\textwidth]{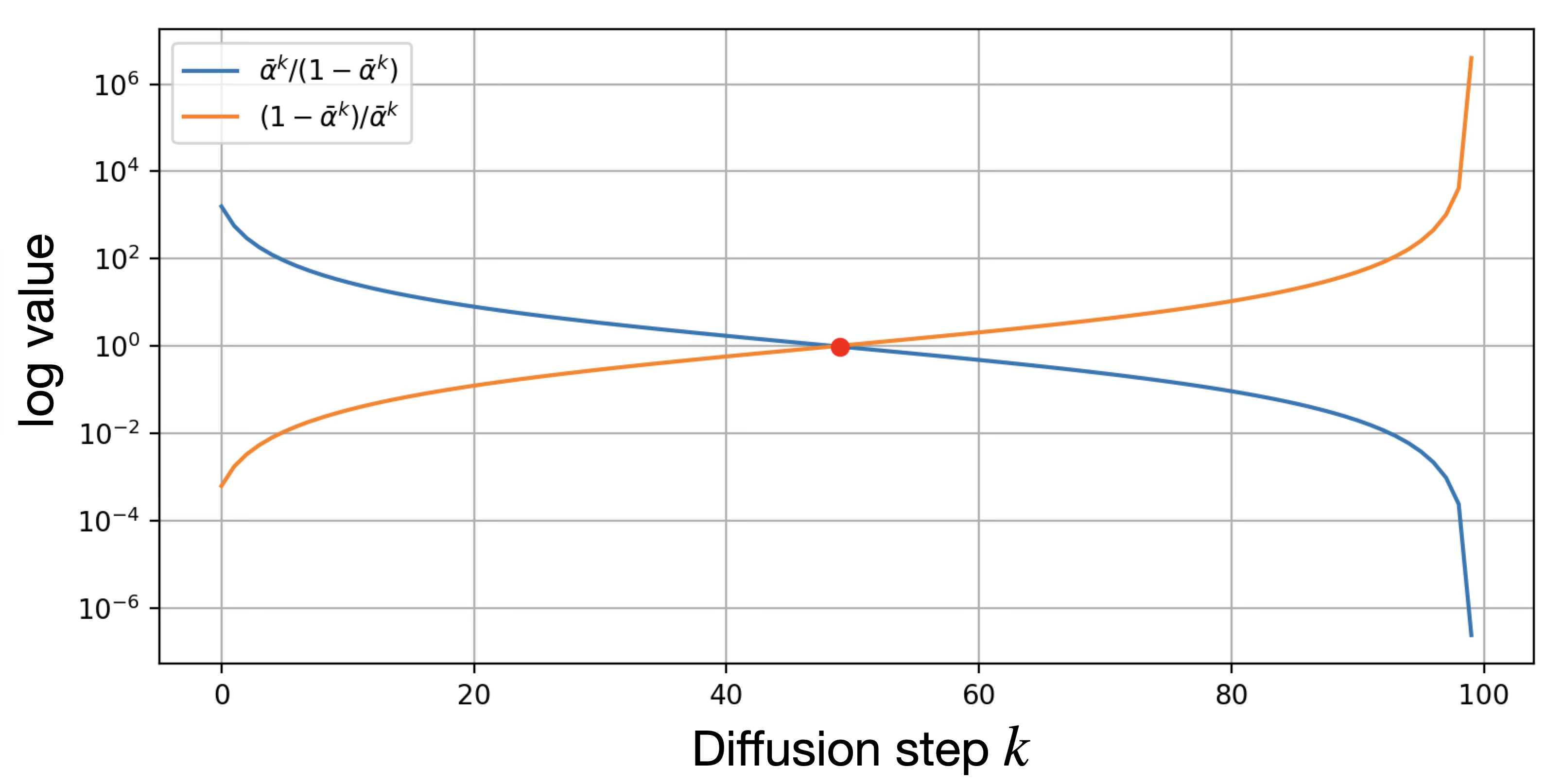}
            \caption{$k$-dependent weighting induced by one-step reconstruction: $\frac{1-\bar{\alpha}^k}{\bar{\alpha}^k}$, and its reciprocal reweighting factor 
            $\omega^k$.}            
            \label{fig:reweighting_factor_plot_w&reci} 
        \end{subfigure}
    \caption{\textbf{Diffusion schedules and $k$-dependent reweighting.}}
    \label{fig:reweighting_factor_plot}
    \end{minipage}
\end{figure}

%% file: figures/reweighting_mse_inspect.tex
\begin{figure}[H]
    \centering
    \begin{minipage}{0.9\textwidth}
            \begin{subfigure}[b]{0.495\textwidth}
                \centering
                \includegraphics[width=\textwidth]{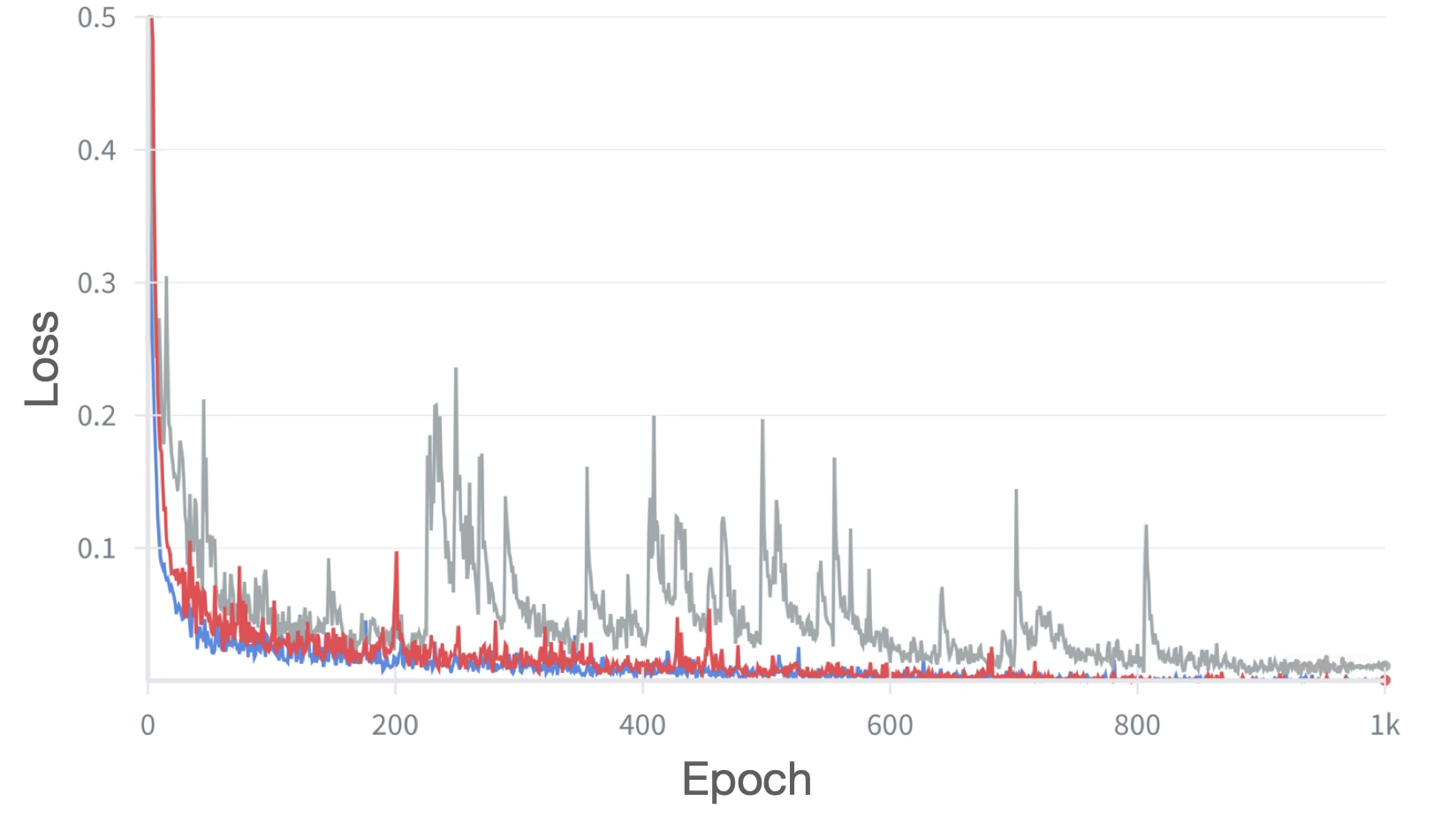}
                \caption{Training MSE loss curves}
                \label{fig:reweighting_mse_inspect_loss}
            \end{subfigure}
            \hfill
            \begin{subfigure}[b]{0.495\textwidth}
                \centering
                \includegraphics[width=\textwidth]{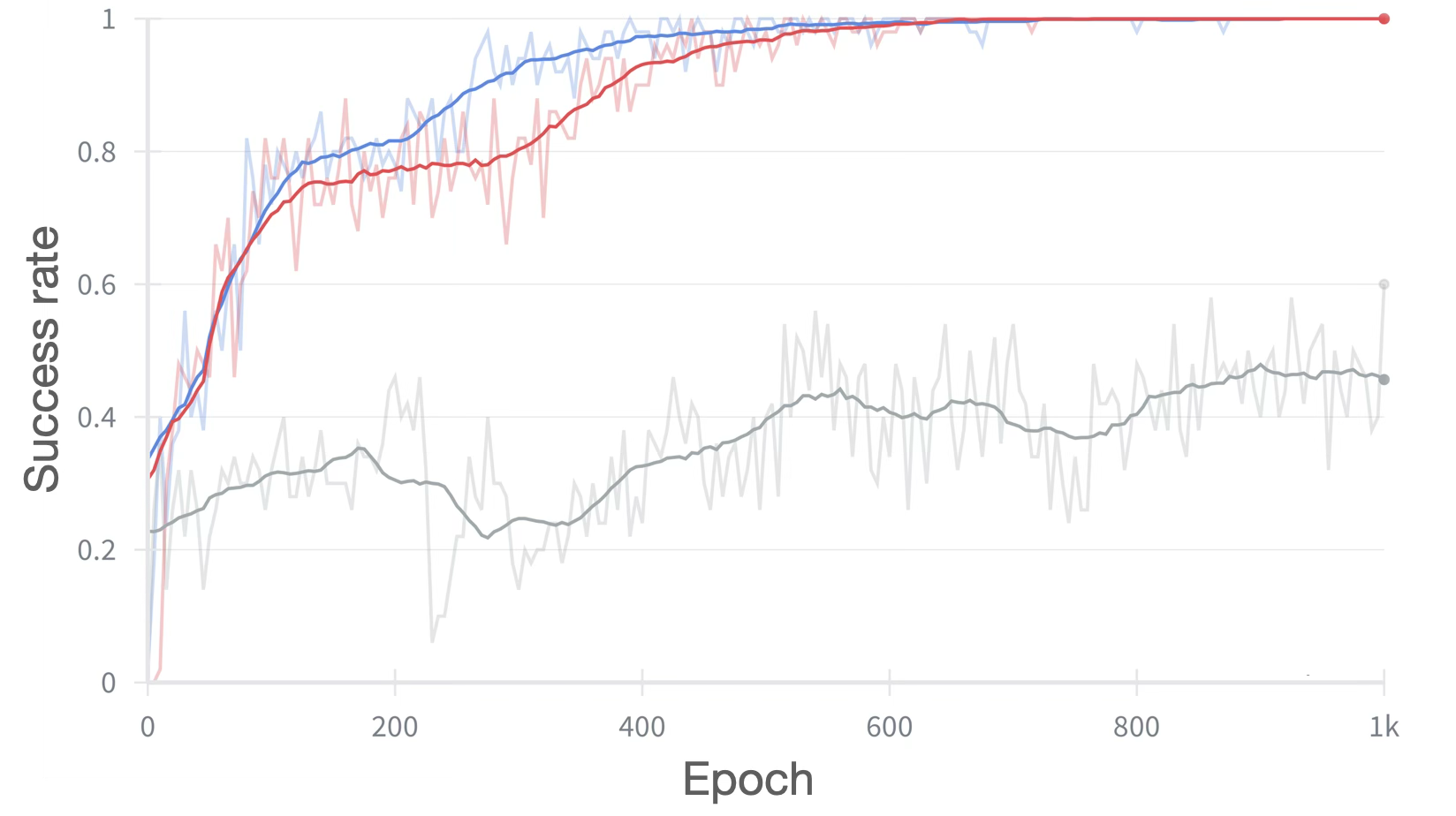}
                \caption{Success rate curves}
                \label{fig:reweighting_mse_inspect_perf} 
            \end{subfigure}
    
        \caption[]{\textbf{Diffusion policy training with different objectives. }
        \textcolor{red}{Red}: standard noise-prediction loss $\mathcal{L}_{\text{DP}}$; 
        \textcolor{gray}{gray}: unweighted action-level loss $\mathcal{L}^a_{\text{DP}}$; 
        \textcolor{blue}{blue}: reweighted action-level loss $\tilde{\mathcal{L}}^a_{\text{DP}}$.
        \textbf{Left}: training MSE, where $\mathcal{L}^a_{\text{DP}}$ exhibits large fluctuations due to timestep $k$ imbalance, while the reweighted and original objectives converge stably.
        \textbf{Right}: success rate (evaluated every 5 epochs), showing that reweighting restores near-optimal performance.}
        \label{fig:reweighting_mse_inspect}
    \end{minipage}
\end{figure}

%% file: tex/appendix/env_details.tex
\clearpage
\section{Environment and tasks details}
\label{appx:env_details}
We list the details of the task environments of Section \ref{exp} here. In Table \ref{tab:dataset}, we list the details of expert and general datasets we used for experiments. It comprises the number of trajectories, expert steps, general steps, and the collecting methods. We also provide the environment names, version numbers, the dimensions of state spaces, and action spaces of each task in table \ref{tab:environments}.

\input{tables/dataset}
\input{tables/env_table}

\paragraph{\textbf{\maze{}.}}
A 2D navigation task from \cite{fu2020d4rl} in which a point agent learns to navigate from an initial position to a goal position within the medium maze. The 2-DoF agent controls its motion by iteratively applying linear forces along $x$ and $y$ directions. The 6D states include the agent’s two-dimensional current position, velocity, and the goal position. The initial and goal positions of the agent are randomly sampled when an episode is initialized. The maximum episode length for this task is set at 250; termination occurs if the goal is reached earlier.

\paragraph{\textbf{\parking{}.}}
A goal-conditioned continuous control driving task from \cite{highway-env}, where the vehicle is asked to park in the designated target spot within a parking lot. The 2-DoF agent navigates the car by controlling the acceleration and steering angle. The 18D observation space includes the current position, the longitudinal velocity, the heading angle of the car, along with the achieved goal state, and the desired target configuration. The dynamics are kinematic bicycle constraints. The maximum episode length for this task is set at 80; termination occurs if the vehicle reaches the target position earlier.

\paragraph{\textbf{\sweep{}.}}

A manipulation task from Metaworld~\cite{yu2019meta} where the agent has to sweep an object off the table. It requires the agent to maintain continuous contact with the object over a long trajectory to the off-table goal. The agent controls the position of the end-effector and the gripper actuation via a 4D continuous action space. The 9D observation states provide the $x$, $y$, $z$ position of the robot arm, the manipulation block, and the goal target. The maximum episode length for this task is set at 30, and termination occurs if the object is swept off the table.

\paragraph{\textbf{\boxclose{}.}}
A contact-rich manipulation task from Metaworld~\cite{yu2019meta} in which the agent should grasp the cover and close the box with it. The robot arm should precisely align the position of the cover, grasp the handle, hold it towards the opened box, and put it in a closed position. The agent controls the position of the end-effector and the gripper actuation via a 4D continuous action space. The 9D observation states provide the $x$, $y$, $z$ position of the robot arm, the manipulation block, and the target. The maximum episode length for this task is set at 30, and termination occurs if the box is closed with the cover earlier.

\paragraph{\textbf{\blockpushing{}.}}
A multistage table-top manipulation task from \cite{florence2021implicit} where a 6-DoF robot arm should push two blocks into two target zones in any order. It is a multimodal task that allows a robot arm to first move one randomly selected block to one target zone, then move the second block to the other zone. The end-effector is controlled via absolute $x$ and $y$ pose on the 2D plane. The 16D observation space encapsulates the full system state, including the current end-effector pose, the position and orientation of both blocks and target zones. The maximum episode length for this task is set at 350, and termination occurs if two blocks are pushed into their corresponding target zones.

\paragraph{\textbf{\peginsertionside{}.}}
A high-precision manipulation task from ManiSkill-3~\cite{tao2025maniskill3} where a 7-DoF robot arm equipped with a gripper must grasp a peg from a table and insert it into a horizontal box hole. This task poses accurate control challenges due to the tight clearance between the peg and the hole. The agent operates within a 43D state space, which comprises joint positions and velocities of the robot arm, as well as the poses of both the peg and the target hole. The maximum episode length for this task is set at 100, and termination occurs if the end of the peg is successfully inserted within 0.015m of the center of the box.

\paragraph{\textbf{\hammer{}.}}
A high-dimensional dexterous manipulation task from Gymnasium-Robotics~\cite{gymnasium_robotics2023github} that requires a 28-DoF system, consisting of a 24-DoF ShadowHand and a 4-DoF arm, to pick up a hammer and drive a nail into a board. The environment features a randomized nail position with dry friction capable of absorbing up to 15N of force, requiring the agent to exert significant physical impact. The agent operates within a 26D continuous action space representing the scaled absolute angular positions of the joints, including the arm, wrist, and all finger segments. The 46D observation space tracks joint angles, palm and nail positions, hammer pose and velocities, nail displacement, and impact force. Success is achieved when the entire length of the nail is driven into the board, and the episode is truncated after a maximum of 200 steps.

\paragraph{\textbf{\relocate{}.}}
A complex manipulation task from Gymnasium-Robotics~\cite{gymnasium_robotics2023github} that utilizes a 30-DoF system, comprising a 24-DoF ShadowHand mounted on a 6-DoF robotic arm. The objective is to pick up a blue ball from a tabletop and move it to a randomly initialized green target in the 3D workspace. The agent utilizes a 30D action space to control arm translation, orientation, and hand joint positions. The 39D observation space tracks joint states and relative 3D vectors between the palm, ball, and target. This task is challenging due to the randomization of both the object and the target, requiring the agent to master a multi-stage behavior: approaching the object, achieving a stable grasp, and precisely transporting the object through free space to an aerial target. The maximum episode length for this task is set at 200, and termination occurs if the blue ball is successfully moved within an epsilon-ball of the target.

%% file: tables/dataset.tex
\begin{table}[b]
\centering
\caption{\textbf{Dataset.} Detailed information of expert dataset and general dataset in each task.}
\label{tab:dataset}
    \resizebox{\textwidth}{!}{
        \begin{tabular}{l c c c c}
        \toprule
        \textbf{Task} & \# of expert trajectories & \# of expert steps & \# of general steps & Algorithm/Retrieved Sources  \\
        \midrule
        \textbf{\maze{}}      
            & 200
            & \num{26037}     
            & \num{153177}
            &  D4RL controller \cite{fu2020d4rl}
            \\
        \textbf{\parking{}}       
            & 500
            & \num{10693}     
            & \num{73604}
            & SAC-trained policy expert ~\cite{haarnoja2018soft}
            \\
        \textbf{\sweep{}}     
            & 500
            & \num{5035}     
            & \num{73485}
            &  Metaworld controller \cite{yu2019meta}
            \\
        \textbf{\boxclose{}}         
            & 500
            &  \num{4214}     
            &  \num{76031}
            &  Metaworld controller \cite{yu2019meta}
            \\
        \textbf{\blockpushing{}}  
            & 500
            & \num{37994}     
            & \num{301876}
            &  default controller \cite{florence2021implicit}
            \\
        \textbf{\peginsertionside{}}     
            & 500
            & \num{18589}     
            & \num{246576}
            & PPO-trained policy~\cite{schulman2017proximal}
            \\
        \textbf{\hammer{}}     
            & 10
            & \num{439}     
            & \num{24752}
            & DAPG-trained policy~\cite{rajeswaran2018learning}
            \\
        \textbf{\relocate{}}     
            & 20
            & \num{1118}     
            & \num{30136}
            & DAPG-trained policy~\cite{rajeswaran2018learning}
            \\
        \bottomrule
        \end{tabular}
    }
\end{table}

%% file: tables/env_table.tex
\begin{table}
\centering
\caption{\textbf{Environments.} Detailed environment setting of each task.}
\label{tab:environments}
    \resizebox{\textwidth}{!}{
        \begin{tabular}{l l c c c}
        \toprule
        \textbf{Task} & Environment & ID & State space & Action space  \\
        \midrule
        \textbf{\maze{}}      
            & D4RL~\cite{fu2020d4rl}   
            & pointmaze-maze2d-medium-v0         
            & 6 
            & 2  \\
        \textbf{\parking{}}       
            & Highway-Env~\cite{highway-env}       
            & highway-parking-v0        
            & 18 
            & 2  \\
        \textbf{\sweep{}}     
            & Metaworld~\cite{yu2019meta}             
            & metaworld-sweep-v2            
            & 9
            & 4  \\
        \textbf{\boxclose{}}         
            & Metaworld~\cite{yu2019meta}                   
            & metaworld-box-close-v2           
            & 9 
            & 4  \\
        \textbf{\blockpushing{}}  
            & BlockPush~\cite{florence2021implicit}          
            & blockpushing-BlockPushMultimodal-v0         
            & 16
            & 2  \\
        \textbf{\peginsertionside{}}     
            & Maniskill~\cite{tao2025maniskill3}          
            & maniskill-PegInsertionSide-v1
            & 43  
            & 8 \\
        \textbf{\hammer{}}     
            & Gymnasium-Robotics~\cite{gymnasium_robotics2023github} 
            & adroit-hand-hammer-v1
            & 46  
            & 26 \\
        \textbf{\relocate{}}     
            & Gymnasium-Robotics~\cite{gymnasium_robotics2023github} 
            & adroit-hand-relocate-v1
            & 39  
            & 30 \\
        \bottomrule
        \end{tabular}
    }
\end{table}

%% file: tex/appendix/training_details.tex
\section{Training details}
\label{appx:training_details}

Our implementation is built on top of the Diffusion Policy codebase~\cite{chi2023diffusionpolicy}. 
In this section, we describe the model architectures used for all baselines and our proposed methods, along with the corresponding training hyperparameters and computational resources.

\subsection{Model architectures}
\label{appx:model_architecture}
Table~\ref{tab:model_architecture_part1} and \ref{tab:model_architecture_part2} summarize the detailed architectural configurations for all baselines.

\input{tables/model_architecture}

\clearpage
\subsection{Hyperparameters}
\label{appx:hyperparameters}

\mytable{tab:hyperparameters} lists the training hyperparameters used for all baselines and our proposed methods.
Unless otherwise specified, we set the gradient accumulation step to 1.
For \method{} and \methoddp{}, we explicitly increase the gradient accumulation steps to accommodate the additional memory cost introduced by LLM-based computations, allowing us to use smaller per-step batch sizes within our computational constraints.

\input{tables/hyperparameters}

\subsection{Computational resources}
\label{appx:comp_resource}

\mytable{tab:comp_resource} reports the computational resources we have used for our experiments, including the workstation configurations.

\input{tables/comp_resource}

%% file: tables/model_architecture.tex
\begin{table}[t]
\centering
\caption{\textbf{Model architectures (Part I):} \maze, \parking, \sweep, and \boxclose.}
\label{tab:model_architecture_part1}
\resizebox{0.89\textwidth}{!}{
\begin{tabular}{clcccc}
\toprule
\textbf{Method} & \textbf{Models}  & \textsc{Maze} & \textsc{Parking}  & \textsc{Sweep} & \textsc{Box-close} \\
\midrule
    \multirow{3}{*}{BC} 
         & Network Type              
             & MLP & MLP & MLP & MLP \\
         & Hidden Dimension      
             & [256, 256, 256, 256] & [256, 256] & [256, 256] & [256, 256] \\
         & Activation Function           
             & tanh & ReLU & ReLU & ReLU \\
\midrule
    \multirow{4}{*}{DWBC}
        & Network Type    & MLP & MLP & MLP & MLP \\
        & Hidden Dimension     & [256, 256, 256, 256] & [256, 256] & [256, 256] & [256, 256] \\
        & Activation Function & tanh & ReLU & ReLU & ReLU \\
        & Discriminator Hidden Dimension & 256 & 256 & 256 & 256 \\
\midrule

    \multirow{9}{*}{DemoDICE}
        & Network Type  
            & MLP 
            & MLP
            & MLP
            & MLP \\
        & Hidden Dimension  
            & [256, 256, 256, 256]
            & [256, 256]
            & [256, 256]
            & [256, 256]
            \\
        & Activation Function     
            & tanh
            & ReLU
            & ReLU
            & ReLU
        \\
    \cmidrule{2-6}
        & Discriminator Network Type
            & MLP
            & MLP
            & MLP
            & MLP
        \\
        & Discriminator Hidden Dimension 
            & [256, 256]
            & [256, 256]
            & [256, 256]
            & [256, 256]
        \\
        & Discriminator Activation Function 
            & ReLU 
            & ReLU 
            & ReLU 
            & ReLU 
        \\
    \cmidrule{2-6}
        & Nu Network Type
            & MLP
            & MLP
            & MLP
            & MLP
        \\
        & Nu Hidden Dimension 
            & [256, 256]
            & [256, 256]
            & [256, 256]
            & [256, 256]
        \\
        & Nu Activation Function 
            & ReLU 
            & ReLU 
            & ReLU 
            & ReLU 
        \\
\midrule

    \multirow{6}{*}{ILID}
        & Network Type  
            & MLP 
            & MLP
            & MLP
            & MLP \\
        & Hidden Dimension  
            & [256, 256, 256, 256]
            & [256, 256]
            & [256, 256]
            & [256, 256]
            \\
        & Activation Function     
            & tanh
            & ReLU
            & ReLU
            & ReLU
        \\
    \cmidrule{2-6}
        & Discriminator Network Type
            & MLP
            & MLP
            & MLP
            & MLP
        \\
        & Discriminator Hidden Dimension 
            & [256]
            & [256]
            & [256]
            & [256]
        \\
        & Discriminator Activation Function 
            & ReLU 
            & ReLU 
            & ReLU 
            & ReLU 
        \\
\midrule

    \multirow{6}{*}{\rbc{}} 
        & $R_\phi$ Model Type            
            & MLP 
            & MLP 
            & MLP 
            & MLP 
            \\
        & $R_\phi$ Hidden Dimension            
            & [256, 256, 256] 
            & [256, 256, 256] 
            & [256, 256, 256] 
            & [256, 256, 256] 
            \\
        &$R_\phi$ Activation Function            
            & ReLU 
            & ReLU 
            & ReLU 
            & ReLU 
            \\
\cmidrule{2-6}
        & Network Type            
             & MLP 
             & MLP 
             & MLP 
             & MLP 
            \\
         & Hidden Dimension      
             & [256, 256, 256, 256] 
             & [256, 256] 
             & [256, 256] 
             & [256, 256] 
             \\
         & Activation Function           
             & tanh 
             & ReLU 
             & ReLU 
             & ReLU \\
\midrule

    \multirow{6}{*}{\methodcls{}} 
        & $C_\phi$ Model Type            
            & MLP 
            & MLP 
            & MLP 
            & MLP 
            \\
        & $C_\phi$ Hidden Dimension            
            & [128, 128]
            & [256, 256]
            & [256, 256]
            & [256, 256]
            \\
        &$C_\phi$ Activation Function            
            & ReLU 
            & ReLU 
            & ReLU 
            & ReLU 
            \\
\cmidrule{2-6}
        & Network Type            
             & MLP 
             & MLP 
             & MLP 
             & MLP 
            \\
         & Hidden Dimension      
             & [256, 256, 256, 256] 
             & [256, 256]
             & [256, 256] 
             & [256, 256] 
             \\
         & Activation Function           
             & tanh 
             & ReLU
             & ReLU 
             & ReLU \\
\midrule

    \multirow{3}{*}{\begin{tabular}{l}
        \method{} \\(Ours) 
    \end{tabular}} 
        & Network Type            
             & MLP 
             & MLP 
             & MLP 
             & MLP\\
         & Hidden Dimension      
             & [256, 256, 256, 256] 
             & [256, 256] 
             & [256, 256] 
             & [256, 256]\\
         & Activation Function           
             & tanh 
             & ReLU 
             & ReLU 
             & ReLU \\

\bottomrule

    \multirow{2}{*}{DP} 
         & Network Type   
             & Unet 
             & Unet 
             & Unet 
             & Unet \\
         & Hidden Dimension      
             & [256, 512, 1024] 
             & [256, 512, 1024] 
             & [256, 512, 1024] 
             & [256, 512, 1024] \\
\midrule

    \multirow{3}{*}{LPB-Offline} 
        & Dynamics Network Type            
             & MLP 
             & MLP 
             & MLP 
             & MLP \\
         & Dynamics Hidden Dimension      
             & [256, 256] 
             & [256, 256] 
             & [256, 256] 
             & [256, 256] \\
         & Dynamics Activation Function           
             & ReLU
             & ReLU
             & ReLU
             & ReLU \\
\midrule

    \multirow{5}{*}{EDP} 
 
         & Critic Network Type
             & MLP 
             & MLP 
             & MLP 
             & MLP
             \\
        & Hidden Dimension      
             & [256, 256] 
             & [256, 256] 
             & [256, 256]
             & [256, 256] \\
         & Critic Activation Function           
             & ReLU 
             & ReLU
             & ReLU 
             & ReLU \\
 
\cmidrule{2-6}
         & Network Type   
             & Unet 
             & Unet
             & Unet 
             & Unet \\
         & Hidden Dimension      
             & [256, 512, 1024] 
             & [256, 512, 1024] 
             & [256, 512, 1024] 
             & [256, 512, 1024] \\
\midrule

    \multirow{5}{*}{\rdp{}} 
        & $R_\phi$ Model Type            
            & MLP 
            & MLP 
            & MLP
            & MLP  \\
        & $R_\phi$ Hidden Dimension            
            & [64] 
            & [64] 
            & [128, 128] 
            & [128, 128] \\
        &$R_\phi$ Activation Function            
            & ReLU 
            & ReLU 
            & ReLU 
            & ReLU\\
\cmidrule{2-6}
         & Network Type   
             & Unet 
             & Unet
             & Unet 
             & Unet \\
         & Hidden Dimension      
             & [256, 512, 1024] 
             & [256, 512, 1024] 
             & [256, 512, 1024] 
             & [256, 512, 1024] \\
\midrule

    \multirow{5}{*}{\methoddpcls{}} 
        & $C_\phi$ Model Type            
            & MLP 
            & MLP 
            & MLP
            & MLP  \\
        & $C_\phi$ Hidden Dimension            
            & [256, 256] 
            & [256, 256]
            & [256, 256] 
            & [256, 256] \\
        &$C_\phi$ Activation Function            
            & ReLU 
            & ReLU 
            & ReLU 
            & ReLU\\
\cmidrule{2-6}
         & Network Type   
             & Unet 
             & Unet
             & Unet 
             & Unet \\
         & Hidden Dimension      
             & [256, 512, 1024] 
             & [256, 512, 1024] 
             & [256, 512, 1024] 
             & [256, 512, 1024] \\
\midrule
    
    \multirow{2}{*}{\begin{tabular}{l}
        \methoddp{} \\(Ours) 
    \end{tabular}} 
        & Network Type   
             & Unet 
             & Unet 
             & Unet 
             & Unet \\
         & Hidden Dimension      
             & [256, 512, 1024]
             & [256, 512, 1024] 
             & [256, 512, 1024] 
             & [256, 512, 1024] \\
\bottomrule

    \multirow{6}{*}{CQL} 
        & Critic Network Type
             & MLP 
             & MLP 
             & MLP 
             & MLP \\
        & Critic Hidden Dimension
             & [256, 256] 
             & [256, 256] 
             & [256, 256]
             & [256, 256] \\
        & Critic Activation Function
             & ReLU
             & ReLU 
             & ReLU 
             & ReLU \\
    \cmidrule{2-6}
        & Network Type            
             & MLP 
             & MLP 
             & MLP 
             & MLP \\
         & Hidden Dimension      
             & [256, 256, 256, 256] 
             & [256, 256] 
             & [256, 256] 
             & [256, 256] \\
         & Activation Function           
             & tanh 
             & ReLU 
             & ReLU 
             & ReLU \\
\midrule

    \multirow{6}{*}{TD3+BC} 
        & Critic Network Type
             & MLP 
             & MLP 
             & MLP 
             & MLP \\
        & Critic Hidden Dimension
             & [256, 256] 
             & [256, 256] 
             & [256, 256] 
             & [256, 256]  \\
        & Critic Activation Function
             & tanh
             & ReLU 
             & ReLU 
             & ReLU \\
    \cmidrule{2-6}
        & Network Type            
             & MLP 
             & MLP 
             & MLP 
             & MLP\\
         & Hidden Dimension      
             & [256, 256, 256, 256] 
             & [256, 256] 
             & [256, 256] 
             & [256, 256]\\
         & Activation Function           
             & tanh 
             & ReLU 
             & ReLU 
             & ReLU  \\
\midrule

    \multirow{4}{*}{DT} 
     & Transformer   
         & GPT2 
         & GPT2 
         & GPT2 
         & GPT2 \\
     & Hidden Size     
         & 512 
         & 512 
         & 512 
         & 512\\
     & Number of Layers
         & 6 
         & 6 
         & 6 
         & 6 \\
     & Attention Heads    
         & 8 
         & 8 
         & 8 
         & 8\\

\bottomrule
\end{tabular}
}
\end{table}

\begin{table}[t]
\centering
\caption{\textbf{Model architectures (Part II):} \blockpushing, \peginsertionside, \hammer, \relocate }
\label{tab:model_architecture_part2}
\resizebox{0.89\textwidth}{!}{
\begin{tabular}{clcccc}
\toprule
\textbf{Method} & \textbf{Models}  & \textsc{BlockPush} & \textsc{PegInsert} & \textsc{Hammer} & \textsc{Relocate}\\
\midrule
    \multirow{3}{*}{BC} 
         & Network Type            
             & Unet 
             & MLP 
             & MLP 
             & MLP \\
         & Hidden Dimension      
             & [256, 512, 1024] 
             & [256, 256] 
             & [256, 256] 
             & [256, 256] \\
         & Activation Function           
             & - 
             & tanh 
             & ReLU 
             & ReLU \\
\midrule
    \multirow{4}{*}{DWBC}
        & Network Type   
            & Unet
            & MLP
            & MLP
            & MLP
        \\
        & Hidden Dimension 
            & [256, 512, 1024]
            & [256, 256]
            & [256, 256]
            & [256, 256]
        \\
        & Activation Function 
            & -
            & tanh
            & ReLU
            & ReLU
        \\
        & Discriminator Hidden Dimension
            & 256
            & 256
            & 256
            & 256
        \\

\midrule

    \multirow{9}{*}{DemoDICE}
        & Network Type  
            & Unet
            & MLP
            & MLP
            & MLP
        \\
        & Hidden Dimension  
            & [256, 512, 1024]
            & [256, 256]
            & [256, 256]
            & [256, 256]
        \\
        & Activation Function  
            & -
            & tanh
            & ReLU
            & ReLU
        \\
    \cmidrule{2-6}
        & Discriminator Network Type
            & MLP
            & MLP
            & MLP
            & MLP
        \\
        & Discriminator Hidden Dimension 
            & [128, 128]
            & [256, 256]
            & [256, 256]
            & [256, 256]
        \\
        & Discriminator Activation Function 
            & ReLU 
            & ReLU 
            & ReLU 
            & ReLU 
        \\
    \cmidrule{2-6}
        & Nu Network Type
            & MLP
            & MLP
            & MLP
            & MLP
        \\
        & Nu Hidden Dimension 
            & [128, 128]
            & [256, 256]
            & [256, 256]
            & [256, 256]
        \\
        & Nu Activation Function 
            & ReLU 
            & ReLU 
            & ReLU 
            & ReLU 
        \\
\midrule

    \multirow{6}{*}{ILID}
        & Network Type  
            & Unet 
            & MLP
            & MLP
            & MLP \\
        & Hidden Dimension  
            & [256, 512, 1024]
            & [256, 256]
            & [256, 256]
            & [256, 256]
            \\
        & Activation Function     
            & -
            & tanh
            & ReLU
            & ReLU
        \\
    \cmidrule{2-6}
        & Discriminator Network Type
            & MLP
            & MLP
            & MLP
            & MLP
        \\
        & Discriminator Hidden Dimension 
            & [256]
            & [256]
            & [256]
            & [256]
        \\
        & Discriminator Activation Function 
            & ReLU 
            & ReLU 
            & ReLU 
            & ReLU 
        \\
\midrule

    \multirow{6}{*}{\rbc{}} 
        & $R_\phi$ Model Type    
            & MLP 
            & MLP  
            & MLP 
            & MLP  
            \\
        & $R_\phi$ Hidden Dimension 
            & [256, 256] 
            & [256, 256]  
            & [128]
            & [128]
            \\
        &$R_\phi$ Activation Function   
            & ReLU 
            & ReLU  
            & ReLU 
            & ReLU  
            \\
\cmidrule{2-6}
        & Network Type   
             & Unet 
             & MLP 
             & MLP 
             & MLP \\
         & Hidden Dimension   
             & [256, 512, 1024] 
             & [256, 256]
             & [256, 256] 
             & [256, 256] 
             \\
         & Activation Function  
             & - 
             & tanh 
             & ReLU 
             & ReLU \\
\midrule

    \multirow{6}{*}{\methodcls{}} 
        & $C_\phi$ Model Type            
            & MLP 
            & MLP 
            & MLP 
            & MLP 
            \\
        & $C_\phi$ Hidden Dimension            
            & [128]
            & [256, 256]
            & [256]
            & [256]
            \\
        &$C_\phi$ Activation Function            
            & ReLU 
            & ReLU 
            & ReLU 
            & ReLU 
            \\
\cmidrule{2-6}
        & Network Type            
             & Unet 
             & MLP 
             & MLP 
             & MLP 
            \\
         & Hidden Dimension      
             & [256, 512, 1024] 
             & [256, 256]
             & [256, 256]
             & [256, 256]
             \\
         & Activation Function           
             & - 
             & tanh
             & ReLU
             & ReLU \\
\midrule

    \multirow{3}{*}{\begin{tabular}{l}
        \method{} \\(Ours) 
    \end{tabular}} 
        & Network Type 
             & Unet 
             & MLP 
             & MLP 
             & MLP \\
         & Hidden Dimension    
             & [256, 512, 1024] 
             & [256, 256] 
             & [256, 256] 
             & [256, 256] \\
         & Activation Function  
             & - 
             & tanh 
             & ReLU 
             & ReLU \\

\bottomrule

    \multirow{2}{*}{DP} 
         & Network Type 
             & Unet 
             & Unet 
             & Unet 
             & Unet \\
         & Hidden Dimension      
             & [256, 512, 1024] 
             & [256, 512, 1024] 
             & [256, 512, 1024] 
             & [256, 512, 1024] \\
\midrule

    \multirow{3}{*}{LPB-Offline} 
        & Dynamics Network Type  
             & MLP
             & MLP 
             & MLP
             & MLP \\
         & Dynamics Hidden Dimension  
             & [256, 256] 
             & [256, 256] 
             & [256, 256] 
             & [256, 256] \\
         & Dynamics Activation Function  
             & ReLU
             & ReLU 
             & ReLU
             & ReLU \\
\midrule

    \multirow{5}{*}{EDP} 
 
         & Critic Network Type
             & Unet 
             & MLP 
             & MLP 
             & MLP 
             \\
        & Hidden Dimension      
             & [256, 256] 
             & [256, 256] 
             & [256, 256] 
             & [256, 256] \\
         & Critic Activation Function
             & ReLU 
             & ReLU 
             & ReLU 
             & ReLU \\
 
\cmidrule{2-6}
         & Network Type   
             & Unet 
             & Unet 
             & Unet 
             & Unet \\
         & Hidden Dimension      
             & [256, 512, 1024] 
             & [256, 512, 1024] 
             & [256, 512, 1024] 
             & [256, 512, 1024] \\
\midrule

    \multirow{5}{*}{\rdp{}} 
        & $R_\phi$ Model Type    
            & MLP 
            & MLP  
            & MLP 
            & MLP  \\
        & $R_\phi$ Hidden Dimension  
            & [64] 
            & [128]
            & [128] 
            & [64] \\
        &$R_\phi$ Activation Function  
            & ReLU 
            & ReLU 
            & ReLU 
            & ReLU \\
\cmidrule{2-6}
         & Network Type   
             & Unet 
             & Unet 
             & Unet 
             & Unet \\
         & Hidden Dimension      
             & [256, 512, 1024] 
             & [256, 512, 1024] 
             & [256, 512, 1024] 
             & [256, 512, 1024] \\
\midrule

    \multirow{5}{*}{\methoddpcls{}} 
        & $C_\phi$ Model Type    
            & MLP 
            & MLP  
            & MLP 
            & MLP  \\
        & $C_\phi$ Hidden Dimension  
            & [128] 
            & [256, 256]
            & [128] 
            & [128] \\
        &$C_\phi$ Activation Function  
            & ReLU 
            & ReLU 
            & ReLU 
            & ReLU \\
\cmidrule{2-6}
         & Network Type   
             & Unet 
             & Unet 
             & Unet 
             & Unet \\
         & Hidden Dimension      
             & [256, 512, 1024] 
             & [256, 512, 1024] 
             & [256, 512, 1024] 
             & [256, 512, 1024] \\
\midrule
    
    \multirow{2}{*}{\begin{tabular}{l}
        \methoddp{} \\(Ours) 
    \end{tabular}} 
        & Network Type   
             & Unet 
             & Unet 
             & Unet 
             & Unet \\
         & Hidden Dimension      
             & [256, 512, 1024] 
             & [256, 512, 1024] 
             & [256, 512, 1024] 
             & [256, 512, 1024] \\
\bottomrule

    \multirow{6}{*}{CQL} 
        & Critic Network Type
             & MLP 
             & MLP 
             & MLP 
             & MLP \\
        & Critic Hidden Dimension
             & [256, 256]
             & [256, 256] 
             & [256, 256]
             & [256, 256] \\
        & Critic Activation Function
             & ReLU 
             & ReLU 
             & ReLU 
             & ReLU \\
    \cmidrule{2-6}
        & Network Type   
             & Unet 
             & MLP 
             & MLP 
             & MLP \\
         & Hidden Dimension      
             & [256, 512, 1024] 
             & [256, 256] 
             & [256, 256] 
             & [256, 256] \\
         & Activation Function 
             & - 
             & tanh 
             & ReLU 
             & ReLU \\
\midrule

    \multirow{6}{*}{TD3+BC} 
        & Critic Network Type
             & MLP 
             & MLP 
             & MLP 
             & MLP \\
        & Critic Hidden Dimension
             & [256, 256]
             & [256, 256] 
             & [256, 256]
             & [256, 256] \\
        & Critic Activation Function
             & ReLU 
             & ReLU 
             & ReLU 
             & ReLU \\
    \cmidrule{2-6}
        & Network Type        
             & Unet 
             & MLP 
             & MLP 
             & MLP \\
         & Hidden Dimension    
             & [256, 512, 1024] 
             & [256, 256] 
             & [256, 256] 
             & [256, 256]\\
         & Activation Function  
             & - 
             & ReLU 
             & ReLU 
             & ReLU \\
\midrule

    \multirow{4}{*}{DT} 
     & Transformer   
         & GPT2 
         & GPT2 
         & GPT2 
         & GPT2 \\
     & Hidden Size   
         & 1024 
         & 512  
         & 128
         & 128\\
     & Number of Layers
         & 6 
         & 6  
         & 2 
         & 4\\
     & Attention Heads  
         & 8 
         & 8 
         & 4
         & 4\\

\bottomrule
\end{tabular}
}
\end{table}

%% file: tables/hyperparameters.tex
{
\tiny
\setlength{\tabcolsep}{2pt}

\begin{longtable}{clcccccccc}
\caption{\textbf{Hyperparameters.} The overview of the hyperparameters used for all the methods in every task. We abbreviate "Gradient Accumulation" as "Grad Acc." and "Finetune" as "FT" in this table.} 
\vspace{0.1in}
\label{tab:hyperparameters}\\

\toprule
\textbf{Method} & \textbf{Hyperparameter} & \textsc{Maze} & \textsc{Parking} 
 & \textsc{Sweep} & \textsc{Box-Close} & \textsc{BlockPush} & \textsc{PegInsert} & \textsc{Hammer} & \textsc{Relocate} \\
\midrule
\endfirsthead

\caption[]{Hyperparameters (Continued)}\\
\toprule
\textbf{Method} & \textbf{Hyperparameter} & \textsc{Maze} & \textsc{Parking} 
 & \textsc{Sweep} & \textsc{Box-close} & \textsc{BlockPush} & \textsc{PegInsert} & \textsc{Hammer} & \textsc{Relocate} \\
\midrule
\endhead

\bottomrule
\endlastfoot

\multirow{3}{*}{BC} 
 & Policy Batch Size            
 & 128 & 128  & 128 & 128 & 256 & 128 &128 &128\\
 & Policy Learning Rate         
 & 1e-4 & 1e-4  & 1e-2 & 1e-2 & 2e-4 & 1e-3 & 1e-3 & 1e-3 \\
 & Policy Epochs                
 & 1001 & 1001 & 1001 & 1001 & 1001 & 1001 & 1001 & 1001 \\
\midrule

\multirow{6}{*}{DWBC} 
 & Policy Batch Size  
 &128 &128 &128 &128 &256 &128 &128 &128 \\
 & Policy Learning Rate         
 & 1e-4 & 3e-4 & 3e-4 & 3e-4 & 2e-4 & 1e-4 & 3e-4  &3e-5\\
 & $\alpha$
 & 7.5 & 10 & 7.5 & 7.5 & 15 & 12 & 7.5 &10 \\
 & Discriminator Hidden Dim. logpi       
 &128 &128 &128 &128 &128 &128 &128 &128 \\
 & Discriminator Hidden Dim. sa               
 &128 &128 &128 &128 &128 &128 &128 &128 \\
 & Policy Epochs                
 & 1001 & 1001 & 1001 & 1001 & 401 & 401 & 1001 & 1001 \\
\midrule

\multirow{4}{*}{DemoDICE} 
 & Policy Batch Size            
 & 128 &128 &128 &128 &256 &128 &128 &128 \\
 & Policy Learning Rate         
 & 1e-4 & 3e-4 & 1e-4 & 5e-4 & 1e-4 &1e-3  & 1e-4 & 5e-5 \\
 & $\alpha$
 & 0.1 & 0.1 & 0.1 & 0.05 & 0.01 & 0.1  & 0.1 & 0.1 \\
 & Policy Epochs                
 & 1001 & 1001 & 1001 & 1001 & 1001 & 401 & 1001 & 1001 \\
\midrule

\multirow{6}{*}{ILID} 
 & Policy Batch Size  
 & 128 & 128 & 128 & 128 & 256 & 128 & 128 & 128 \\
 & Policy Learning Rate         
 & 1e-4 & 1e-4 & 1e-2 & 1e-2 & 2e-4 & 1e-3 & 1e-3 & 1e-3 \\
 & Discriminator Learning Rate         
 & 1e-4 & 1e-4 & 1e-4 & 1e-4 & 1e-4 & 1e-4 & 1e-4 & 1e-4 \\
 & Rollback steps       
 & 5 & 3 & 3 & 5 & 5 & 3 & 5 & 3 \\
 & Rollback Decay Rate              
 & 0.9 & 0.9 & 0.9 & 0.9 & 0.9 & 0.9 & 0.9 & 0.9 \\
 & Select Threshold    
 & 0.89 & 0.89 & 0.89 & 0.89 & 0.89 & 0.89 & 0.89 & 0.89 \\
\midrule

\multirow{7}{*}{\rbc{}} 
 & Reward Model Batch Size            
 & 256 & 256  & 256 & 256 & 512 & 512 & 256 & 256 \\
 & Reward Model Learning Rate         
 & 1e-4 & 1e-4  & 1e-4 & 1e-4 & 1e-4 & 1e-4 & 1e-4 & 1e-4 \\
 & Reward Model Epochs                
 & 2001 & 2000  & 1001 & 1001 & 501  & 501 & 200 & 500 \\
\cmidrule{2-10}
 & Policy Batch Size            
 & 128 & 128  & 128 & 128 & 256 & 128 &128 &128 \\
 & Policy Reward Loss Weight  
 & 1e-2 & 1e-2 & 1e-3 & 1e-3 & 1e-4 & 1e-2 & 1e-2 & 1e-2  \\
 & Policy Learning Rate         
 & 1e-4 & 1e-4  & 1e-2 & 1e-2 & 2e-4 & 1e-3 & 1e-3 & 1e-3 \\
 & Policy Epochs                
 & 1001 & 1001 & 1001 & 1001 & 1001 & 1001 & 1001 & 1001\\
\midrule

\multirow{9}{*}{\methodcls{}}  
 & Classifier Model Batch Size            
 & 256 & 256 & 256 & 256 & 512 & 256 & 256 & 256 \\
 & Classifier Model Learning Rate         
 & 1e-4 & 1e-4 & 1e-4 & 1e-4 & 5e-5 & 1e-4 & 1e-4 & 1e-4 \\
 & Classifier Model Epochs                
 & 1000 & 1000 & 1000 & 1000 & 1000 & 1000 & 1000 & 1000 \\
 & Classifier Heads                
 & 6 & 7 & 6 & 6 & 4 & 10 & 7 & 6 \\
 & Classifier Classes                
 & 8 & 3 & 3 & 3 & 4 & 11 & 4 & 3 \\
\cmidrule{2-10}
 & Policy Batch Size            
 & 128 & 128 & 128 & 128 & 256 & 128 & 128 & 128 \\
 & Policy Class Loss Weight  
 & 1e-4 & 1e-4 & 1e-4 & 1e-4 & 2e-4 & 1e-4 & 1e-1 & 1e-4 \\
 & Policy Learning Rate         
 & 1e-4 & 1e-2 & 1e-2 & 1e-2 & 1e-3 & 1e-3 & 1e-2 & 1e-3 \\
 & Policy Epochs                
 & 1001 & 1001 & 1001 & 1001 & 1001 & 1001 & 1001 & 1001 \\
\midrule

\multirow{14}{*}{
\begin{tabular}{c}
\method{} \\
(Ours)
\end{tabular}}
 & MLP Projector Batch Size 
 & 128 & 128 & 128 & 128 & 128 & 64 & 128 & 128 \\
 & MLP Projector Grad Acc.
 & 4 & 2 & 2 & 2 & 4 & 8 & 2 & 2 \\
 & MLP Projector Learning Rate 
 & 1e-4 & 1e-4 & 1e-4 & 1e-4 & 1e-4 & 1e-4 & 1e-4 & 1e-4 \\
 & MLP Projector Epochs 
 & 100 & 100 & 100 & 100 & 100 & 100 & 100 & 100 \\
 \cmidrule{2-10}
 & $\mu_\phi$ FT Batch Size
 & 128 & 128  & 128 & 128 & 128 & 64 & 128 & 128 \\
 & $\mu_\phi$ FT Grad Acc. 
 & 1 & 1 & 1 & 1 & 1 & 2 & 1 & 1 \\
 & \llm{} FT Learning Rate 
 & 1e-5 & 5e-5 & 5e-5 & 1e-4 & 1e-5 & 1e-5 & 1e-4 & 1e-4 \\
 & MLP Projector FT Learning Rate 
 & 1e-5 & 5e-5 & 5e-5 & 1e-4 & 1e-5 & 1e-5 & 5e-5 & 5e-5 \\
 & $\mu_\phi$ FT Epochs 
 & 10 & 10 & 10 & 10 & 10 & 10 & 10 & 10 \\
 \cmidrule{2-10}
 & Policy Batch Size 
 & 16 & 16  & 16 & 16 & 16 & 16 & 32 & 16\\
 & Policy Grad Acc. 
 & 8 & 8 & 8 & 8 & 16 & 8 & 4 & 8 \\
 & \lcloss{} Weight
 & 1e-2 & 1e-2 & 1e-2 & 1e-2 & 1e-4 & 1e-2 & 1e-2 & 1e-2\\
 & Policy Learning Rate 
 & 1e-4 & 1e-4 & 1e-2 & 1e-2 & 2e-4 & 1e-3 & 1e-3 & 1e-3\\
 & Policy Epochs 
 & 1001 & 1001 & 1001 & 1001 & 1001 & 1001 & 1001 & 1001\\

\bottomrule

\multirow{3}{*}{DP} 
 & Policy Batch Size            
 & 256 & 256 & 256 & 256 & 256 & 256 & 256 & 256\\
 & Policy Learning Rate         
 & 1e-4 & 1e-4 & 1e-4 & 1e-4 & 1e-4 & 1e-4 & 1e-4 & 1e-4\\
 & Policy Epochs                
 & 1001 & 1001 & 1001 & 1001 & 1001 & 1001 & 1001 & 1001\\
\midrule

\multirow{3}{*}{LPB} 
 & Dynamic model Batch Size            
 & 256 & 256 & 256 & 256 & 256 & 256 & 256 & 256 \\
 & Dynamic model Learning Rate         
 & 1e-4 & 1e-4 & 1e-4 & 1e-4 & 1e-4 & 1e-5 & 1e-4 & 1e-4 \\
 & Dynamic model Epochs                
 & 1001 & 1001 & 1001 & 1001 & 1001 & 1001 & 1001 & 1001 \\
\midrule

\multirow{8}{*}{EDP} 
 & Pretrain Policy Batch Size            
 & 256 & 256 & 256 & 256 & 256 & 256 & 256 & 256\\
 & Pretrain Policy Alpha
 & 1e-2 & 1e-2 & 1e-3 & 1e-3 & 1e-3 & 1e-3 & 1e-3 & 1e-3\\
 & Pretrain Policy Learning Rate         
 & 1e-4 & 1e-4 & 1e-4 & 1e-4 & 1e-5 & 1e-4 & 1e-4 & 1e-4 \\
 & Pretrain Policy Epochs                
 & 1001 & 1001 & 1001 & 1001 & 501 & 1001 & 1001 & 1001 \\

 \cmidrule{2-10}
 & Policy Batch Size            
 & 256 & 256 & 256 & 256 & 256 & 256 & 256 & 256\\
 & Policy Alpha
 & 1e-2 & 1e-2 & 1e-3 & 1e-3 & 1e-3 & 1e-3 & 1e-3 & 1e-3\\
 & Policy Learning Rate         
 & 1e-4 & 1e-4 & 1e-4 & 1e-4 & 1e-4 & 1e-4 & 1e-4 & 1e-4 \\
 & Policy Epochs                
 & 1001 & 1001 & 1001 & 1001 & 1001 & 1001 & 1001 & 1001 \\
\midrule

\multirow{7}{*}{\rdp{}} 
 & Reward Model Batch Size            
 & 256 & 256 & 256 & 256 & 256 & 512 & 256 & 256 \\
 & Reward Model Learning Rate         
 & 1e-6 & 1e-5 & 1e-5 & 1e-5 & 1e-6 & 1e-6 & 1e-4 & 1e-4 \\
 & Reward Model Epochs                
 & 1000 & 1000 & 500 & 500 & 100 & 500 & 200 & 500\\
\cmidrule{2-10}
 & Policy Batch Size            
 & 256 & 256 & 256 & 256 & 256 & 256 & 256 & 256\\
 & Policy Reward Loss Weight    
 & 1e-3 & 1e-3 & 1e-2 & 1e-3 & 1e-4 & 1e-3 & 1e-2 & 1e2 \\
 & Policy Learning Rate         
 & 1e-4 & 1e-4 & 1e-4 & 1e-4 & 1e-4 & 1e-4 & 1e-4 & 1e-4 \\
 & Policy Epochs                
 & 1001 & 1001 & 1001 & 1001 & 1001 & 1001 & 1001 & 1001 \\
\midrule

\multirow{9}{*}{\methoddpcls{}}  
 & Classifier Model Batch Size            
 & 256 & 256 & 256 & 256 & 512 & 256 & 256 & 256 \\
 & Classifier Model Learning Rate         
 & 1e-4 & 1e-4 & 1e-4 & 1e-4 & 1e-5 & 1e-4 & 1e-4 & 1e-4 \\
 & Classifier Model Epochs                
 & 1000 & 1000 & 1000 & 1001 & 1000 & 1000 & 1000 & 1000 \\
 & Classifier Heads                
 & 6 & 7 & 6 & 6 & 4 & 10 & 7 & 6 \\
 & Classifier Classes                
 & 8 & 3 & 3 & 3 & 4 & 11 & 4 & 3 \\
\cmidrule{2-10}
 & Policy Batch Size            
 & 256 & 256 & 256 & 256 & 256 & 256 & 256 & 256 \\
 & Policy Class Loss Weight  
 & 1e-4 & 1e-3 & 1e-4 & 1e-4 & 1e-3 & 1e-4 & 1e-4 & 1e-3 \\
 & Policy Learning Rate         
 & 1e-4 & 1e-4 & 1e-4 & 1e-4 & 1e-4 & 1e-4 & 1e-4 & 1e-4 \\
 & Policy Epochs                
 & 1001 & 1001 & 1001 & 1001 & 1001 & 1001 & 1001 & 1001 \\
\midrule

\multirow{14}{*}{
\begin{tabular}{l}
\methoddp{} \\
(Ours)
\end{tabular}}
 & MLP Projector Batch Size 
 & 128 & 128 & 128 & 128 & 128 & 64 & 128 & 128 \\
 & MLP Projector Grad Acc.
 & 4 & 2 & 2 & 2 & 4 & 8 & 2 & 2 \\
 & MLP Projector Learning Rate 
 & 1e-4 & 1e-4 & 1e-4 & 1e-4 & 1e-4 & 1e-4 & 1e-4 & 1e-4 \\
 & MLP Projector Epochs 
 & 100 & 100 & 100 & 100 & 100 & 100 & 100 & 100 \\
 \cmidrule{2-10}
 & $\mu_\phi$ FT Batch Size 
 & 128 & 128& 128 & 128 & 128 & 64 & 128 & 128 \\
 & $\mu_\phi$ FT Grad Acc. 
 & 1 & 1 & 1 & 1 & 1 & 2 & 1 & 1 \\
 & \llm{} FT Learning Rate 
 & 1e-5 & 5e-5 & 5e-5 & 5e-5 & 1e-5 & 1e-5 & 1e-4 & 1e-4 \\
 & MLP Projector FT Learning Rate 
 & 1e-5 & 5e-5 & 5e-5 & 5e-5 & 1e-5 & 1e-5 & 5e-5 & 5e-5 \\
 & $\mu_\phi$ FT Epochs 
 & 10 & 10 & 10 & 10 & 10 & 10 & 10 & 10 \\
 \cmidrule{2-10}
 
 & Policy Batch Size            
 & 16 & 16 & 16 & 16 & 16 & 16 & 16 & 16 \\
 & Policy Grad Acc. 
 & 16 & 16  & 16 & 16 & 16 & 16 & 16 & 16 \\
 & \lcloss{} Weight              
 & 1e-2 & 1e-4 & 1e-3 & 1e-4 & 1e-5 & 1e-2 & 5e-2 & 1e-3 \\
 & Policy Learning Rate         
 & 1e-4 & 1e-4 & 1e-4 & 1e-4 & 1e-4 & 1e-4 & 1e-4 & 1e-4 \\
 & Policy Epochs                
 & 1001 & 1001 & 1001 & 1001 & 1001 & 1001 & 1001 & 1001 \\
\bottomrule

 \multirow{7}{*}{CQL} 
 & Policy Batch Size            
 & 128 & 128 & 128 & 128 & 256 & 128 & 128 & 128 \\
 & Critic Learning Rate
 & 3e-4 & 1e-4 & 1e-4 & 1e-4 & 1e-4 & 1e-5 & 1e-4 & 1e-4 \\
 & Policy Learning Rate         
 & 1e-4 & 1e-4 & 1e-4 & 1e-4 & 1e-4 & 1e-4 & 1e-4 & 1e-4 \\
 & Policy Epochs                
 & 3001 & 3001 & 1001 & 2001 & 1001 & 1001 & 1001 & 1001 \\
 & BC Training Epochs                
 & 1000 & 200 & 200 & 200 & 1000 & 200 & 200 & 200 \\
 & CQL $\alpha'$
 & 5.0 & 10.0  & 5.0 & 10.0 & 10.0 & 5.0 & 5.0 & 5.0 \\
 & CQL action samples 
 & 64 & 64 & 64 & 64 & 64 & 32  & 64 & 64 \\
\midrule

\multirow{8}{*}{TD3+BC} 
 & Pretrain Batch Size            
 & 128 & 128& 128 & 128 & 256 & 128 & 128 & 128 \\
 & Pretrain Alpha                 
 & 1e-3 & 1e-4 & 1e-3 & 1e-3 & 1e-3 & 1e-3 & 1e-3 & 1e-3\\
 & Pretrain Learning Rate        
 & 1e-4 & 1e-4 & 1e-3 & 1e-3 & 1e-5 & 1e-4 & 1e-3 & 1e-3 \\
 & Pretrain Epochs                
 & 1001 & 1001 & 1001 & 1001 & 1001 & 1001 & 1001 & 1001 \\
\cmidrule{2-10}
 & FT Batch Size            
 & 128 & 128 & 128 & 128 & 256 & 128 & 128 & 128 \\
 & FT Alpha                 
  & 1e-3 & 1e-4 & 1e-3 & 1e-3 & 1e-3 & 1e-3 & 1e-3 & 1e-3\\
 & FT Learning Rate         
 & 1e-3 & 1e-4 & 1e-3 & 1e-3 & 2e-4 & 1e-3 & 1e-3 & 1e-3\\
 & FT Epochs                
 & 1001 & 1001 & 1001 & 1001 & 1001 & 1001 & 1001 & 1001 \\
\midrule

\multirow{4}{*}{DT} 
 & Policy Batch Size            
 & 256 & 256 & 256 & 256 & 256 & 256 & 256 & 256 \\
 & Policy Learning Rate         
 & 1e-4 & 1e-4 & 1e-4 & 1e-4 & 1e-4 & 1e-4 & 5e-5 & 5e-5 \\
 & Policy Epochs                
 & 1001 & 1001 & 1001 & 1001 & 1001 & 1001 & 1001 & 1001 \\
 & Target Return               
 & 100.45082 & 18.32615 & 6.21332 &  5.72269 & 36.94823 & 22.09803 & 29.38271  & 27.49282\\
\end{longtable}

}

%% file: tables/comp_resource.tex
\begin{table}[H]
    \centering
    \small
    \caption{\textbf{Computation resources.}}
    \resizebox{\textwidth}{!}{
    \begin{tabular}{cccc}
        \toprule
        \textbf{Workstation} & CPU & GPU & RAM \\
        \midrule
        1 
        & Intel Xeon W-2255 
        & NVIDIA GeForce RTX 3090, NVIDIA GeForce RTX 3080
        & 125G \\
        2 
        & Intel Xeon W-2255 
        & NVIDIA GeForce RTX 3090, NVIDIA GeForce RTX 3080
        & 125G \\
        3 
        & Intel Xeon W-2255 
        & NVIDIA GeForce RTX 3080 Ti $\times$ 2
        & 125G \\
        4 
        & Intel Xeon W-2255 
        & NVIDIA GeForce RTX 4070 Ti $\times$ 2
        & 125G \\
        5 
        & Intel Xeon W-2255 
        & NVIDIA GeForce RTX 4070 Ti $\times$ 2
        & 125G \\
        6 
        & Intel Xeon W-2255 
        & NVIDIA GeForce RTX 4070 Ti $\times$ 2
        & 125G \\
        7 
        & Intel Xeon W-2255 
        & NVIDIA GeForce RTX 4070 Ti $\times$ 2
        & 125G \\
        8 
        & Intel Xeon W-2255 
        & NVIDIA GeForce RTX 4070 Ti $\times$ 2
        & 125G \\
        9 
        & Intel Xeon w7-2475X
        & NVIDIA GeForce RTX 4090 $\times$ 2
        & 125G \\
        10 
        & Intel Xeon w7-2475X
        & NVIDIA GeForce RTX 4090 $\times$ 2
        & 125G \\
        11 
        & Intel Xeon w5-2455X
        & NVIDIA RTX A6000 $\times$ 2
        & 125G \\
        12 
        & Intel Xeon Platinum 8480+
        & NVIDIA H100 $\times$ 2
        & 500G \\
        \bottomrule
    \end{tabular}
    \label{tab:comp_resource}
    }
\end{table}

%% file: tex/appendix/baseline_details.tex
\section{Baseline details}
\label{appx:baseline_details}

\subsection{\rbc{} and \rdp{}}
\label{appx:rbc_rdp_details}

We design two reward-based ablation baselines, \rbc{} and \rdp{}, as scalar-reward counterparts to \method{} and \methoddp{}, respectively. 
These baselines test whether the gains of language-critique imitation arise merely from learning an auxiliary scalar reward model, or from preserving the structured information contained in \lang{} supervision.
Both baselines follow the same overall pipeline as our language-based methods: first, a reward predictor is trained on the general dataset; second, the reward predictor is frozen and used to provide auxiliary supervision during policy learning on the expert dataset.
Thus, \rbc{} and \rdp{} replace the language model $\mu_\phi$ and the \lcloss{} with a scalar reward model $R_\phi$ and a reward-critique loss.

\paragraph{Reward model.}
Given a reward-labeled general dataset $\mathcal{D}^r_G=\{(s_t,a_t,r_t)\}$, we train a reward model $R_\phi$ to predict the scalar reward associated with each state-action pair. 
The reward model is optimized with a mean squared error objective:
\begin{align}
    \mathcal{L}_R^\phi
    =
    \mathbb{E}_{(s_t,a_t,r_t)\sim\mathcal{D}^r_G}
    \left[
        \left\| R_\phi(s_t,a_t) - r_t \right\|_2^2
    \right].
    \label{eq:reward_model_objective}
\end{align}
After training, $R_\phi$ is frozen and used for downstream policy optimization.

\paragraph{Reward-critique loss.}
Given the frozen reward model $R_\phi$, we define a reward-based analogue of the \lcloss{} on the expert dataset 
$\mathcal{D}^r_E=\{(s_t,a^E_t,r_t)\}$.
For a policy action $\hat{a}_t\sim\pi_\theta(\cdot\mid s_t)$, we define the policy reward-prediction loss and the expert reward-prediction loss as
\begin{align}
    \mathcal{L}_{R}(\pi_\theta)
    &=
    \mathbb{E}_{
        (s_t,r_t)\sim\mathcal{D}^r_E,
        \hat{a}_t\sim\pi_\theta(\cdot\mid s_t)
    }
    \left[
        \left\| R_\phi(s_t,\hat{a}_t) - r_t \right\|_2^2
    \right],
    \label{eq:policy_reward_loss}
    \\
    \mathcal{L}_{R}(\pi_E)
    &=
    \mathbb{E}_{
        (s_t,a^E_t,r_t)\sim\mathcal{D}^r_E
    }
    \left[
        \left\| R_\phi(s_t,a^E_t) - r_t \right\|_2^2
    \right].
    \label{eq:expert_reward_loss}
\end{align}
The reward-critique loss clips the policy loss at the expert loss:
\begin{align}
    \mathcal{L}_\text{RC}(\pi_\theta,\pi_E)
    =
    \left[
        \mathcal{L}_{R}(\pi_\theta)
        -
        \mathrm{sg}\!\left(\mathcal{L}_{R}(\pi_E)\right)
    \right]_+,
    \label{eq:reward_critique_loss}
\end{align}
where $[x]_+=\max(x,0)$ and $\mathrm{sg}(\cdot)$ denotes stop-gradient.
This objective penalizes policy actions whose predicted rewards are less consistent with the expert reward labels than the corresponding expert actions, and the penalty vanishes once the policy matches the expert under the frozen reward model.

\paragraph{Reward-critique behavior cloning.}
\rbc{} applies the reward-critique loss to a feedforward behavior cloning policy.
The standard behavior cloning objective is
\begin{align}
    \mathcal{L}_\text{BC}
    =
    \mathbb{E}_{(s_t,a^E_t)\sim\mathcal{D}_E,
    \hat{a}_t\sim\pi_\theta(\cdot\mid s_t)}
    \left[
        \left\| a^E_t-\hat{a}_t \right\|_2^2
    \right].
    \label{eq:rbc_bc_loss}
\end{align}
\rbc{} jointly optimizes the imitation loss and the reward-critique loss:
\begin{align}
    \mathcal{L}_\text{BC}
    +
    \lambda \mathcal{L}_\text{RC}(\pi_\theta,\pi_E),
    \label{eq:rbc_objective}
\end{align}
where $\lambda$ controls the strength of reward guidance.

\paragraph{Reward-critique diffusion policy.}
\rdp{} applies the same reward-critique principle to diffusion policy.
The diffusion policy $\epsilon_\theta$ is trained with the standard noise-prediction objective
\begin{align}
    \mathcal{L}_\text{DP}
    =
    \mathbb{E}_{(s_t,a^E_t)\sim\mathcal{D}_E,
    \epsilon\sim\mathcal{N}(0,\mathbf{I}),k}
    \left[
        \left\|
            \epsilon-\hat{\epsilon}
        \right\|_2^2
    \right],
    \qquad
    \hat{\epsilon}\sim\epsilon_\theta(\cdot\mid s_t,a_t^k,k).
    \label{eq:rdp_dp_loss}
\end{align}
Since the diffusion policy predicts noise rather than actions, we apply the reward-critique loss to the one-step reconstructed action
\begin{align}
    \hat{a}_t^0
    =
    \frac{1}{\sqrt{\bar{\alpha}^k}}a_t^k
    -
    \frac{\sqrt{1-\bar{\alpha}^k}}{\sqrt{\bar{\alpha}^k}}\hat{\epsilon},
    \qquad
    \hat{\epsilon}\sim\epsilon_\theta(\cdot\mid s_t,a_t^k,k).
    \label{eq:rdp_1step_reconstruct}
\end{align}
Following \Cref{appx:reweighting_factor_derive}, we use the timestep-dependent weight $\omega^k=\frac{\bar{\alpha}^k}{1-\bar{\alpha}^k}$ to reduce the imbalance caused by noisier reconstructions at larger diffusion steps.
The reweighted reward-prediction losses are
\begin{align}
    \tilde{\mathcal{L}}_{R}(\pi_\theta)
    &=
    \mathbb{E}_{
        (s_t,r_t)\sim\mathcal{D}^r_E,
        \hat{a}_t^0\sim\epsilon_\theta(s_t),
        k
    }
    \left[
        \omega^k
        \left\| R_\phi(s_t,\hat{a}_t^0) - r_t \right\|_2^2
    \right],
    \label{eq:reweight_policy_reward_loss}
    \\
    \tilde{\mathcal{L}}_{R}(\pi_E)
    &=
    \mathbb{E}_{
        (s_t,a^E_t,r_t)\sim\mathcal{D}^r_E,
        k
    }
    \left[
        \omega^k
        \left\| R_\phi(s_t,a^E_t) - r_t \right\|_2^2
    \right].
    \label{eq:reweight_expert_reward_loss}
\end{align}
The reweighted reward-critique loss is then
\begin{align}
    \tilde{\mathcal{L}}_\text{RC}(\pi_\theta,\pi_E)
    =
    \left[
        \tilde{\mathcal{L}}_{R}(\pi_\theta)
        -
        \mathrm{sg}\!\left(\tilde{\mathcal{L}}_{R}(\pi_E)\right)
    \right]_+.
    \label{eq:reweight_reward_critique_loss}
\end{align}
The final \rdp{} objective is
\begin{align}
    \mathcal{L}_\text{DP}
    +
    \lambda \tilde{\mathcal{L}}_\text{RC}(\pi_\theta,\pi_E).
    \label{eq:rdp_objective}
\end{align}
\rdp{} therefore trains the diffusion policy to predict noise such that the reconstructed actions both remain close to expert actions and match the expert scalar rewards under the frozen reward model.
Unlike \methoddp{}, however, \rdp{} compresses behavioral supervision into scalar rewards and therefore cannot preserve the structured behavioral distinctions provided by \lang{}.

\subsection{\methodcls{} and \methoddpcls{}}
\label{appx:mhcbc_mhcdp_details}

We design two multi-head categorical baselines, \methodcls{} and \methoddpcls{}, as discrete-classification variants to \method{} and \methoddp{}, respectively.
These baselines convert each \lang{} $l_t$ into a tuple of categorical indices $c_t=(c_t^T, c_t^A, c_t^M)$ corresponding to the task progress \texttt{<T>}, action optimality \texttt{<A>}, and movement guidance \texttt{<M>}, discarding the natural-language form.

We test whether the gains of language-critique imitation arise from the expressiveness of natural language and the capacity of a pretrained language model to capture inter-label dependencies, or the independent categorical predictions over the same label dimensions.
Both baselines follow the same two-stage pipeline: first, a multi-head classifier is trained on the general dataset with categorical labels $c_t$; second, the classifier is frozen and used for auxiliary supervision during policy learning on the expert dataset.
Thus, \methodcls{} and \methoddpcls{} replace the language model $\mu_\phi$ and the \lcloss{} with a multi-head classifier $C_\phi$ and a classifier-critique loss.

\paragraph{Classifier model.}
Each \lang{} is decomposed into $H$ independent categorical dimensions corresponding to task progress \texttt{<T>}, action optimality \texttt{<A>}, and movement guidance \texttt{<M>}, with $C_h$ classes for each dimension $h\in\{1,\dots, H\}$.
Given a categorically labeled general dataset $\mathcal{D}^c_G=\{(s_t,a_t,c_t)\}$, where $c_t=(c_t^1,\ldots,c_t^H)\in\{1,\ldots,C_1\}\times\cdots\times\{1,\ldots,C_H\}$ encodes the class index for each dimension, we train a multi-head classifier $C_\phi$ that maps each state-action pair to per-head logits.
The classifier consists of a shared MLP encoder followed by $H$ parallel linear heads, where head $h$ outputs logits $C_\phi^h(s_t,a_t)\in\mathbb{R}^{C_h}$.
The classifier is optimized with a cross-entropy objective averaged over heads:
\begin{align}
    \mathcal{L}_C^\phi
    =
    \mathbb{E}_{(s_t,a_t,c_t)\sim\mathcal{D}^c_G}
    \left[
        \frac{1}{H}\sum_{h=1}^{H}
        \mathrm{CE}\!\left(C_\phi^h(s_t,a_t),\; c_t^h\right)
    \right],
    \label{eq:classifier_model_objective}
\end{align}
where $\mathrm{CE}(\cdot,\cdot)$ denotes the standard cross-entropy loss.
After training, $C_\phi$ is frozen and used for downstream policy optimization.

\paragraph{Classifier-critique loss.}
Given the frozen classifier $C_\phi$, we define a classification-based analogue of the \lcloss{} on the expert dataset
$\mathcal{D}^c_E=\{(s_t,a^E_t,c_t)\}$.
For a policy action $\hat{a}_t\sim\pi_\theta(\cdot\mid s_t)$, the policy classification loss and the expert classification loss are
\begin{align}
    \mathcal{L}_{C}(\pi_\theta)
    &=
    \mathbb{E}_{
        (s_t,c_t)\sim\mathcal{D}^c_E,\;
        \hat{a}_t\sim\pi_\theta(\cdot\mid s_t)
    }
    \left[
        \frac{1}{H}\sum_{h=1}^{H}
        \mathrm{CE}\!\left(C_\phi^h(s_t,\hat{a}_t),\; c_t^h\right)
    \right],
    \label{eq:policy_classifier_loss}
    \\
    \mathcal{L}_{C}(\pi_E)
    &=
    \mathbb{E}_{
        (s_t,a^E_t,c_t)\sim\mathcal{D}^c_E
    }
    \left[
        \frac{1}{H}\sum_{h=1}^{H}
        \mathrm{CE}\!\left(C_\phi^h(s_t,a^E_t),\; c_t^h\right)
    \right].
    \label{eq:expert_classifier_loss}
\end{align}
The classifier-critique loss clips the policy loss at the expert loss:
\begin{align}
    \mathcal{L}_\text{CC}(\pi_\theta,\pi_E)
    =
    \left[
        \mathcal{L}_{C}(\pi_\theta)
        -
        \mathrm{sg}\!\left(\mathcal{L}_{C}(\pi_E)\right)
    \right]_+.
    \label{eq:classifier_critique_loss}
\end{align}
This objective penalizes policy actions whose predicted class distributions are less consistent with the expert categorical labels than the corresponding expert actions.

\paragraph{Classifier-critique behavior cloning.}
\methodcls{} jointly optimizes the behavior cloning loss (\Cref{eq:rbc_bc_loss}) and the classifier-critique loss:
\begin{align}
    \mathcal{L}_\text{BC}
    +
    \lambda \mathcal{L}_\text{CC}(\pi_\theta,\pi_E),
    \label{eq:mhcbc_objective}
\end{align}
where $\lambda$ controls the strength of classifier guidance.

\paragraph{Classifier-critique diffusion policy.}
\methoddpcls{} applies the classifier-critique principle to diffusion policy.
As in \rdp{}, the classifier-critique loss is applied to the one-step reconstructed action $\hat{a}_t^0$ (\Cref{eq:rdp_1step_reconstruct}).
Following \Cref{appx:reweighting_factor_derive}, we use the timestep-dependent weight $\omega^k=\frac{\bar{\alpha}^k}{1-\bar{\alpha}^k}$.
The reweighted classification losses are
\begin{align}
    \tilde{\mathcal{L}}_{C}(\pi_\theta)
    &=
    \mathbb{E}_{
        (s_t,c_t)\sim\mathcal{D}^c_E,\;
        \hat{a}_t^0\sim\epsilon_\theta(s_t),\;
        k
    }
    \left[
        \frac{\omega^k}{H}\sum_{h=1}^{H}
        \mathrm{CE}\!\left(C_\phi^h(s_t,\hat{a}_t^0),\; c_t^h\right)
    \right],
    \label{eq:reweight_policy_classifier_loss}
    \\
    \tilde{\mathcal{L}}_{C}(\pi_E)
    &=
    \mathbb{E}_{
        (s_t,a^E_t,c_t)\sim\mathcal{D}^c_E,\;
        k
    }
    \left[
        \frac{\omega^k}{H}\sum_{h=1}^{H}
        \mathrm{CE}\!\left(C_\phi^h(s_t,a^E_t),\; c_t^h\right)
    \right].
    \label{eq:reweight_expert_classifier_loss}
\end{align}
The reweighted classifier-critique loss is then
\begin{align}
    \tilde{\mathcal{L}}_\text{CC}(\pi_\theta,\pi_E)
    =
    \left[
        \tilde{\mathcal{L}}_{C}(\pi_\theta)
        -
        \mathrm{sg}\!\left(\tilde{\mathcal{L}}_{C}(\pi_E)\right)
    \right]_+.
    \label{eq:reweight_classifier_critique_loss}
\end{align}
The final \methoddpcls{} objective is
\begin{align}
    \mathcal{L}_\text{DP}
    +
    \lambda \tilde{\mathcal{L}}_\text{CC}(\pi_\theta,\pi_E).
    \label{eq:mhcdp_objective}
\end{align}

\paragraph{Comparison with \method{}/\methoddp{}.} 
\methodcls{} and \methoddpcls{} preserve the same structured label dimensions as \lang{} (\texttt{<T>}, \texttt{<A>}, \texttt{<M>}) but discard two properties of language-based supervision: (i) the inter-label dependencies captured by the autoregressive pretrained language model, and (ii) the grounding of categorical distinctions in natural-language semantics through pretrained token embeddings.
By reducing language labels to independent categorical predictions, these baselines isolate whether the structured label dimensions alone drive the gains, or whether the expressiveness of natural language supervision is essential.

%% file: tex/appendix/add_experiments.tex
\newpage

\section{Additional experimental results}
\label{appx:add_exp}

\subsection{VLM-generated \lang{}s.}
\label{appx:vlm_labels}

In \Cref{tab:lang_ablation}, we study whether general-purpose VLMs can serve as an alternative to our proposed language generator $\mu_g$ on the task \boxclose{}. Specifically, we replace the generator $\mu_g$ in \Cref{approach:lang_design} with \texttt{o4-mini} prompting, using the template in \Cref{fig:vlm_prompt_temp}. However, directly using \texttt{o4-mini}-generated \lang{}s with the small \texttt{SmolLM2-135M-Instruct} \llm{} yields limited gains: \method{} reaches only $55.2\%$, which is comparable to shuffled $\mu_g$ labels ($55.6\%$), and \methoddp{} reaches $63.6\%$. This suggests that open-ended VLM labels are not automatically useful for language-critique imitation. One possible reason is that VLM labels are often verbose, diverse, and redundant, which makes their distribution harder for $\mu_\phi$ to model and may introduce noisy supervision during policy training.

We therefore evaluate two ways to improve VLM-based supervision. First, we increase the capacity of $\mu_\phi$ from \texttt{SmolLM2-135M-Instruct} to \texttt{SmolLM2-360M-Instruct}, testing whether a stronger captioner can better model the more diverse \texttt{o4-mini} label distribution. This improves both methods, with \method{} increasing from $55.2\%$ to $61.6\%$ and \methoddp{} increasing from $63.6\%$ to $71.2\%$, even surpassing the $\mu_g$ result for \methoddp{}. Second, we post-process \texttt{o4-mini} labels to reduce linguistic redundancy. Asking an LLM to make the labels concise improves \method{} to $63.4\%$ but does not improve \methoddp{}. In contrast, using an LLM to rewrite \texttt{o4-mini} labels into a $\mu_g$-style structured format performs worse, reaching only $56.4\%$ for \method{} and $62.4\%$ for \methoddp{}. This indicates that surface-level structure or style alone is insufficient.

To understand this failure, we compare \texttt{o4-mini} labels against $\mu_g$ labels by treating $\mu_g$ as a feature-aligned reference. The match rates are low: \texttt{o4-mini} agrees with $\mu_g$ on only $51.8\%$ of \texttt{<T>} descriptions, $46.1\%$ of \texttt{<A>} descriptions, and $18.8\%$ of \texttt{<M>} descriptions. 
Specifically, accuracy is particularly low for fine-grained movement guidance ($18.8\%$ for <M> labels) compared to higher-level task progress ($51.8\%$ for <T>). This suggests that current VLMs struggle with the precise spatial reasoning required to describe state-action effects in continuous control.
Overall, these results suggest that the degraded performance of VLM-generated labels comes from two sources: linguistic verbosity and semantic inaccuracy. The latter is especially severe for fine-grained movement guidance, where accurate \texttt{<M>} labels require precise spatial reasoning over state-action effects.
We further isolate the effects of linguistic form versus semantic content using two controlled $\mu_g$ variants:
\begin{itemize}
    \item $\mu_g$-shuffled: Preserves the concise, structured form but destroys semantic alignment.
    \item $\mu_g$-verbose: Preserves semantic correctness but rewrites labels into diverse, verbose, and human-like language.
\end{itemize}
The $\mu_g$-shuffled baseline preserves concise and structured label form but destroys semantic accuracy, while the $\mu_g$-verbose variant preserves semantic correctness but rewrites labels into more diverse, verbose, and human-like language. For \method{}, $\mu_g$-verbose substantially improves over shuffled labels ($68.0\%$ vs. $55.6\%$), suggesting that semantic accuracy is more important than concise surface form for feedforward policies. Although $\mu_g$-verbose still underperforms the original concise $\mu_g$ labels ($80.4\%$), it remains clearly beneficial when the label content is accurate. For \methoddp{}, however, $\mu_g$-verbose does not improve over shuffled labels and remains below the best \texttt{o4-mini} result with the larger \llm{}. This suggests that diffusion-policy training is more sensitive to label distribution complexity and benefits more from concise or well-modeled language supervision. This is aligned with the observations from \texttt{o4-mini}-generated labels plus the larger \llm{} backbone \texttt{SmolLM2-360M-Instruct}.

Overall, VLM-generated \lang{}s can be useful when paired with a stronger \llm{}, but $\mu_g$ remains the most reliable source of supervision because it provides both semantic consistency and feature-aligned structure to exploit useful behaviors to optimize policy from expert and suboptimal demonstrations.

\subsection{\llm{} ablation}
\label{appx:llm_ablation}

\input{tables/llm_captioner_ablation}

Table~\ref{tab:llm_captioner_ablation} ablates the choice of \llm{} backbone and the degree to which it is finetuned. We vary the backbone across \texttt{SmolLM2-135M-Instruct}~\citep{allal2025smollm2}, \texttt{SmolLM2-360M-Instruct}~\citep{allal2025smollm2}, \texttt{Qwen2-0.5B-Instruct}~\citep{yang2024qwen2}, and \texttt{Mistral-7B-Instruct-v0.3}~\citep{jiang2023mistral}, and consider three finetuning regimes: full-model finetuning, projector-only finetuning (with the \llm{} backbone frozen), and a pretrained baseline that uses the \llm{} off-the-shelf without any task-specific adaptation. Please refer to \Cref{appx:llm_captioner_details} for architectural details.

First, adaptation matters far more than scale: the smallest backbone, \texttt{SmolLM2-135M-Instruct}, achieves the best \method{} score $80.4\%$ when fully finetuned, and degrades substantially when only the projector is trained $62.0\%$ or when the pretrained model is used $53.6\%$. Projector-only training on the much larger \texttt{Mistral-7B-Instruct-v0.3} ($61.2\%$) does not recover this gap, suggesting that simply increasing parameter count without exposing the language model to the task is insufficient. 

Second, scaling the backbone under full-model finetuning does not yield consistent gains—\texttt{SmolLM2-360M-Instruct} $55.2\%$ and \texttt{Qwen2-0.5B-Instruct} 67.2\% both underperform the 135M variant and exhibit higher variance across seeds. Results on \methoddp{} are noisier and show a smaller spread across configurations, but the same qualitative ordering holds, with full finetuning of the 135M backbone remaining competitive. 
We therefore adopt \texttt{SmolLM2-135M-Instruct} with full-model finetuning as our default \llm{}: it is the smallest, fastest option and also the most accurate, indicating that the captioner's backbone is not the main bottleneck for our methods.

\subsection{\llm{} training data analysis}
\label{appx:llm_data_analysis}

Our framework uses the general dataset $\mathcal{D}_G$ and the expert dataset $\mathcal{D}_E$, following the typical problem setup of imitation learning from suboptimal demonstrations. The policy is trained on $\mathcal{D}_E$ (with the standard BC or diffusion objective on expert actions and the \lcloss{} evaluated on expert language labels), while the \llm{} is fine-tuned on the full $\mathcal{D}_G^{\text{lang}}$. Since the policy training objective does not directly consume suboptimal trajectories, it is important to verify that the general dataset provides a meaningful contribution through the \llm{}. To this end, we train an expert-only captioner $\mu_\phi^E$ on $\mathcal{D}_E^{\text{lang}}$ alone and use it for downstream \method{} and \methoddp{} training. Table~\ref{tab:captioner_data_ablation} reports results on \boxclose{} over 50 environment rollouts and 5 training seeds.

\input{tables/llm_training_data_analysis}

Removing $\mathcal{D}_G$ from captioner training degrades \method{} from $80.4\%$ to $59.2\%$--recovering only $4.4\%$ over plain BC ($54.8\%$). For \methoddp{}, the effect is more severe: the expert-only captioner drops performance to $56.0\%$, below the plain DP baseline ($64.8\%$), indicating that a captioner without contrastive supervision can actively harm policy learning.

When $\mu_\phi$ is trained on the full $\mathcal{D}_G^{\text{lang}}$, it observes both expert and suboptimal behaviors paired with their structured labels and learns a discriminative mapping that distinguishes good from bad behaviors. The \lcloss{} then transmits this discrimination to the policy at expert states. When $\mu_\phi$ is trained only on $\mathcal{D}_E^{\text{lang}}$, it never observes negative or corrective labels, so its likelihood landscape over \lang{}s becomes unable to distinguish action quality. The resulting \lcloss{} provides weak, non-discriminative gradients for \method{} and noisy gradients that destabilize diffusion training for \methoddp{}. These results confirm that the general dataset is essential despite never entering the policy objective directly. $\mathcal{D}_G$ provides the contrastive supervision that shapes $\mu_\phi$, which in turn transmits behavioral-quality distinctions to the policy via the \lcloss{}. Without it, the framework collapses toward standard expert-only imitation.

\subsection{$\lambda$ sensitivity}
\label{appx:lambda_sensitivity}

We study the effect of the coefficient $\lambda$, which balances the behavior cloning (BC) objective and the \lcloss{} for \method{}, and the diffusion objective and the \lcloss{} for \methoddp{}. In \Cref{fig:lambda_sensitivity}, we evaluate $\lambda \in \{1.0, 0.1, 0.01, 0.001, 0.0001\}$. The results show that feedforward policies improve consistently by a large margin, but are sensitive to larger values of $\lambda$, with $\lambda \ge 1.0$ leading to degraded performance, whereas diffusion policies improve over $\lambda \in \{1.0, 0.1, 0.0001\}$, but degrade at moderate $\lambda$.

\input{figures/lambda_sensitivity}

%% file: tables/llm_captioner_ablation.tex
\begin{table}
  \caption{\textbf{\llm{} ablation.} We ablate the choice of \llm{} backbone and finetuning strategy, reporting mean and standard deviation across 50 environment rollouts and 5 seeds for both \method{} and \methoddp{} on the task \boxclose{}. (\% omitted.)}
  \label{tab:llm_captioner_ablation}
  \begin{center}
    \begin{small}
      \scalebox{1.0}{
        \begin{tabular}{cccc}
          \toprule
 Backbone & Finetune mode & \method{} & \methoddp{} \\
\midrule
    \texttt{SmolLM2-135M-Instruct}
    & Full-model
    & \textbf{80.4} $\pm$ 4.6 & \textbf{69.2} $\pm$ 6.7 \\

    \texttt{SmolLM2-135M-Instruct}
    & Projector-only
    & 62.0 $\pm$ 14.3 & \underline{66.4} $\pm$ 7.1 \\

    \texttt{SmolLM2-135M-Instruct}
    & Pretrained
    & 53.6 $\pm$ 8.2 & 60.0 $\pm$ 15.0 \\
    
    \texttt{SmolLM2-360M-Instruct}
    & Full-model
    & 55.2 $\pm$ 6.6 & 66.0 $\pm$ 4.9 \\
    
    \texttt{Qwen2-0.5B-Instruct}
    & Full-model
    & \underline{67.2} $\pm$ 17.1 & 62.4 $\pm$ 6.1 \\

    \texttt{Mistral-7B-Instruct-v0.3}
    & Projector-only
    & 61.2 $\pm$ 7.9 & 58.8 $\pm$ 5.4 \\
          \bottomrule
        \end{tabular}
      }
    \end{small}
  \end{center}
\end{table}

%% file: tables/llm_training_data_analysis.tex
\begin{table}[h]
\centering
\caption{\textbf{LLM-Captioner training data analysis.} We compare \llm{} training on the full general dataset $\mathcal{D}_G$ against training on the expert subset $\mathcal{D}_E$ only, evaluated on \boxclose{} over 50 rollouts and 5 seeds (\% omitted). Removing $\mathcal{D}_G$ from captioner training degrades \method{} by $21.2\%$ and \methoddp{} by $13.2\%$, driving \methoddp{} below the plain DP baseline.}

\label{tab:captioner_data_ablation}
    \begin{tabular}{llc}
        \toprule
        Method & $\mu_\phi$ training data & \boxclose{} \\
        \midrule
        BC & -- & $54.8 \pm 8.4$ \\
        \method{} & $\mathcal{D}_E$ only & $59.2 \pm 10.4$ \\
        \method{} & $\mathcal{D}_G$ (ours) & $\mathbf{80.4 \pm 4.6}$ \\
        \midrule
        DP & -- & $64.8 \pm 5.0$ \\
        \methoddp{} & $\mathcal{D}_E$ only & $56.0 \pm 9.1$ \\
        \methoddp{} & $\mathcal{D}_G$ (ours) & $\mathbf{69.2 \pm 6.7}$ \\
        \bottomrule
    \end{tabular}
\end{table}

%% file: figures/lambda_sensitivity.tex
\begin{figure}[H]
    \centering
        \begin{subfigure}[b]{0.495\textwidth}
            \centering
            \includegraphics[width=\textwidth]{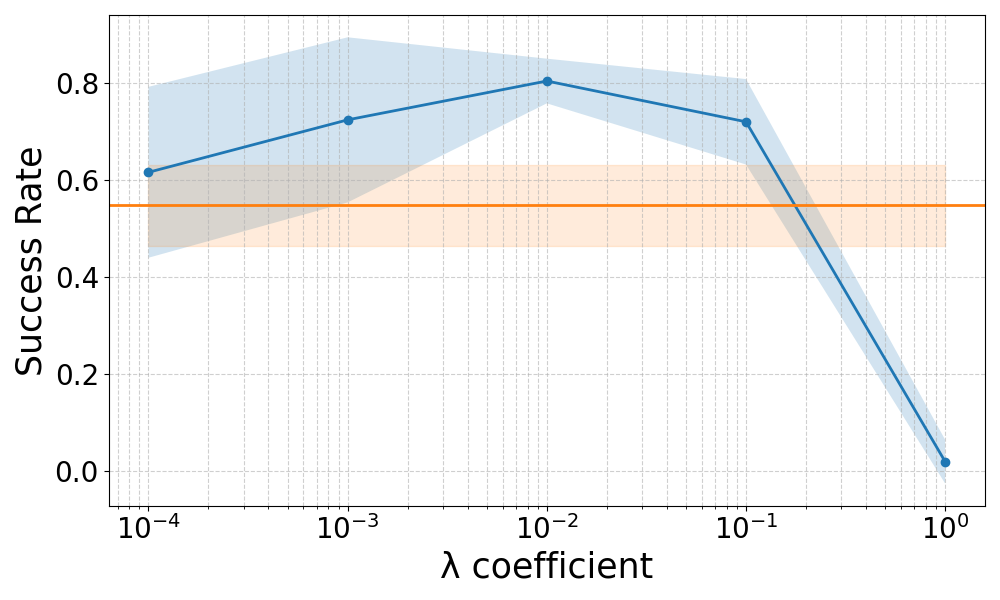}
            \caption{\method{} $\lambda$ sensitivity}
            \label{fig:lambda_sensitivity_lcbc}
        \end{subfigure}
        \hfill
        \begin{subfigure}[b]{0.495\textwidth}
            \centering
            \includegraphics[width=\textwidth]{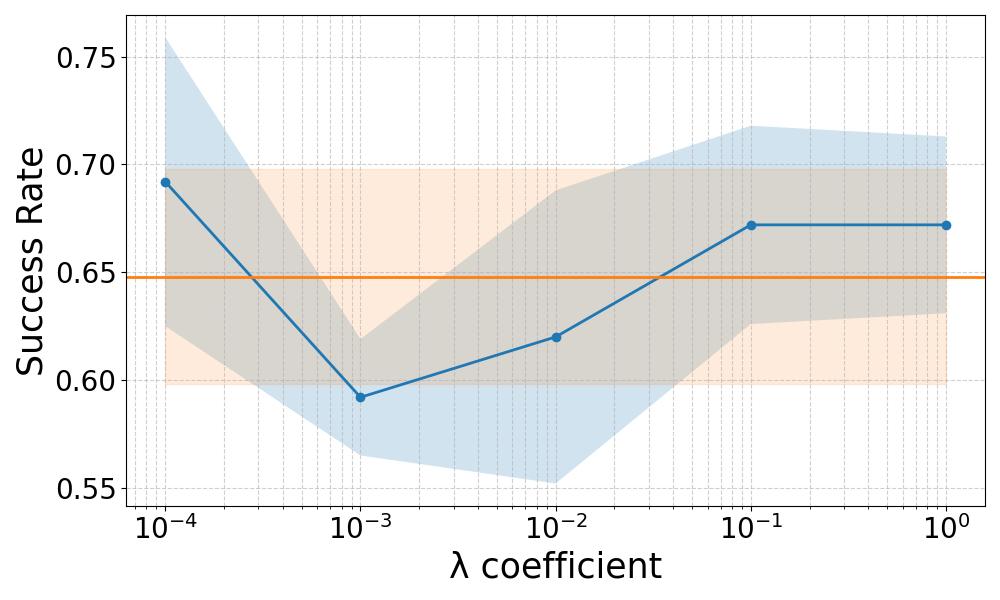}
            \caption{\methoddp{} $\lambda$ sensitivity}
            \label{fig:lambda_sensitivity_lcdp} 
        \end{subfigure}
    \caption[]{\textbf{$\lambda$ sensitivity.}
    Blue curves represent \method{} and \methoddp{}, while orange horizontal lines represent BC and DP performance. We vary the balancing coefficient $\lambda \in \{1.0, 0.1, 0.01, 0.001, 0.0001\}$.}
    \label{fig:lambda_sensitivity}
\end{figure}

%% file: tex/appendix/lc_loss_rollouts.tex
\section{Qualitative results of \lcloss{} and \llm{}}
\label{appx:lc_loss_rollouts}

To evaluate the qualitative performance, we visualize the execution of \peginsertionside{} and \boxclose{} tasks. We display six key frames for each, accompanied by the instantaneous \lcloss{} and the generated \lang{}s. The results demonstrate that \method{} achieves significantly lower \lcloss{} compared to BC due to the direct optimization of this objective. This stability allows \method{} to generate \lang{}s that offer clear, corrective semantic guidance even when the agent faces potential misalignment. This capability to detect and correct errors via language ensures robust performance and task success.

In the \peginsertionside{} task (\Cref{fig:pegInsertion_rollout}), the \method{} agent successfully aligns and inserts the peg and maintains a perfect \lcloss{} of 0.0. The generated \lang{}s provide accurate, actionable corrections, and generate precise instructions like, \textit{"Yet the action is not good for the state, as the peg misaligns with the hole. Also, you should yaw toward the left softly."} allowing it to identify the misalignment of the peg and the hole and include accurate correction, which leads to the success of the task. In contrast, the BC policy achieves a relatively low \lcloss{}, but the peg remains outside the hole; when the observation deviates from the desired state, the policy fails to recover.

\begin{figure}[htp]
    \centering
    \begin{minipage}{\textwidth}
        \centering
        \includegraphics[width=\linewidth]{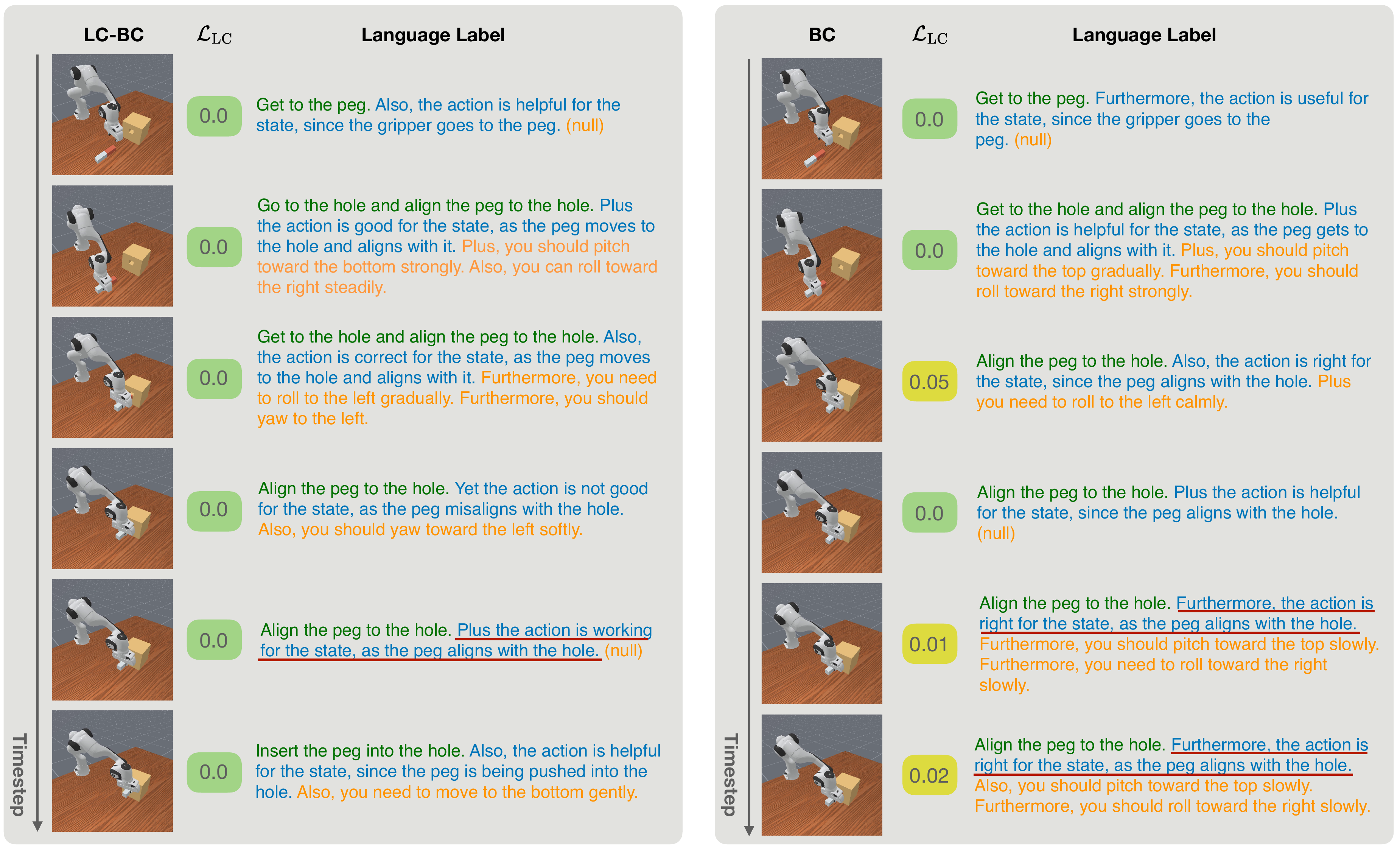}
        \caption{\textbf{Qualitative result of \peginsertionside{}.}}
        \label{fig:pegInsertion_rollout}
    \end{minipage}
\end{figure}

\begin{figure}[htp]
    \centering
    \begin{minipage}{\textwidth}
        \centering
        \includegraphics[width=\linewidth]{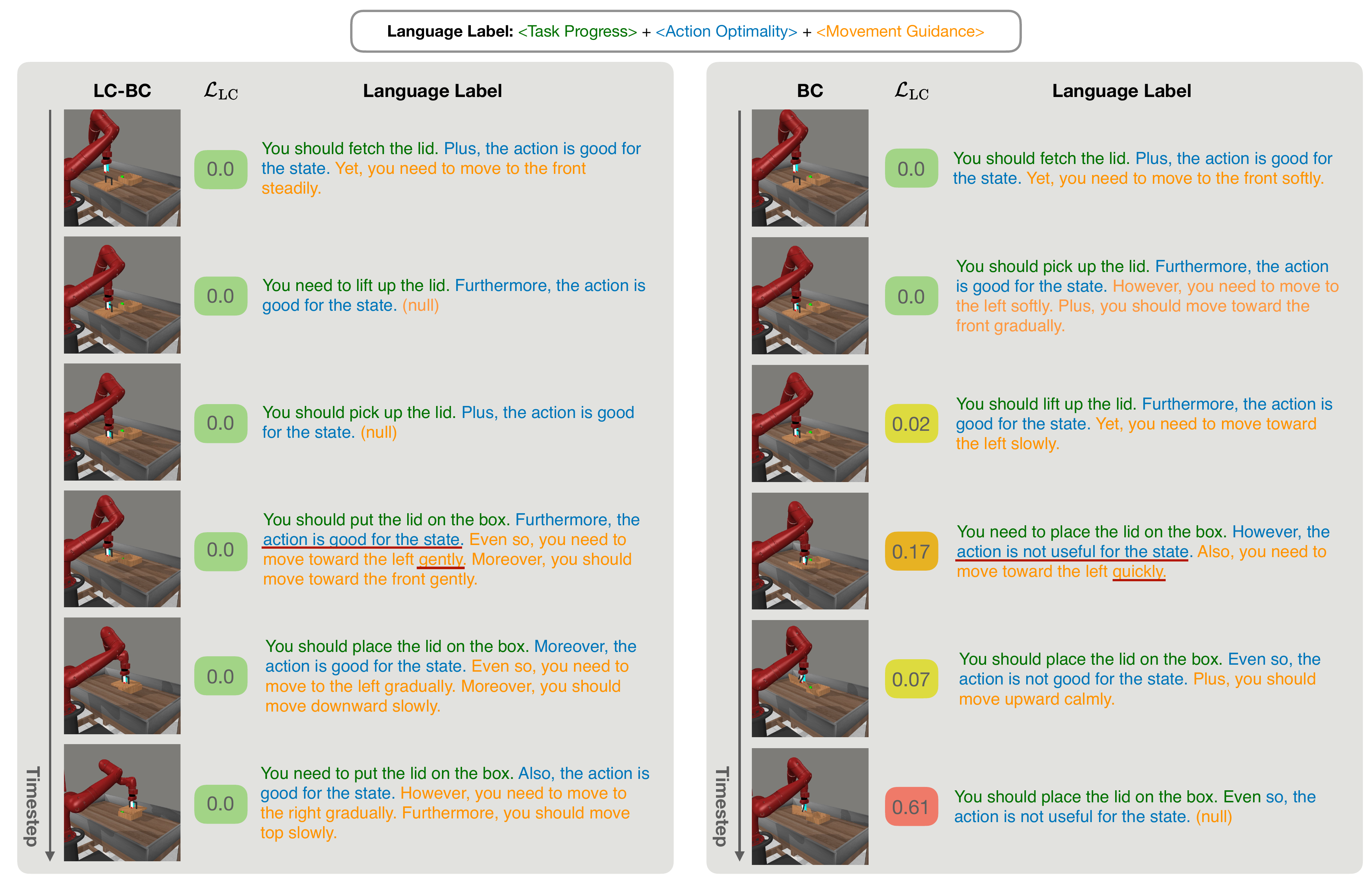}
        \caption{\textbf{Qualitative result of \boxclose{}.}}
        \label{fig:box-close_rollout}
    \end{minipage}
\end{figure}

For \boxclose{} task (\Cref{fig:box-close_rollout}), the \method{} agent maintains a consistently low \lcloss{} throughout the episode; this stability is reflected in the generated language labels. From the fourth key frame of \method{}, it shows that instructions like \textit{"the action is good for the state."} and \textit{"you need to move toward the left gently"} provide assessment of action and movement guidance. As a result, the \method{} agent successfully aligns the lid and closes the box. On the other hand, the BC agent struggles as the task progresses, its \lcloss{} spikes significantly (reaching 0.17 and 0.61) in the later stages, and the BC agent cannot recover from errors once the observation drifts from the training distribution, thus resulting in failure of the task.

%% file: tex/appendix/impact.tex
\section{Impact statement}
\label{appx:impact}

This work studies the use of natural-language feedback as a supervision signal for offline imitation learning in continuous-control settings. By providing structured guidance, our approach improves learning from suboptimal datasets. Our approach directly advances the potential applications of continuous-control systems, with promising applications in robotic control. We do not foresee immediate negative societal impacts beyond these considerations.

%% file: tex/appendix/limitation.tex
\section{Limitations}
\label{appx:limitation}
While our proposed framework demonstrates strong performance on continuous-control benchmarks, it has several limitations.
The language feedback used as supervision may reflect implicit biases derived from its source, and structured language critiques are most naturally suited to tasks with semantically interpretable motions; for tasks such as locomotion, where behavioral distinctions are difficult to articulate in natural language, \lang{}s may be less informative.
Additionally, the two-stage pipeline introduces computational overhead relative to standard behavior cloning, as training and evaluating the language model $\mu_\phi$ adds cost despite being frozen during policy learning.

Our current label generator $\mu_g$ relies on designed heuristics over privileged state information; replacing it with a vision-language model would broaden applicability to image-based settings, but existing VLMs do not yet produce the fine-grained, action-level critiques our framework requires, making this a promising direction for future work.
More broadly, our empirical evaluations are restricted to offline continuous-control benchmarks, and it remains an open question how well the approach generalizes to more complex and dynamic real-world robotic settings.